\documentclass{article}

\usepackage[main, final]{neurips_2026}

\usepackage[utf8]{inputenc} 
\usepackage[T1]{fontenc}    
\usepackage{hyperref}       
\usepackage{url}            
\usepackage{booktabs}       
\usepackage{amsfonts}       
\usepackage{amssymb}        
\usepackage{nicefrac}       
\usepackage{microtype}      
\usepackage[table]{xcolor}
\usepackage{amsmath}
\usepackage{amsthm}
\usepackage{algorithm}
\usepackage{algpseudocode}
\usepackage{graphicx}
\usepackage{multirow}
\usepackage{arydshln}
\usepackage{subcaption}
\usepackage{makecell}
\usepackage{pifont}
\usepackage{colortbl}
\usepackage{placeins}       

\usepackage[capitalize,noabbrev]{cleveref}

\usepackage[textsize=tiny]{todonotes}
\usepackage{ragged2e}
\usepackage{bm}
\usepackage{wrapfig}
\usepackage{booktabs}
\usepackage{xcolor}
\usepackage{colortbl}

\definecolor{flameA}{HTML}{8FCB89}   
\definecolor{flameB}{HTML}{B7DEB2}   
\definecolor{flameC}{HTML}{D7EBD3}   
\definecolor{flameD}{HTML}{EEF6EC}   

\newcommand{\fA}[1]{\cellcolor{flameA}\textbf{#1}}  
\newcommand{\fB}[1]{\cellcolor{flameB}#1}
\newcommand{\fC}[1]{\cellcolor{flameC}#1}
\newcommand{\fD}[1]{\cellcolor{flameD}#1}
\newcommand{\bbest}[1]{\textbf{#1}}                  

\newtheorem{proposition}{Proposition}
\newtheorem{corollary}{Corollary}

\newtheorem{remark}{Remark}
\algnewcommand\algorithmicforeach{\textbf{for each}}
\algdef{S}[FOR]{ForEach}[1]{\algorithmicforeach\ #1\ \algorithmicdo}

\title{FLAME: Adaptive Mixture-of-Experts for Continual Multimodal Multi-Task Learning}

\author{%
  Xing Han\thanks{Equal Contribution.}, ~Shravan Chaudhari$^*$, ~Tanvi Ranade \\
  Department of Computer Science\\
  Johns Hopkins University\\
  \texttt{\{xhan56, schaud35, tranade1\}@jhu.edu} \\
  \AND
  Rama Chellappa \\
  Department of ECE\\
  Johns Hopkins University \\
  \texttt{rchella4@jhu.edu} \\
  \And
  Suchi Saria \\
  Department of Computer Science \\
  Johns Hopkins University, Bayesian Health\\
  \texttt{ssaria1@jhu.edu} \\
}

\begin{document}
\newcommand{\flamecl}{\textsc{FLAME-CL}}
\newcommand{\flame}{\textsc{FLAME}}
\newcommand{\flexmoe}{\textsc{FlexMoE}}
\newcommand{\lora}{\textsc{LoRA}}
\newcommand{\ewc}{\textsc{EWC}}
\newcommand{\simpleft}{\textsc{Simple FT}}
\newcommand{\highmmt}{\textsc{HighMMT}}
\newcommand{\fusemoe}{\textsc{FuseMoE}}
\newcommand{\moe}{MoE}

\maketitle

\begin{abstract}
Real-world model deployment across multiple domains requires multimodal models to operate under two complementary regimes: (1) multi-task pretraining, tasks are co-available at design time where related tasks could borrow representational strength from one another, (2) continual adaptation, in which new tasks emerge after deployment with previously unseen modality combinations. However, neither regime alone suffices: the pretraining task set is never exhaustive, while bypassing joint training forfeits the transfer gains and efficiency among co-trainable tasks. 
Sparse Mixture-of-Experts (MoE) is a natural fit for this dual requirement: sparse activation enables modular capacity expansion as new tasks arrive, while routing decouples modality-level computation from task-level composition.
In this work, we propose a scalable MoE framework for multitask pretraining and continual learning across flexible modality combinations. The framework is designed to support training on multimodal tasks with diverse modality configurations by leveraging modality-specific routers that process tokens from each modality across tasks. Furthermore, it enables continual learning over sequential multimodal tasks within a fixed-capacity MoE by compressing accumulated expert knowledge into low-rank memory subspaces, while expanding only the lightweight routers. 
We validate the effectiveness of our method on multiple healthcare multimodal benchmarks. It demonstrates competitive multitask pretraining performance while alleviating catastrophic forgetting and improving parameter efficiency\footnote{Code is available at \url{https://github.com/aaronhan223/FLAME/tree/continual-learning}. A minimal demo for \flame{} multi-task continual learning on MNIST \citep{deng2012mnist} is available for experiment at \url{https://tinyurl.com/45h6bm8e}.}.
\end{abstract}

\section{Introduction}\label{sec:intro}
The proliferation of large-scale multimodal foundation models \cite{hurst2024gpt, team2023gemini, moor2023med} has demonstrated great potential of learning across diverse modality combinations. Recent multimodal foundation models \cite{han2024fusemoe, yun2024flex, liang2022high, han2025guiding} have extended beyond conventional image-text pairs to support flexible modality combinations, accommodating arbitrary missingness and irregular temporal dynamics across a broader spectrum of modality types. This practical setting is referred as \textbf{Flexi-Modal Data} \cite{han2024fusemoe}, demanding fusion frameworks that are simultaneously flexible, scalable, and temporally aware. However, as shown in Figure \ref{fig:overview}, real-world deployment requires models that simultaneously handle multiple predictive tasks \cite{zhang2021survey, wu2025dynamic}, each associated with a heterogeneous set of input modalities, a setting rarely addressed by existing frameworks. For example, in clinical situations, a mortality rate prediction task may rely on vital signs and clinical notes, whereas a related length-of-stay prediction task may additionally incorporate electrocardiogram (ECG) and imaging data \cite{johnson2023mimic, johnson2024mimic, gow2023mimic}. Moreover, even for similar tasks, modality combinations can differ across populations \cite{wang2020multimodal, tran2017missing, han2024fusemoe}. Moreover, the set of tasks available during pretraining is rarely exhaustive: after deployment, new clinical questions, sensor modalities, and institutional protocols continually emerge \cite{lenga2020continual, rieke2020future}, introducing tasks with previously unseen modality combinations that must be absorbed without retraining the model from scratch or sacrificing performance on existing tasks \cite{kirkpatrick2017overcoming, de2021continual, wang2024comprehensive}. This variability is the norm rather than the exception. A straightforward solution is training separate models for each task, but this would incur substantial computational overhead, as each model requires extensive tuning, validation, and approval, and forfeits the transfer gains available when related tasks share modalities. Recent state-of-the-art multimodal models have demonstrated strong generalization across different tasks; however, these frameworks typically assume a unified and consistent input space across tasks \cite{moor2023med, li2023llava, vandenhende2021multi, fifty2021efficiently}, failing to handle either the Flexi-Modal multi-task setting at training time or the arrival of new tasks with shifted modality configurations after deployment.

\begin{figure}[t]
    \centering
    \includegraphics[width=\linewidth]{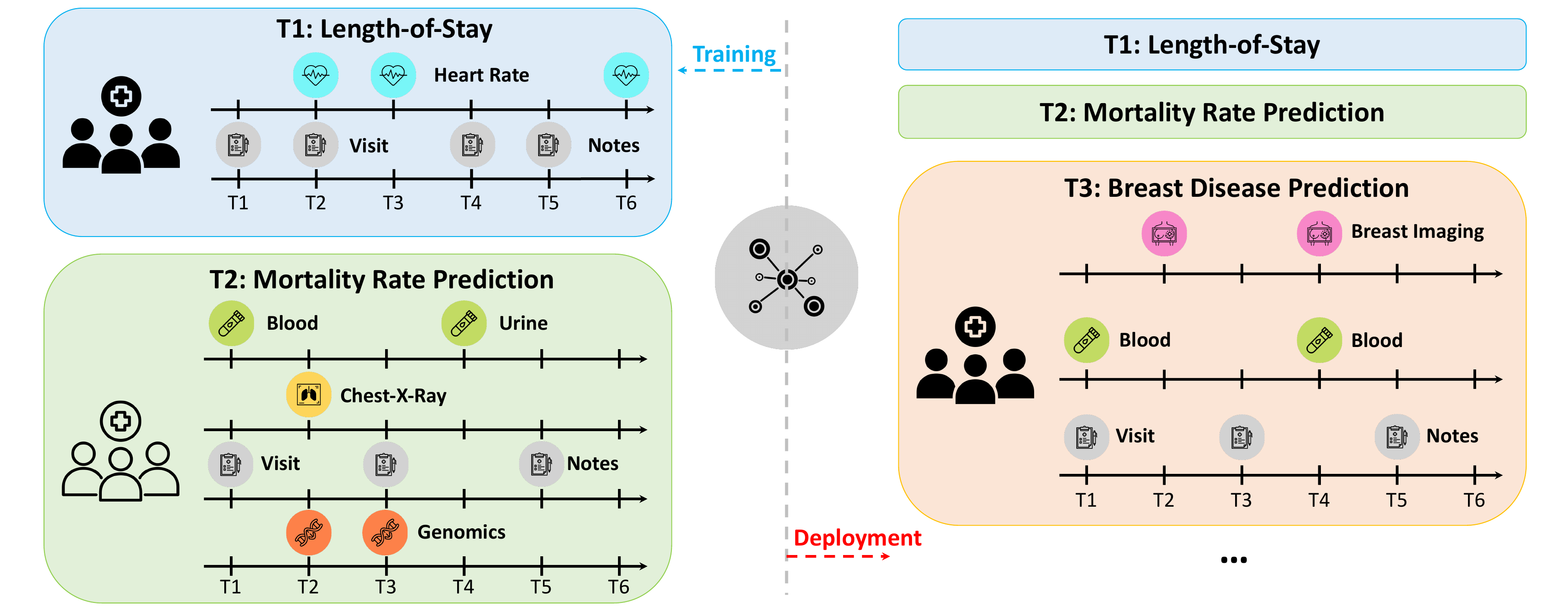}
    \caption{An illustration of the Flexi-Modal multi-task setting. Each task is associated with an arbitrary combination of modalities, where within each task, individual modality measurements may be recorded at irregular time points across channels. During pretraining, models are expected to simultaneously handle various tasks. In deployment, new tasks with unseen modality combinations frequently emerge. FLAME aims to address two core challenges: (1) how to jointly train a unified model across Flexi-Modal multitasks with heterogeneous modality combinations; and (2) how to handle arbitrary train-test modality combination shift and preserve continual learning capacity when new tasks with different modality compositions are introduced after deployment.}
    \label{fig:overview}
\end{figure}

A principled entry point is to study modality sharing across tasks: a text encoder trained on clinical notes for mortality prediction should transfer useful representations to a Phenotyping task that also consumes clinical text, even if the remaining modalities differ. HighMMT \cite{liang2022high} empirically supports this view, but relies on a pairwise Perceiver-based cross-attention fusion \cite{jaegle2021perceiver} that scales poorly as new tasks with varying modality configurations are introduced. $\text{M}^4\text{oE}$ \cite{wu2025dynamic} models patient-level multimodal information via dedicated experts but assumes a fixed modality combination and ignores temporal dependencies in time-varying modalities. More broadly, sparse MoEs \cite{shazeer2017outrageously, fedus2022switch, jiang2024mixtral, guo2025deepseek} are a natural fit for joint pretraining and continual adaptation: the routing mechanism gives flexibility to allocate task-level composition, while the sparsely activated set of experts naturally supports expansion for continual adaptation. However, existing multimodal MoEs route over concatenated, flattened token sequences, coupling gate parameters to a specific modality subset and offering no mechanism to selectively share a single modality across tasks. Although MoEs are theoretically shown to mitigate forgetting \cite{li2024theory} and have been explored in continual learning \cite{chen2023lifelong, rypesc2024divide}, prior work is confined to small-scale unimodal settings with inefficient knowledge reuse.

To address these limitations, we propose \textbf{FLAME}, a \textbf{F}lexi-modal \textbf{L}earning \textbf{A}daptive \textbf{M}ixture-of-\textbf{E}xperts framework for multitask learning and continual adaptation across heterogeneous modality combinations, with a particular focus on healthcare applications. FLAME's contributions are as follows: (1) a per-modality routing mechanism that assigns each modality type a dedicated router dispatching its sequences (regardless of task origin) into a shared expert pool. It decouples the routing topology from any specific task's modality combination and naturally accommodates train/test modality shifts; (2) an efficient continual learning scheme that compresses each new task's contribution into a low-rank additive slice of a \textit{fixed-size} expert pool while expanding only lightweight task-specific routers. This is motivated by our observation that the functional energy of trained experts concentrates in a sharply low-rank subspace despite near-full-rank weights. This yields a structural no-forgetting guarantee via cursor-based inference at $5\sim 15 \times$ fewer parameters than fine-tuning, EWC, and LoRA; and (3) interpretable insights in which inter-task routing fingerprints reveal which tasks collapse onto overlapping experts, providing a data-driven recommender for which multimodal task combinations benefit from joint training.

\textbf{Related Works.}
Mixture-of-Experts has emerged as a scalable paradigm for multimodal learning, capturing inter- and intra-modality interactions through learned routing rather than explicit pairwise attention \cite{shazeer2017outrageously, mustafa2022multimodal}. LIMoE \cite{mustafa2022multimodal} pioneered large-scale image-text MoE with entropy-based expert specialization; FuseMoE \cite{han2024fusemoe} introduced Laplace gating for irregularly sampled and missing modalities; Flex-MoE \cite{yun2024flex} added a missing-modality bank with dual routers for arbitrary modality availability; and MERGE \cite{han2025guiding} scaled to the massively multimodal regime via redundancy–uniqueness–synergy interactions. All assume a \emph{fixed} modality structure across inputs, making them incompatible with flexi-modal multi-task settings. Orthogonally, sparse MoEs are effective multitask learners \cite{hendawy2023multi, gupta2022sparsely}, while continual learning (CL) addresses the forward-transfer/backward-interference trade-off \cite{kirkpatrick2017overcoming, rusu2016progressive}. Recent MoE-based CL methods expand experts for new distributions (Lifelong-MoE \cite{chen2023lifelong}), use LoRA experts for lifelong LLM adaptation (MoRAL \cite{yang2024moral}), or selectively fine-tune the least-overlapping expert per task (SEED \cite{rypesc2024divide}). However, none investigate whether a multimodal MoE can continually adapt to new tasks with unseen modality combinations without retraining or degrading prior performance. We compare our work with representative baselines in Table \ref{tab:method_comparison}.

\newcommand{\cmark}{\textcolor{green!55!black}{\ding{51}}}
\newcommand{\xmark}{\textcolor{red!70!black}{\ding{55}}}
\newcommand{\pmark}{\textcolor{orange!85!black}{$\boldsymbol{\sim}$}}

\begin{table*}[t]
\centering
\caption{Capability comparison between \textbf{FLAME} and representative baselines along different axes. Existing multimodal models rarely support both multitask pretraining and continual learning on new tasks simultaneously. Expert knowledge compression for efficient CL is another unique contribution.}
\label{tab:method_comparison}
\renewcommand{\arraystretch}{1.18}
\resizebox{\textwidth}{!}{%
\begin{tabular}{l ccc cc cc}
\toprule
 & \multicolumn{3}{c}{\textbf{Multi-Modal Multi-Task Pretraining}} & \multicolumn{2}{c}{\textbf{Continual Learning and Shifts}} & \multicolumn{2}{c}{\textbf{Architecture \& Efficiency}} \\
\cmidrule(lr){2-4} \cmidrule(lr){5-6} \cmidrule(lr){7-8}
\textbf{Method} & \makecell{Flexi-Modal\\Inputs} & \makecell{Temporal\\Irregularity} & \makecell{Heterogeneous\\Task Modalities} & \makecell{New-Task\\Adaptation} & \makecell{Shift of Modality\\Composition} & \makecell{Sparse MoE\\Routing} & \makecell{Expert Knowledge\\Compression} \\
\midrule
HighMMT~\cite{liang2022high}          & \cmark & \xmark & \cmark & \xmark & \cmark & \xmark & N.A. \\
FuseMoE~\cite{han2024fusemoe}         & \cmark & \cmark & \xmark & \xmark & \xmark & \cmark & \xmark \\
Flex-MoE~\cite{yun2024flex}           & \cmark & \xmark & \xmark & \xmark & \cmark & \cmark & \xmark \\
M$^4$oE~\cite{wu2025dynamic}          & \cmark & \xmark & \cmark & \xmark & \cmark & \cmark & \xmark \\
SEED~\cite{rypesc2024divide}          & \xmark & \xmark & N.A. & \cmark & N.A. & \cmark & \xmark \\
Lifelong-MoE~\cite{chen2023lifelong}  & \xmark & \xmark & N.A. & \cmark & N.A. & \cmark & \xmark \\
\midrule
\rowcolor{gray!12}
\textbf{FLAME (Ours)}                 & \cmark & \cmark & \cmark & \cmark & \cmark & \cmark & \cmark \\
\bottomrule
\end{tabular}%
}
\vspace{-0.4em}
\end{table*}

\section{Adaptive MoE for Continual Multimodal Multi-Task Learning} \label{sec:method}
In this section, we present the core modules of the proposed FLAME framework, which consists of two components. Sec.~\ref{subsec:multitask} describes the training of multiple flexi-modal tasks, providing a foundation for multimodal models to acquire task-specific knowledge. We then introduce continual adaptation to new multimodal tasks in Sec.\ref{subsec:cl}.
We first show that, during pretraining, the functional energy of expert weights tends to concentrate in lower-rank components, indicating that the representational capacity of experts is underutilized. This observation motivates our parameter-efficient continual learning approach, which compresses accumulated expert knowledge into low-rank subspaces. 

\subsection{Multi-Task Pretraining of FLAME}\label{subsec:multitask}
\textbf{Preliminary: Flexi-Modal Formulation and Sparse MoE.}
Let $\mathcal{M} = \{m_1, \ldots, m_M\}$ denote all possible modality types, and let $\mathcal{T} = \{\mathcal{T}_1, \ldots, \mathcal{T}_T\}$ be a set of tasks. Each task $\mathcal{T}_k$ is associated with a modality subset $\mathcal{M}_k \subseteq \mathcal{M}$, a dataset $\mathcal{D}_k$, and a prediction objective $\mathcal{L}_k$. A sample of $\mathcal{T}_k$ is a collection of modality sequences
$
    \mathbf{x} = \big\{\, \mathbf{x}_m \in \mathbb{R}^{L_m \times d_m} \,\big|\, m \in \mathcal{M}_k \,\big\},
$
where $L_m$ and $d_m$ are the length and feature dimension of modality $m$. Both the modality subset $\mathcal{M}_k$ and the lengths $\{L_m\}_{m \in \mathcal{M}_k}$ vary arbitrarily across tasks, and within each task, measurements may be recorded at irregular time points. A standard sparse MoE layer with $N$ experts $\{E_1, \ldots, E_N\}$ and a gating network $G$ maps a token $\mathbf{z} \in \mathbb{R}^d$ to $\sum_{i=1}^{N} G(\mathbf{z})_i \cdot E_i(\mathbf{z})$, with sparsity enforced by activating only the top-$K$ experts. To enter the MoE, each modality $m \in \mathcal{M}_k$ is first projected by an encoder $\phi_m : \mathbb{R}^{L_m \times d_m} \to \mathbb{R}^{L_m \times d}$ into a sequence of $d$-dimensional embeddings $\mathbf{z}_m = \phi_m(\mathbf{x}_m)$.
 
\textbf{Modality-Specific Routing.}
Existing multimodal MoEs \cite{han2024fusemoe, yun2024flex} concatenate all modalities into a single flattened sequence and route over its tokens. This couples the gate parameter shapes to a particular modality combination $\mathcal{M}_k$ and to the lengths $\{L_m\}$, preventing reuse across tasks with heterogeneous modality subsets. We instead instantiate one router $G_m$ \emph{per modality type}, dispatching modality-$m$ embeddings from \emph{any} task into a single shared expert pool $\{E_1, \ldots, E_N\}$. Routing is performed at the \emph{sample} level rather than at the token level, so each decision sees an entire sequence's temporal context while the expert input shape remains $\mathbb{R}^{L_m \times d}$ regardless of how many modalities the current task contains.
Sample-level routing requires a fixed-size summary of the variable-length sequence $\mathbf{z}_m \in \mathbb{R}^{L_m \times d}$. Mean pooling discards temporal structure, while flattening reintroduces $L_m$ into the gate parameters. We instead use a learnable query $\mathbf{q}_m \in \mathbb{R}^d$ that attends over time:
\begin{equation}
\bar{\mathbf{z}}_m \;=\; \sum_{t=1}^{L_m} \alpha_{m,t}\, \mathbf{z}_{m,t},
\qquad
\alpha_{m,t} \;=\; \frac{\exp(\langle \mathbf{z}_{m,t},\, \mathbf{q}_m \rangle / \sqrt{d})}{\sum_{t'} \exp(\langle \mathbf{z}_{m,t'},\, \mathbf{q}_m \rangle / \sqrt{d})}.
\label{eq:tap}
\end{equation}
The summary $\bar{\mathbf{z}}_m \in \mathbb{R}^d$ has fixed size and emphasizes routing-relevant time steps. Router $G_m$ then produces sparse top-$K$ gates over the shared pool from $\bar{\mathbf{z}}_m$ via noisy top-$K$ gating~\citep{shazeer2017outrageously}; the gate parameters of $G_m$ are shaped only by $d$ and $N$ and are independent of $L_m$ and $|\mathcal{M}_k|$.

\begin{figure}[t]
    \centering
    \includegraphics[width=0.95\linewidth]{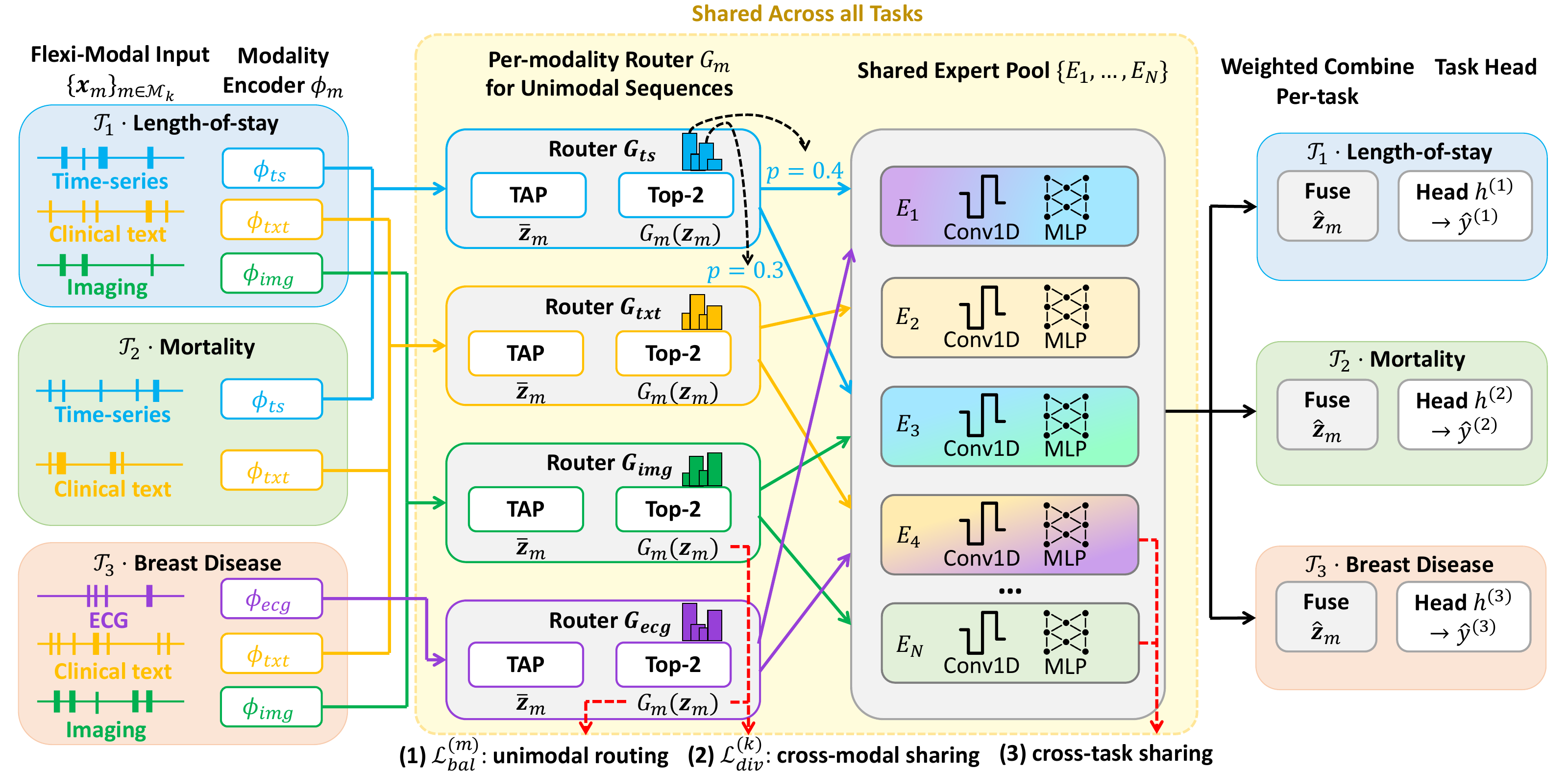}
    \caption{FLAME multi-task pretraining architecture. Flexi-Modal tasks with overlapping but distinct modality subsets are jointly trained. Each modality $m$ is encoded by $\phi_m$ and routed by a dedicated router $G_m$ (Temporal Attention Pooling + Top-K gating) into a shared expert pool. Each expert applies a 1D temporal convolution followed by a position-wise MLP. Same-modality inputs from different tasks (matching arrow/module color) traverse the same router and the same experts, providing structural inter-task sharing. The router balancing function $\mathcal{L}^{(m)}_{\mathrm{bal}}$ regulates unimodal routing for $m$, while the routing distribution divergence $\mathcal{L}^{(k)}_{\mathrm{div}}$ controls cross-modal sharing for $\mathcal{T}_k$.}
    \label{fig:seqmoe}
\end{figure}

\textbf{Temporal-Aware Expert.}
Each expert $E_i$ processes a full single-modality sequence at its native temporal resolution. $E_i$ first applies a length-preserving 1D convolution along the time axis with kernel size $\kappa$ to capture local temporal dynamics, followed by a position-wise two-layer MLP applied identically at every step. All expert weights depend only on $d$ and $\kappa$, so sequences of any length and any modality can traverse the same expert. The MoE output for modality $m$ broadcasts the sample-level gate $G_m(\mathbf{z}_m) \in \mathbb{R}^N$ across the temporal axis, so the same $K$-expert subset processes every time step of a given sample, preserving temporal coherence while remaining sparse.
 
\textbf{Cross-Task and Cross-Modal Interactions.}
Per-modality routing creates different interaction levels to be regulated jointly: (1) \emph{intra-router} expert utilization within a single modality, (2) \emph{inter-task} sharing across tasks that contain the same modality, and (3) \emph{intra-task} interaction across the modalities of a single task. We address (1) and (3) with explicit objectives, while (2) is handled by parameter sharing.
Specifically, (2) arises structurally: because the router $G_m$ and the expert pool $\{E_i\}$ are shared, gradients from $\mathcal{L}_k$ update the same components that any other task $\mathcal{T}_{k'}$ with $m \in \mathcal{M}_k \cap \mathcal{M}_{k'}$ relies on. Similar patterns of modality $m$ across tasks are thus routed to overlapping experts without an explicit alignment loss, while tasks with disjoint modality subsets interact only through experts that both routers happen to select, naturally bounding cross-task interference.
For (1), we apply the coefficient-of-variation penalty \cite{shazeer2017outrageously} to each router separately, ensuring full utilization of the shared pool without forcing distinct modalities to use experts in identical proportions.
For (3), the modalities in $\mathcal{M}_k$ may be either complementary or partially redundant~\citep{han2025guiding}. Let $\bar{\mathbf{g}}^{(k)}_m = \mathbb{E}_{\mathbf{x} \sim \mathcal{D}_k}[G_m(\mathbf{z}_m)] \in \Delta^{N-1}$ denote the average routing distribution of modality $m$ on task $\mathcal{T}_k$. We regulate cross-modal expert overlap with
$
\mathcal{L}^{(k)}_{\mathrm{div}} \;=\; \beta \cdot \binom{|\mathcal{M}_k|}{2}^{-1} \sum_{\substack{m, m' \in \mathcal{M}_k \\ m \neq m'}} \cos\!\big( \bar{\mathbf{g}}^{(k)}_m,\, \bar{\mathbf{g}}^{(k)}_{m'} \big),
$
where $\cos(\cdot, \cdot)$ denotes cosine similarity. Setting $\beta = +1$ (\emph{spread}) penalizes overlap, so modalities specialize on disjoint experts and yield complementary representations; setting $\beta = -1$ (\emph{concentrate}) rewards overlap, encouraging fused representations. The mode is selected per task family. Overall, FLAME minimizes the task loss together with the (1) router balancing loss and (3) interaction loss.


\subsection{Efficient Continual Learning on New Multimodal Tasks} \label{subsec:cl}

Multitask pretraining requires all tasks to be available simultaneously, which is rarely the case in practice. In real-world deployment scenarios, FLAME must be capable of continuously absorbing a stream of new tasks $\mathcal{T}_{T+1}, \mathcal{T}_{T+2}, \ldots$ that arrive with previously unseen modality combinations, without revisiting earlier data. At stage $t$ only $\mathcal{D}_t$ is available, and the model is required to retain performance on previously seen tasks. Given the modularized expert pool in sparse MoEs, they have an inherent advantage in expanding capacity by adding new experts to accommodate new tasks \cite{chen2023lifelong}. However, this approach leads to rapid growth in model size, especially for multimodal tasks that require substantially more computation, making it impractical. Instead, we demonstrate a more efficient solution. We achieve this within a \emph{fixed-size} expert pool $\{E_1,\ldots,E_N\}$ by combining spectral compression of expert knowledge with lightweight task-specific router expansion.

\begin{figure}[t]
\centering
\includegraphics[width=\linewidth]{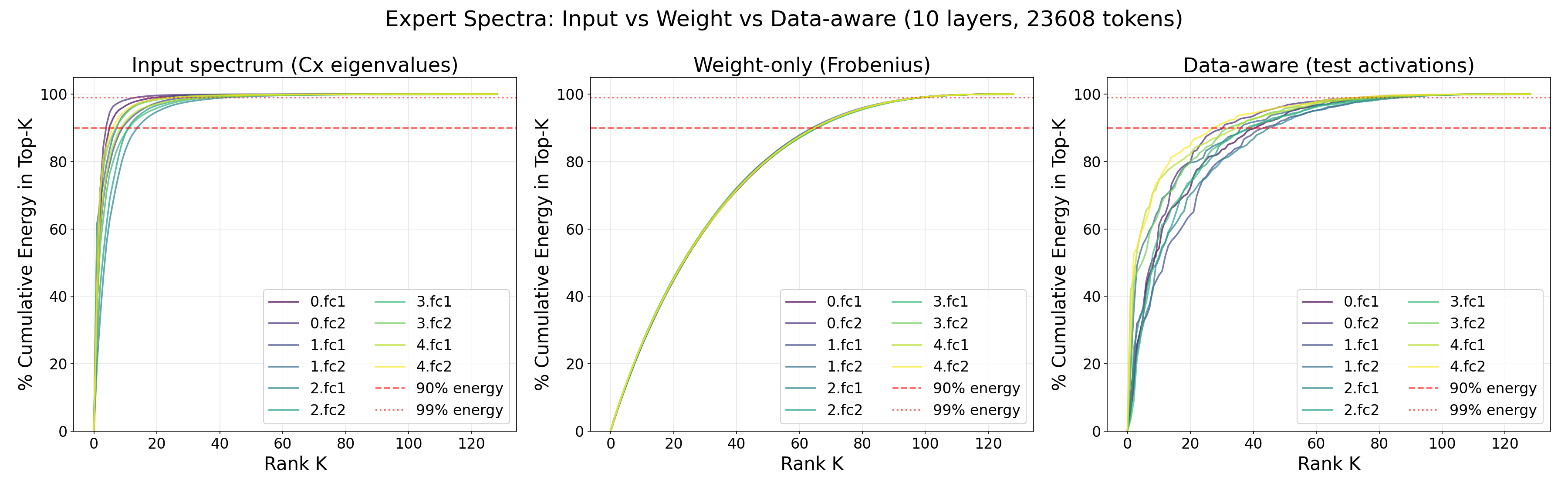}
\caption{Cumulative top-$K$ energy across all 10 expert sublayers (5 experts $\times$ \{\texttt{fc1}, \texttt{fc2}\}) of the first cross-modal MoE block, after FLAME multi-task pretraining on all 9 healthcare tasks.
\textbf{Left:} input-covariance spectrum $\{\lambda_j\}$ of $C_i = \mathbb{E}_{\bm{z}}[\bm{z}\bm{z}^\top]$, computed over inputs the router actually dispatches to each expert.
\textbf{Center:} weight-only Frobenius spectrum $\{\sigma_{i,k}^2\}$ of $W_i$.
\textbf{Right:} data-aware functional energy $\mathcal{E}_{i,k} = \sigma_{i,k}^2 \bm{v}_{i,k}^\top C_i \bm{v}_{i,k}$ on test activations.
The input spectrum and the data-aware energy both saturate at very small rank ($90\%$ energy reached well before $K{=}20$, $99\%$ before $K{=}40$), while the weight-only spectrum is essentially full rank. This empirically confirms Proposition~\ref{prop:funcrank}: although $W_i$ is nearly full rank, its functional contribution along the routed-input distribution lives in a sharply low-rank subspace, exposing the idle capacity that FLAME reallocates to incoming tasks.}
\label{fig:functional_energy}
\end{figure}

\textbf{Motivation: Why is Functional Energy Low-Rank?}
After multi-task pretraining, we observed that the cumulative functional energy of expert weights saturates at much smaller ranks than the ranks of the weight matrices. Specifically, for each expert $E_i$ with weight matrix $W_i$ and Singular Value Decomposition (SVD) $W_i = U_i \Sigma_i V_i^\top$, let $C_i = \mathbb{E}_{\bm{z}}[\bm{z}\bm{z}^\top]$ denote the covariance of inputs that the router actually dispatches to $E_i$. The total functional energy admits two equivalent decompositions,
\begin{equation}
\mathcal{E}_i \;=\; \mathrm{tr}(W_i\, C_i\, W_i^\top) \;=\; \underbrace{\sum_{k} \sigma_{i,k}^2\, \bm{v}_{i,k}^\top C_i\, \bm{v}_{i,k}}_{\text{in } W_i\text{'s SVD basis}} \;=\; \underbrace{\sum_{j} \lambda_j\, \|W_i\, \bm{q}_j\|^2}_{\text{in } C_i\text{'s eigenbasis}},
\label{eq:funcrank-decomp}
\end{equation}
where $\{(\lambda_j, \bm{q}_j)\}$ are eigenpairs of $C_i$. The first sum yields the per-rank functional energy $\mathcal{E}_{i,k} = \sigma_{i,k}^2 \bm{v}_{i,k}^\top C_i \bm{v}_{i,k}$ plotted in Fig.~\ref{fig:functional_energy} (per-task spectra across all benchmarks are collected in Appendix~\ref{app:spectra}); the second isolates its source. We can therefore characterize when the curve saturates early.

\begin{proposition}\label{prop:funcrank}
If $C_i$ has $\epsilon$-effective rank $r^*$, i.e., $\sum_{j > r^*} \lambda_j \le \epsilon\,\mathrm{tr}(C_i)$, then
$\sum_{j > r^*} \lambda_j\, \|W_i \bm{q}_j\|^2 \;\le\; \epsilon\, \|W_i\|_{\mathrm{op}}^2\,\mathrm{tr}(C_i)$,
so the output $W_i \bm{z}$ is concentrated in the rank-$r^*$ subspace $W_i\,\mathrm{span}\{\bm{q}_1,\ldots,\bm{q}_{r^*}\}$, \emph{regardless of} $\mathrm{rank}(W_i)$.
\end{proposition}

\noindent
Proposition~\ref{prop:funcrank} (proof in Appendix~\ref{app:funcrank}) states that the functional rank of an expert is governed by the spectrum of its routed inputs, and not just by the spectrum of its weights as proposed in \cite{kaushik2025eigenlorax, kaushik2025universalweightsubspacehypothesis}. Two companion results in Appendix~\ref{app:alignment} bridge this input-side guarantee to the weight-side truncation in Eq.~\eqref{eq:compress}: gradient flow with the zero initialization of $\widetilde{W}_i^{(t)}$ aligns its right singular vectors with $C_i$'s eigenvectors (Proposition~\ref{prop:align}), and the rank-$r_t$ truncation error generalizes from $n \gtrsim d_{\mathrm{eff}}(C_i)$ stage-$t$ samples (Proposition~\ref{prop:gen}).
We conjecture that modality-specific encoders $\phi_m$ project inputs onto manifolds with intrinsic dimensionality smaller than $d$, which drives $C_i$ to be sharply low-rank. In addition, after sufficient training, $W_i$'s top singular vectors $\bm{v}_{i,k}$ align with $C_i$'s top eigenvectors $\bm{q}_j$, as evidenced by prior works \cite{saxe2013exact, gunasekar2018implicit, arora2019implicit}, which leads to the saturation of $\mathcal{E}_{i,k}$ as observed in Fig.~\ref{fig:functional_energy}: although $W_i$ is nearly full rank, its singular directions beyond saturation are functionally idle, they contribute negligibly to $\mathcal{E}_i$. These idle ranks are the capacity FLAME reallocates to incoming tasks.
 
\textbf{Spectral Compression of Expert Knowledge.}
Let $W_i^{(0)}$ denote the expert weight obtained from the multitask pretraining of Sec.~\ref{subsec:multitask}. At each continual stage $t \ge 1$, we attach to every expert a fresh, full-rank, zero-initialized additive component $\widetilde{W}_i^{(t)}$ of the same shape as $W_i$, train it under unconstrained SGD on $\mathcal{L}_t$ with all components from prior stages frozen, and after convergence replace it by its rank-$r_t$ truncation,
\begin{equation}
\widetilde{W}_i^{(t)} \;=\; U_i^{(t)} \Sigma_i^{(t)} V_i^{(t)\top},
\qquad
W_i^{(t)} \;\leftarrow\; U_{i,\,1:r_t}^{(t)}\, \Sigma_{i,\,1:r_t}^{(t)}\, V_{i,\,1:r_t}^{(t)\top},
\label{eq:compress}
\end{equation}
which is appended to a frozen stack $\Pi_i = (W_i^{(1)}, W_i^{(2)}, \ldots)$ along the rank dimension, so after $\tau$ stages the cumulative delta atop $W_i^{(0)}$ has rank at most $\sum_{j=1}^{\tau} r_j$. Each task $t$ carries a cursor $\tau = k(t)$ pointing to the stage at which it was first trained, and its effective expert weight is
\begin{equation}
W_i^{\mathrm{eff}}(\tau) \;=\; W_i^{(0)} \;+\; \sum_{j=1}^{\tau} W_i^{(j)},
\label{eq:cursor}
\end{equation}
\begin{wrapfigure}[17]{r}{0.49\textwidth}
\vspace{-1.5em}
\begin{center}
    \includegraphics[width=.5\textwidth]{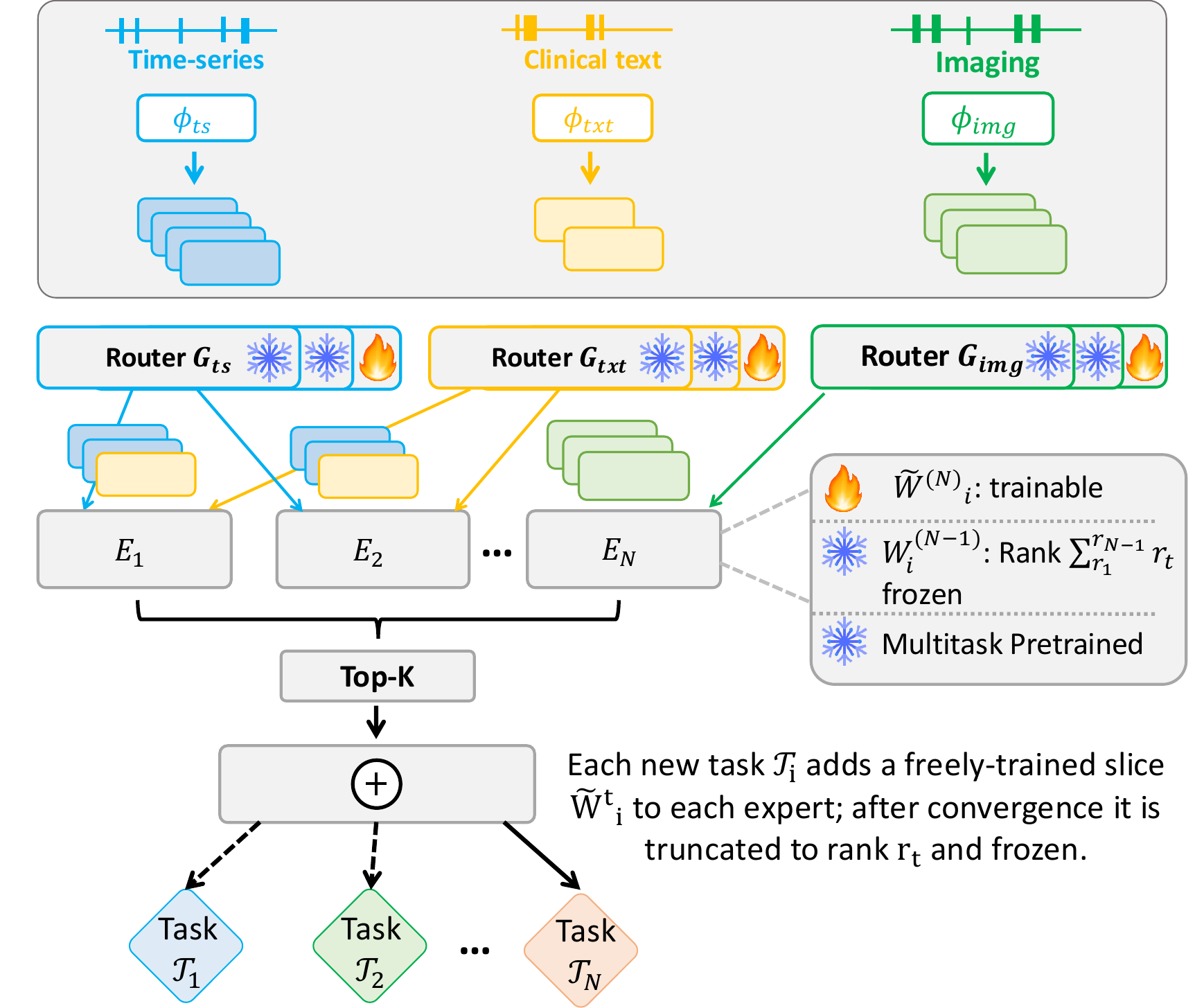}
\end{center}
\vspace{-1.2em}
\caption{Overview of FLAME's CL procedure for MoE, with improved parameter efficiency.}
\label{fig:flame_cl}
\end{wrapfigure}
so the forward pass for task $t$ excludes any component reserved at a later stage. The per-stage rank $r_t$ is a hyperparameter, fixing the cumulative reserved rank $\sum_t r_t$ and turning capacity planning into an explicit choice under a fixed-size MoE. The full set of trainable and frozen parameters at stage $t$, together with the per-stage memory footprint, is detailed in Appendix~\ref{app:method_details}.

\textbf{Task-Specific Router and Head Expansion.}
Compressing experts preserves \emph{what} each expert knows; the routing patterns that compose them into task-specific predictions must also be preserved. We extend the per-modality router $G_m$ from Sec.~\ref{subsec:multitask} into a task-indexed family $\{G_m^{(t)}\}_t$. At stage $t$, for every $m \in \mathcal{M}_t$ we add a new lightweight router head $G_m^{(t)} \in \mathbb{R}^{d \times N}$ dispatching modality $m$ on $\mathcal{T}_t$, while earlier heads $\{G_m^{(s)}\}_{s<t}$ are frozen. A new task head $h^{(t)}$ is added in parallel. Modality encoders $\phi_m$ are instantiated fresh for any genuinely new modality and reused for those already seen; in either case, the variable-length attention layers inside $\phi_m$ that standardize sequence lengths across modalities receive the same per-stage compress-and-stack treatment as the experts via Eq.~\eqref{eq:compress}. At inference, task identity selects the appropriate $\{G_m^{(t)}\}_{m \in \mathcal{M}_t}$, the cursor $\tau = k(t)$ on every stack $\Pi_i$, and the head $h^{(t)}$. The per-stage router overhead is $O(|\mathcal{M}_t|\,d\,N)$, negligible relative to the experts. More details can be found in Algorithm~\ref{alg:ours}.

\textbf{Structural Forgetting Guarantee.}
The combination of frozen prior components and cursor-based inference delivers a no-forgetting property that is structural rather than statistical. \emph{Parameter-level isolation}: all $\{W_i^{(j)}\}_{j<t}$ are frozen during stage-$t$ training, so the optimizer cannot modify any block reserved for a prior task. \emph{Inference-level isolation}: under cursor $\tau = k(t)$, Eq.~\eqref{eq:cursor} sums only components reserved at or before stage $k(t)$, so any $W_i^{(j)}$ with $j > k(t)$ is excluded from task $t$'s forward pass by construction. New tasks therefore cannot perturb the predictions of earlier tasks regardless of what they learn. This is strictly stronger than methods that prevent interference at the optimization level only, since those still combine all parameters at inference time.

\textbf{Comparison to Subspace-Orthogonality CL.}
A complementary line of work enforces non-interference \emph{during} training: GPM~\cite{saha2021gpm} projects gradients onto the orthogonal complement of past-task input subspaces at every step, InfLoRA~\cite{liang2024inflora} constrains the LoRA parameterization to a pre-designed orthogonal subspace, and O-LoRA~\cite{wang2023olora} adds an orthogonality penalty between successive LoRA matrices. \flamecl{} instead trains $\widetilde{W}_i^{(t)}$ under unconstrained SGD and applies rank-$r_t$ truncation post-hoc. This decouples the forgetting constraint from the optimizer: forgetting is prevented structurally by Eq.~\eqref{eq:cursor}, regardless of how $\widetilde{W}_i^{(t)}$ is optimized. The validity of the truncation rests on Proposition~\ref{prop:funcrank} applied to $\widetilde{W}_i^{(t)}$, with $C_i$ retaining its meaning as the routed-input covariance under stage-$t$ data: the functional energy of $\widetilde{W}_i^{(t)}$ concentrates in a low-rank subspace, so discarding low-singular-value directions removes only directions with negligible contribution to stage-$t$ behavior.

\section{Experiments}\label{sec:exp}
\textbf{Overview}. We evaluate FLAME on multi-task pretraining and continual adaptation, organizing the analysis around four questions:
\textbf{(RQ1)} Is FLAME robust to heterogeneous modality combinations in joint multitask pretraining, matching task-specific multimodal baselines?
\textbf{(RQ2)} Which task pairings benefit from joint training, and does the learned per-modality routing recover interpretable cross-task structure?
\textbf{(RQ3)} Can FLAME-CL absorb a stream of new tasks with shifting modality combinations without forgetting earlier ones?
\textbf{(RQ4)} What is the parameter footprint of FLAME-CL relative to the relevant baselines as the task stream grows, and is the underlying compression empirically justified?

\subsection{Experimental Setup}

\textbf{Dataset and Task Information.}
We evaluate \flame{} on four publicly available datasets that together yield nine prediction tasks spanning critical care, breast imaging, and neurodegenerative disease. \textbf{MIMIC-IV}~\citep{johnson2023mimic} provides ICU stays with three modalities: vital/laboratory time series, clinical notes, and chest X-rays via MIMIC-CXR~\cite{johnson2024mimic}, evaluated on three tasks (\textsc{48-IHM}, \textsc{LOS}, \textsc{25-PHENO}). \textbf{eICU}~\citep{pollard2019eicu} adds a multi-center ICU counterpart with time series and clinical notes for two tasks (\textsc{MOR}, \textsc{RAD}). \textbf{EMBED}~\citep{Jeong2023embed} contributes mammography studies with four imaging views (FFDM and synthesized 2D C-View, in CC/MLO projections) and tabular metadata, evaluated on \textsc{BIRADS}, \textsc{RISK}, and \textsc{DENSITY}. \textbf{ADNI}~\citep{Weiner2016adni} provides five per-subject modalities (MRI, FDG-PET, SNP profiles, biospecimen assays, tabular features) over modality-stratified subsets, evaluated on the three-way diagnosis task (\textsc{DIAG}). Further details in Appendix~\ref{app:datasets}.

\textbf{Baselines.}
For multitask multimodal performance, we compare \flame{} against single-task training of recent multimodal baselines: \flexmoe{}~\cite{yun2024flex}, \fusemoe{}~\cite{han2024fusemoe}, and \highmmt{}~\cite{liang2022high}. \fusemoe{} and \flexmoe{} assume a fixed modality combination and cannot be trained across tasks with heterogeneous modality subsets; \highmmt{} supports multitask training but its pairwise Perceiver cross-attention follows an offline modality- and interaction-heterogeneity sharing schedule fixed over a closed task-modality pool: any new modality combination forces re-estimation and joint retraining with all prior tasks, ruling out continual use. It also assumes regularly sampled inputs, with no mechanism for irregular clinical time series, restricting the comparison to multitask pretraining.

For continual learning, we benchmark \flamecl{} against simple fine-tuning (\simpleft{}), Elastic Weight Consolidation (EWC)~\cite{kirkpatrick2017overcoming}, Lifelong-PT~\cite{chen2023lifelong}, and Low-Rank Adaptation (\lora{})~\cite{hu2022lora} (Figs.~\ref{fig:cl-grid},~\ref{fig:cl-vs-lora}); \lora{} is included only as a reference point for adaptation quality rather than a practical CL solution, as its parameter count grows with the task stream. We exclude existing CL methods Inflora~\cite{liang2024inflora}, O-LoRA~\cite{wang2023olora}, and MoRAL~\cite{yang2024moral}: these methods attach per-task LoRA experts to a unimodal pretrained backbone, whereas our setup trains a multimodal MoE backbone from scratch with learned cross-modal routing. Similarly, SEED~\citep{rypesc2024divide} is a class-incremental image classifier that replaces routing and discriminative heads with a per-task generative Gaussian ensemble which makes it structurally incompatible with our pipeline.

\subsection{Results}\label{sec:results}

A single \flame{} backbone trained jointly across heterogeneous task and modality combinations matches task-specific multimodal baselines on AUROC and exceeds them on AUPRC for the imbalanced binary tasks where positives are rare (Tables~\ref{tab:full_single_task_mimic_eicu},~\ref{tab:full_single_task_embed_adni}). The aim is not to outperform single-task models per-task but to show that one architecture absorbs the modality gap typical of joint multimodal training: tasks with overlapping modality pools gain mildly under per-dataset multitask, while disjoint-pool tasks regress only modestly under the all-nine-task regime. In the continual setting (Fig.~\ref{fig:cl-grid}), \flamecl{} retains end-of-stage AUROC within $0.01$ of each task's first-introduction value across four sequences while storing five- to fifteen-fold fewer encoder parameters than \simpleft{}, EWC, and \lora{}, delivering the structural no-forgetting argument of Sec.~\ref{subsec:cl} empirically.

\begin{table}[t]
\centering
\caption{Per-task performance across three runs. Multi-class/multi-label tasks use macro-averaged AUROC/AUPRC; all MoE methods use 5 experts. Cells in all FLAME rows are shaded green, darker for smaller gaps to the best value (darkest when FLAME \emph{is} the best): \fA{\,gap $=0$\,}, \fB{\,gap $\leq 0.005$\,}, \fC{\,gap $\leq 0.015$\,}, \fD{\,gap $> 0.015$\,}. See \cref{tab:full_single_task_appendix} for full statistics and \cref{fig:expert_ablation_auprc} for the experts ablation.}
\label{tab:full_single_task}

\begin{subtable}{\textwidth}
\centering
\caption{MIMIC IV and eICU}
\label{tab:full_single_task_mimic_eicu}
\resizebox{\textwidth}{!}{%
\begin{tabular}{lll ccc cc}
\toprule
\multirow{2}{*}{Metric} & \multirow{2}{*}{Setting} & \multirow{2}{*}{Method} & \multicolumn{3}{c}{MIMIC IV} & \multicolumn{2}{c}{eICU} \\
\cmidrule(lr){4-6}\cmidrule(lr){7-8}
 &  &  & 48-IHM & LOS & 25-PHENO & MOR & RAD \\
\midrule
\multirow{6}{*}{AUROC}
 & \multirow{4}{*}{Single Task}
        & HighMMT & $0.809 \pm 0.003$ & \bbest{$0.830 \pm 0.004$} & \bbest{$0.724 \pm 0.002$} & \bbest{$0.850 \pm 0.001$} & $0.767 \pm 0.003$ \\
 &      & FuseMoE & $0.798 \pm 0.008$ & $0.819 \pm 0.004$ & $0.706 \pm 0.002$ & $0.846 \pm 0.003$ & $0.763 \pm 0.001$ \\
 &      & FlexMoE & $0.789 \pm 0.009$ & $0.817 \pm 0.004$ & $0.722 \pm 0.001$ & $0.849 \pm 0.003$ & \bbest{$0.768 \pm 0.003$} \\
 &      & FLAME   & \fC{$0.808 \pm 0.001$} & \fC{$0.816 \pm 0.005$} & \fB{$0.723 \pm 0.001$} & \fB{$0.846 \pm 0.002$} & \fC{$0.756 \pm 0.006$} \\
\cmidrule(lr){2-8}
 & Multitask (per dataset)  & FLAME & \fA{$0.817 \pm 0.005$} & \fC{$0.821 \pm 0.007$} & \fB{$0.719 \pm 0.012$} & \fC{$0.843 \pm 0.005$} & \fC{$0.758 \pm 0.001$} \\
 & Multitask (all datasets) & FLAME & \fB{$0.815 \pm 0.002$} & \fC{$0.821 \pm 0.006$} & \fB{$0.719 \pm 0.006$} & \fC{$0.835 \pm 0.005$} & \fC{$0.758 \pm 0.001$} \\
\midrule
\multirow{6}{*}{AUPRC}
 & \multirow{4}{*}{Single Task}
        & HighMMT & $0.480 \pm 0.019$ & \bbest{$0.748 \pm 0.008$} & \bbest{$0.485 \pm 0.004$} & \bbest{$0.300 \pm 0.007$} & $0.459 \pm 0.003$ \\
 &      & FuseMoE & $0.444 \pm 0.014$ & $0.737 \pm 0.003$ & $0.455 \pm 0.003$ & $0.293 \pm 0.007$ & $0.449 \pm 0.003$ \\
 &      & FlexMoE & $0.445 \pm 0.020$ & $0.726 \pm 0.010$ & $0.484 \pm 0.002$ & \bbest{$0.300 \pm 0.006$} & \bbest{$0.460 \pm 0.003$} \\
 &      & FLAME   & \fA{$0.492 \pm 0.009$} & \fD{$0.731 \pm 0.007$} & \fB{$0.481 \pm 0.002$} & \fC{$0.293 \pm 0.006$} & \fC{$0.449 \pm 0.011$} \\
\cmidrule(lr){2-8}
 & Multitask (per dataset)  & FLAME & \fD{$0.450 \pm 0.023$} & \fD{$0.731 \pm 0.012$} & \fC{$0.471 \pm 0.016$} & \fD{$0.279 \pm 0.004$} & \fC{$0.451 \pm 0.003$} \\
 & Multitask (all datasets) & FLAME & \fD{$0.470 \pm 0.018$} & \fB{$0.743 \pm 0.008$} & \fC{$0.471 \pm 0.004$} & \fD{$0.271 \pm 0.008$} & \fC{$0.451 \pm 0.001$} \\
\bottomrule
\end{tabular}}
\end{subtable}

\vspace{0.8em}

\begin{subtable}{\textwidth}
\centering
\caption{EMBED and ADNI}
\label{tab:full_single_task_embed_adni}
\resizebox{\textwidth}{!}{%
\begin{tabular}{lll ccc c}
\toprule
\multirow{2}{*}{Metric} & \multirow{2}{*}{Setting} & \multirow{2}{*}{Method} & \multicolumn{3}{c}{EMBED} & \multicolumn{1}{c}{ADNI} \\
\cmidrule(lr){4-6}\cmidrule(lr){7-7}
 &  &  & BIRADS & RISK & DENSITY & DIAG \\
\midrule
\multirow{6}{*}{AUROC}
 & \multirow{4}{*}{Single Task}
        & HighMMT & \bbest{$0.813 \pm 0.002$} & $0.733 \pm 0.004$ & \bbest{$0.927 \pm 0.001$} & $0.743 \pm 0.021$ \\
 &      & FuseMoE & $0.802 \pm 0.001$ & \bbest{$0.738 \pm 0.001$} & $0.920 \pm 0.003$ & $0.747 \pm 0.028$ \\
 &      & FlexMoE & $0.809 \pm 0.003$ & $0.735 \pm 0.003$ & $0.925 \pm 0.000$ & $0.702 \pm 0.025$ \\
 &      & FLAME   & \fA{$0.813 \pm 0.002$} & \fB{$0.736 \pm 0.001$} & \fB{$0.926 \pm 0.001$} & \fA{$0.788 \pm 0.023$} \\
\cmidrule(lr){2-7}
 & Multitask (per dataset)  & FLAME & \fB{$0.811 \pm 0.002$} & \fA{$0.738 \pm 0.005$} & \fA{$0.927 \pm 0.000$} & -- \\
 & Multitask (all datasets) & FLAME & \fD{$0.791 \pm 0.007$} & \fD{$0.714 \pm 0.011$} & \fC{$0.915 \pm 0.007$} & \fD{$0.748 \pm 0.006$} \\
\midrule
\multirow{6}{*}{AUPRC}
 & \multirow{4}{*}{Single Task}
        & HighMMT & $0.563 \pm 0.001$ & $0.142 \pm 0.002$ & \bbest{$0.765 \pm 0.003$} & $0.592 \pm 0.028$ \\
 &      & FuseMoE & $0.550 \pm 0.003$ & $0.151 \pm 0.008$ & $0.755 \pm 0.008$ & $0.601 \pm 0.040$ \\
 &      & FlexMoE & \bbest{$0.567 \pm 0.002$} & $0.142 \pm 0.003$ & $0.756 \pm 0.004$ & $0.549 \pm 0.022$ \\
 &      & FLAME   & \fB{$0.564 \pm 0.003$} & \fA{$0.153 \pm 0.010$} & \fC{$0.758 \pm 0.002$} & \fA{$0.660 \pm 0.039$} \\
\cmidrule(lr){2-7}
 & Multitask (per dataset)  & FLAME & \fC{$0.560 \pm 0.008$} & \fB{$0.151 \pm 0.005$} & \fB{$0.764 \pm 0.002$} & -- \\
 & Multitask (all datasets) & FLAME & \fD{$0.542 \pm 0.010$} & \fB{$0.148 \pm 0.012$} & \fD{$0.743 \pm 0.011$} & \fD{$0.611 \pm 0.018$} \\
\bottomrule
\end{tabular}}
\end{subtable}
\end{table}

\begin{figure}[t]
\centering
\includegraphics[width=\linewidth]{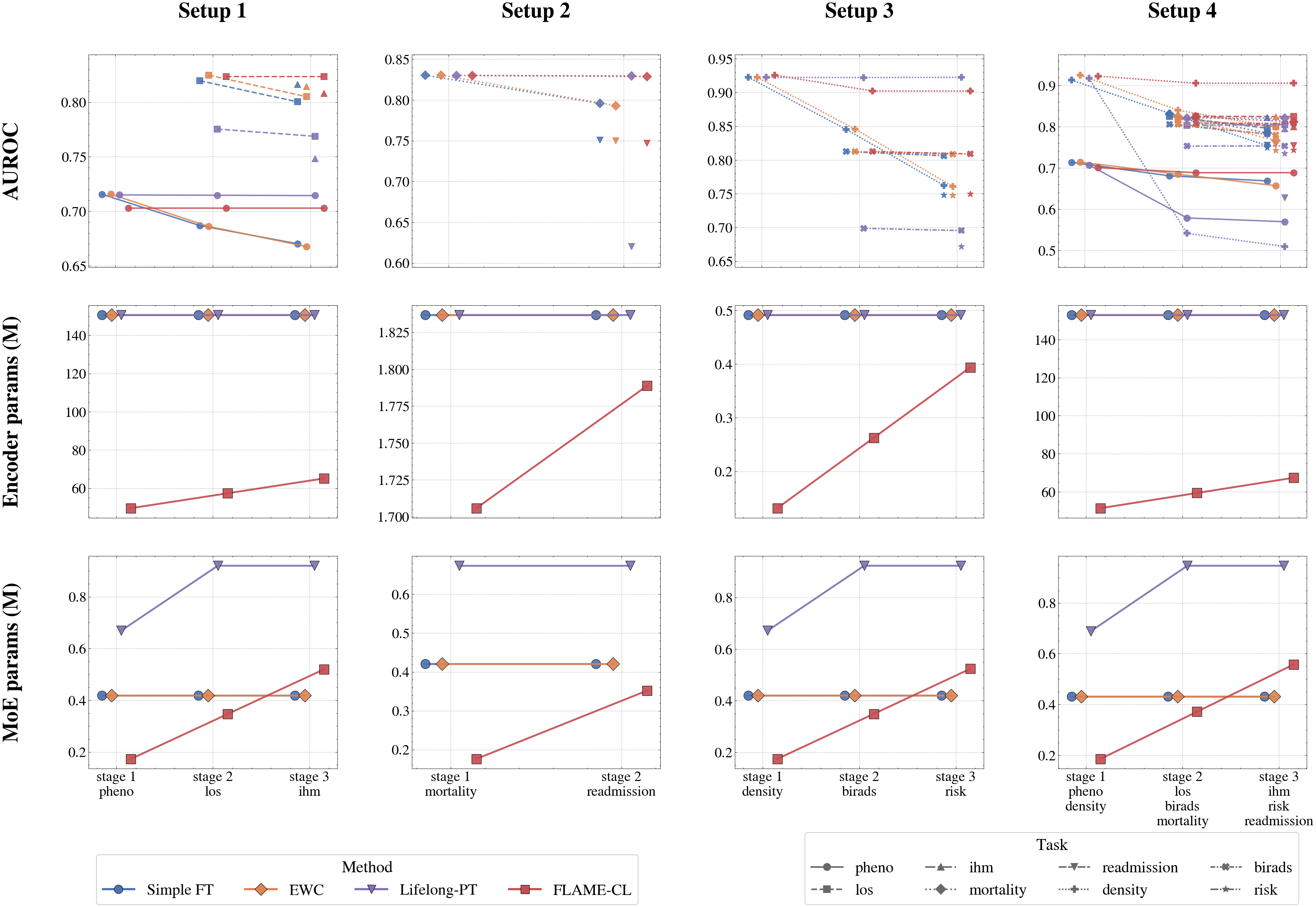}
\caption{Per-stage continual-learning trajectories across four task sequences (columns: Setup 1--4) and three metrics (rows: AUROC, encoder params, MoE params). Series: \{Simple FT, EWC, LoRA, FLAME-CL\}. AUPRC counterpart in Fig.~\ref{fig:cl-auprc}. We reserve 32 ranks after each stage. }
\label{fig:cl-grid}
\end{figure}

\subsection{Discussions}\label{subsec:discussion}

\textbf{(RQ1)} \flame{} matches task-specific multimodal baselines: Tables~\ref{tab:full_single_task_mimic_eicu},~\ref{tab:full_single_task_embed_adni} place it within $0.005$ AUROC of the best baseline on six of nine tasks and ahead on AUPRC for the imbalanced binaries (\textsc{48-IHM}, \textsc{RISK}, \textsc{DIAG}), with the all-nine-task regime trailing per-task models by under $0.02$ AUROC---the invariance to modality-pool composition that makes a single shared backbone deployable.
\textbf{(RQ2)} Tasks pair well exactly when their per-modality routing fingerprints overlap, and routing recovers this structure without supervision: the pairwise heatmap (Fig.~\ref{fig:confusion}) and per-task routing panels (Fig.~\ref{fig:routing-alljoint}) both group EMBED's \textsc{BIRADS}/\textsc{DENSITY}/\textsc{RISK}, eICU's \textsc{MOR}/\textsc{RAD}, and MIMIC-IV's \textsc{IHM}/\textsc{LOS} onto overlapping expert subsets that reinforce, while cross-dataset pairs occupy disjoint experts and gain nothing.
\textbf{(RQ3)} \flamecl{} absorbs the entire stream without forgetting, structurally rather than statistically: end-of-stage AUROC stays within $0.01$ of each task's first-introduction value across all four sequences (Fig.~\ref{fig:cl-grid}), whereas Simple FT drops $0.03$--$0.05$ on the earliest tasks, EWC slows but does not stop drift, and LoRA matches retention only by freezing the base; cursor-based inference plus rank reservation prevent later stages from perturbing prior-task forward passes by construction (Sec.~\ref{subsec:cl}).
\textbf{(RQ4)} \flamecl{} stores $5{-}15{\times}$ fewer encoder parameters than Simple FT, EWC, and LoRA at the same retention level: data-aware spectra saturate by rank~$20$--$40$ on every task while weight-only spectra remain near full rank, so rank reservation reuses functionally idle capacity instead of duplicating weights (Proposition~\ref{prop:funcrank}). Per-task tables, AUPRC trajectories, the Lifelong-PT comparison~\cite{chen2023lifelong}, and full discussion are in Appendix~\ref{app:extended_discussion}.

\section{Conclusion and Future Works}\label{sec:conclusion}

Real-world multimodal deployment couples flexi-modal multitask pretraining with continual absorption of new tasks, two regimes existing frameworks treat in isolation. \flame{} unifies them in a single sparse MoE: per-modality routers decouple gating from any task's modality subset, and a stacked, rank-truncated additive design with cursor-based inference makes no-forgetting structural rather than a tunable regularizer. Across diverse tasks, \flame{} matches task-specific multimodal baselines with a shared backbone while \flamecl{} retains end-of-stage AUROC at five- to fifteen-fold fewer encoder parameters than CL baselines.

\bibliography{reference}
\bibliographystyle{plain}

\newpage
\appendix
\onecolumn

\vspace{1cm}
\centering{\textbf{\Large{Supplementary Material for 
``FLAME: Adaptive Mixture-of-Experts for Continual Multimodal Multi-Task Learning''}}}

\justifying
\setlength{\parindent}{0pt}
\vspace{0.5cm}

\noindent\textbf{\large Appendix Contents}\vspace{0.4em}
\begin{itemize}\itemsep0.1em
\item \textbf{A.}~\hyperref[app:extended_related]{Extended Related Works}
\item \textbf{B.}~\hyperref[app:funcrank]{Proof of Proposition~\ref*{prop:funcrank}}
\item \textbf{C.}~\hyperref[app:datasets]{Dataset Details}
\item \textbf{D.}~\hyperref[app:additional_result]{Additional Experimental Results}
\begin{itemize}\itemsep0em
\item D.1~\hyperref[app:hardware]{Hardware and Implementation}
\item D.2~\hyperref[app:multitask_results]{Detailed Multitask Pretraining Results} --- full per-task table, pairwise interactions, number-of-experts ablation, gating-function ablation
\item D.3~\hyperref[app:cl_results]{Detailed Continual Learning Results} --- per-method walkthrough against Simple FT/EWC/LoRA, isolated FLAME-CL vs.\ LoRA comparison, comparison against Lifelong-PT, per-stage AUPRC trajectories
\item D.4~\hyperref[app:spectra]{Spectral Validation of Proposition~\ref*{prop:funcrank}} --- single-task and multi-task expert input spectra
\item D.5~\hyperref[app:routing]{Per-Modality Routing Distributions} --- single-task, multi-task, and continual-stage routing fingerprints
\item D.6~\hyperref[app:extended_discussion]{Extended Discussion of Research Questions} --- long-form RQ1--RQ4 analysis with interpretive framing
\end{itemize}
\item \textbf{E.}~\hyperref[app:method_details]{Additional Methodological Details}
\item \textbf{F.}~\hyperref[app:limitations]{Limitations}
\item \textbf{G.}~\hyperref[app:impact]{Broader Impact}
\end{itemize}
\vspace{0.5em}

\section{Extended Related Works}\label{app:extended_related}

\textbf{Multimodal Learning and MoEs.}~
Multimodal learning spans early fusion, late fusion, and cross-modal attention architectures \cite{ngiam2011multimodal, baltruvsaitis2018multimodal, xu2023multimodal}, with recent Transformer-based approaches enabling scalable joint modeling of heterogeneous inputs. MoE \cite{shazeer2017outrageously, fedus2022switch, han2022dynamic, jiang2024mixtral, akbarian2024quadratic, guo2025deepseek, nguyen2024expert} has emerged as a powerful paradigm for scaling model capacity without proportional increases in computation, by dynamically routing inputs through a sparse subset of specialized experts. 
It has achieved superior results in multimodal learning, where the multimodal embeddings are routed to experts based on the learned routing mechanisms. The inter and intra-modality relationships are captured during the routing process without the need to explicitly specify pairwise modality as in cross-attention.
LIMoE \cite{mustafa2022multimodal} pioneered large-scale multimodal MoE by processing images and text through a shared sparse transformer with contrastive learning, demonstrating natural expert specialization via entropy-based regularization. FuseMoE \cite{han2024fusemoe} addressed the flexi-modal setting with irregularly sampled and missing modalities, introducing a Laplace gating function with theoretical convergence guarantees superior to softmax routing. Building on this, Flex-MoE \cite{yun2024flex} proposed a missing modality bank and dual-router design to handle arbitrary modality availability at inference. Most recently, MERGE \cite{han2025guiding} extended multimodal MoE to the massively multimodal regime by using temporal redundancy, uniqueness, and synergy interactions to guide routing. Despite these advances, existing multimodal MoE methods operate under a fixed and consistent modality structure across inputs, making them fundamentally incompatible with the Flexi-Modal multi-task scenario.

\textbf{Multi-Task and Continual Learning.}~
Multi-task learning (MTL) \cite{caruana1997multitask, zhang2021survey, yu2024unleashing} is a principled approach to improve generalization by jointly optimizing across related tasks, leveraging shared representations to reduce sample complexity and improve model efficiency. MTL has been extended to complex settings involving gradient conflict resolution \cite{fifty2021efficiently}, task grouping, and parameter sharing strategies \cite{vandenhende2021multi, ruder2017overview}. Recent works have shown that sparse MoEs are effective at generalizing across multiple tasks and serve as robust multitask learners \cite{hendawy2023multi, gupta2022sparsely, hendawy2023multi}. Complementary to MTL, continual learning (CL) \cite{wang2024comprehensive, de2021continual} addresses the sequential task setting, where a model must acquire new knowledge over time without catastrophically forgetting previously learned tasks \cite{kirkpatrick2017overcoming}. Foundational approaches such as Elastic Weight Consolidation (EWC) \cite{kirkpatrick2017overcoming} and Progressive Neural Networks \cite{rusu2016progressive} have established the canonical trade-off between forward transfer and backward interference. 
A growing line of work has specifically examined the continual learning capacity of MoE. Lifelong-MoE \cite{chen2023lifelong} progressively expands experts and gating dimensions to accommodate new pretraining data distributions. MoRAL \cite{yang2024moral} couples MoE with LoRA, using low-rank adapters as experts so that lifelong adaptation of LLMs can be achieved with minimal trainable parameters. SEED \cite{rypesc2024divide} tackles class-incremental learning by maintaining a fixed-size expert ensemble and selectively fine-tuning only the single expert whose distributions overlap least with the new task.
Despite these advances, the capacity for multitask and continual learning in multimodal models remains substantially understudied. Many multimodal multitask frameworks typically assume a fixed and unified set of modalities across all tasks, and none have explored the CL capabilities of multimodal MoE. A key question is whether a flexi-modal MoE model trained on an initial set of tasks can rapidly adapt to new tasks with previously unseen modality combinations after deployment, without retraining or degrading performance on existing tasks.

\section{Proof of Proposition~\ref{prop:funcrank}}\label{app:funcrank}

\setcounter{proposition}{0}
\begin{proposition}
Let $W_i \in \mathbb{R}^{p \times d}$ be the weight matrix of expert $E_i$ and let
$C_i = \mathbb{E}_{\bm{z}}[\bm{z}\bm{z}^\top] \in \mathbb{R}^{d \times d}$
be the covariance of the inputs routed to $E_i$, with eigendecomposition
$C_i = \sum_{j=1}^{d} \lambda_j\, \bm{q}_j \bm{q}_j^\top$ where
$\lambda_1 \ge \lambda_2 \ge \cdots \ge \lambda_d \ge 0$.
Suppose $C_i$ has $\epsilon$-effective rank $r^*$, i.e.,
\begin{equation}
\sum_{j > r^*} \lambda_j \;\le\; \epsilon\, \mathrm{tr}(C_i).
\label{eq:eff-rank-assumption}
\end{equation}
Then the orthogonal-tail contribution to the functional energy satisfies
\begin{equation}
\sum_{j > r^*} \lambda_j\, \|W_i\, \bm{q}_j\|^2
\;\le\;
\epsilon\, \|W_i\|_{\mathrm{op}}^2\, \mathrm{tr}(C_i),
\label{eq:funcrank-bound}
\end{equation}
where $\|W_i\|_{\mathrm{op}} = \sigma_{i,1}$ denotes the spectral norm of $W_i$.
Consequently, for any input $\bm{z}$ drawn from the routed distribution,
the output $W_i \bm{z}$ is concentrated, in expected squared norm, in the
rank-$r^*$ subspace $W_i\,\mathrm{span}\{\bm{q}_1,\ldots,\bm{q}_{r^*}\}$,
regardless of $\mathrm{rank}(W_i)$.
\end{proposition}

\begin{proof}
We proceed in three steps: (i) decompose the total functional energy in
$C_i$'s eigenbasis, (ii) bound the orthogonal tail using the spectral norm
of $W_i$, and (iii) translate the bound into a statement about the output
$W_i \bm{z}$.

\paragraph{Step 1: Decomposition.}
Since $\{\bm{q}_j\}_{j=1}^{d}$ is an orthonormal eigenbasis of $C_i$, we may
write $C_i = \sum_{j} \lambda_j\, \bm{q}_j \bm{q}_j^\top$. The total
functional energy of $W_i$ on the routed input distribution is
\begin{align}
\mathcal{E}_i
&\;=\; \mathbb{E}_{\bm{z}}\big[\|W_i \bm{z}\|^2\big]
\;=\; \mathbb{E}_{\bm{z}}\big[\mathrm{tr}(W_i \bm{z}\bm{z}^\top W_i^\top)\big]
\;=\; \mathrm{tr}(W_i\, C_i\, W_i^\top) \nonumber \\
&\;=\; \mathrm{tr}\!\left( W_i \Big( \textstyle\sum_{j} \lambda_j\, \bm{q}_j \bm{q}_j^\top \Big) W_i^\top \right)
\;=\; \sum_{j=1}^{d} \lambda_j\, \mathrm{tr}(W_i \bm{q}_j \bm{q}_j^\top W_i^\top) \nonumber \\
&\;=\; \sum_{j=1}^{d} \lambda_j\, \|W_i \bm{q}_j\|^2,
\label{eq:proof-decomp}
\end{align}
where the last equality uses $\mathrm{tr}(\bm{a}\bm{a}^\top) = \|\bm{a}\|^2$
with $\bm{a} = W_i \bm{q}_j$. We split this sum at the cutoff $r^*$:
\begin{equation}
\mathcal{E}_i
\;=\; \underbrace{\sum_{j=1}^{r^*} \lambda_j\, \|W_i \bm{q}_j\|^2}_{\mathcal{E}_i^{\le r^*}}
\;+\; \underbrace{\sum_{j=r^*+1}^{d} \lambda_j\, \|W_i \bm{q}_j\|^2}_{\mathcal{E}_i^{> r^*}}.
\label{eq:proof-split}
\end{equation}

\paragraph{Step 2: Tail bound.}
For each $j$, the spectral norm of $W_i$ implies the operator inequality
$\|W_i \bm{q}_j\|^2 \le \|W_i\|_{\mathrm{op}}^2 \|\bm{q}_j\|^2 = \|W_i\|_{\mathrm{op}}^2$,
since $\bm{q}_j$ is a unit vector. Substituting into $\mathcal{E}_i^{> r^*}$,
\begin{equation}
\mathcal{E}_i^{> r^*}
\;=\; \sum_{j > r^*} \lambda_j\, \|W_i \bm{q}_j\|^2
\;\le\; \|W_i\|_{\mathrm{op}}^2 \sum_{j > r^*} \lambda_j
\;\le\; \epsilon\, \|W_i\|_{\mathrm{op}}^2\, \mathrm{tr}(C_i),
\label{eq:proof-tail}
\end{equation}
where the last inequality applies the effective-rank assumption
\eqref{eq:eff-rank-assumption}. This establishes \eqref{eq:funcrank-bound}.

\paragraph{Step 3: Output concentration.}
Let $\Pi_{r^*} = \sum_{j=1}^{r^*} \bm{q}_j \bm{q}_j^\top$ project onto the
top-$r^*$ eigensubspace of $C_i$, and $\Pi_{r^*}^{\perp} = I - \Pi_{r^*}$.
Decompose $\bm{z} = \Pi_{r^*}\bm{z} + \Pi_{r^*}^{\perp}\bm{z}$ and note that
the cross term vanishes in expectation because the eigenbasis diagonalizes
$C_i$:
\begin{equation}
\mathbb{E}\big[\|W_i \bm{z}\|^2\big]
\;=\; \mathbb{E}\big[\|W_i \Pi_{r^*} \bm{z}\|^2\big]
\;+\; \mathbb{E}\big[\|W_i \Pi_{r^*}^{\perp} \bm{z}\|^2\big],
\label{eq:proof-orth}
\end{equation}
since
$\mathbb{E}[(\Pi_{r^*}\bm{z})^\top W_i^\top W_i (\Pi_{r^*}^{\perp}\bm{z})]
= \mathrm{tr}(W_i^\top W_i \Pi_{r^*}^{\perp} C_i \Pi_{r^*}) = 0$.
The same Step-1 calculation applied to each piece yields
$\mathbb{E}[\|W_i \Pi_{r^*}^{\perp} \bm{z}\|^2] = \mathcal{E}_i^{> r^*}$.
Combining \eqref{eq:proof-orth} with \eqref{eq:proof-tail},
\begin{equation}
\mathbb{E}\big[\|W_i \Pi_{r^*}^{\perp} \bm{z}\|^2\big]
\;\le\; \epsilon\, \|W_i\|_{\mathrm{op}}^2\, \mathrm{tr}(C_i),
\end{equation}
which is independent of $\mathrm{rank}(W_i)$. Hence the output
$W_i \bm{z}$ is concentrated, in expected squared norm, in
$W_i\,\mathrm{span}\{\bm{q}_1,\ldots,\bm{q}_{r^*}\}$.
\end{proof}

\begin{remark}\label{rem:relative}
Dividing both sides of \eqref{eq:funcrank-bound} by $\mathcal{E}_i$ gives
the relative tail energy
$\mathcal{E}_i^{> r^*} / \mathcal{E}_i \le \epsilon\, \kappa_i$,
where $\kappa_i = \|W_i\|_{\mathrm{op}}^2\, \mathrm{tr}(C_i) / \mathcal{E}_i$
measures the misalignment between $W_i$'s top singular directions and
$C_i$'s top eigendirections; $\kappa_i = 1$ in the perfectly aligned case
and grows when $W_i$ wastes capacity on low-energy directions. Empirically,
gradient training drives $\kappa_i$ close to $1$, giving the sharp
saturation observed in Fig.~\ref{fig:functional_energy}.
\end{remark}

\begin{remark}\label{rem:compression}
Proposition~\ref{prop:funcrank} underwrites the spectral truncation in
Eq.~\eqref{eq:compress}. The truncation in \flamecl{} is applied not to
the full expert weight $W_i$ but to the stage-$t$ additive component
$\widetilde{W}_i^{(t)}$, whose role is to encode the new task on top of
the already-frozen base $W_i^{(0)} + \sum_{j<t} W_i^{(j)}$. Applying
Proposition~\ref{prop:funcrank} to $\widetilde{W}_i^{(t)}$, with $C_i$
retaining its meaning as the routed-input covariance under stage-$t$ data,
shows that the functional energy of $\widetilde{W}_i^{(t)}$ concentrates
in a low-rank subspace of its routed-input distribution. The post-training
alignment of Remark~\ref{rem:relative} further implies that low-singular-value
directions of $\widetilde{W}_i^{(t)}$ correspond to low-functional-energy
directions, so the rank-$r_t$ truncation defining $W_i^{(t)}$ is approximately
functionally lossless on the stage-$t$ distribution while reducing the
per-stage memory footprint to $r_t(p+d+1)$ scalars.
\end{remark}

\subsection{Alignment and Sample Complexity for the Spectral Truncation}\label{app:alignment}

Proposition~\ref{prop:funcrank} establishes input-side concentration: the output $W_i \bm{z}$ lives in the rank-$r^*$ eigensubspace of $C_i$. The truncation in Eq.~\eqref{eq:compress} is applied weight-side, however, to the singular vectors of $\widetilde{W}_i^{(t)}$. We close that gap with two complementary results: an alignment argument under gradient flow, showing that with the zero initialization used in Eq.~\eqref{eq:compress} the right singular vectors of $\widetilde{W}_i^{(t)}$ coincide with the eigenvectors of $C_i$ (Proposition~\ref{prop:align}), and a sample-complexity bound, showing that the resulting truncation error is faithfully estimated from finite stage-$t$ data (Proposition~\ref{prop:gen}). Both results study a linear surrogate of the expert nonlinearity; their conclusions are corroborated empirically across all benchmarks in Sec.~\ref{app:spectra}.

\paragraph{Linear quadratic surrogate.}
For the alignment argument we replace the expert by its linear part $\bm{z}\mapsto W\bm{z}$ and consider continuous-time gradient flow on the population least-squares objective
\begin{equation}
\mathcal{L}(W) \;=\; \tfrac{1}{2}\,\mathbb{E}_{\bm{z}\sim p_i}\big\|W^\star \bm{z} - W\bm{z}\big\|^2
\;=\; \tfrac{1}{2}\,\mathrm{tr}\!\big((W - W^\star)\, C_i\, (W - W^\star)^\top\big),
\label{eq:linear-surrogate}
\end{equation}
with stage-$t$ optimum $W^\star \in \mathbb{R}^{p\times d}$ and routed-input distribution $p_i$ of covariance $C_i$. The flow $\dot W_t = -\nabla_W \mathcal{L}(W_t) = (W^\star - W_t)\, C_i$ models the training dynamics of $\widetilde{W}_i^{(t)}$ on stage-$t$ data, with $W^\star$ the optimal linearization at convergence; the linear-surrogate assumption is standard in implicit-bias analyses~\cite{saxe2013exact, gunasekar2018implicit, arora2019implicit}.

\setcounter{proposition}{1}
\begin{proposition}[Alignment under zero-initialized gradient flow]\label{prop:align}
Let $W_t$ evolve under the gradient flow on~\eqref{eq:linear-surrogate} with $W_0 = 0$, matching the zero initialization of $\widetilde{W}_i^{(t)}$ in Eq.~\eqref{eq:compress}. Then for every $t \ge 0$,
\begin{equation}
W_t \;=\; W^\star\,\big(I - e^{-C_i t}\big),
\qquad
W_t \bm{q}_j \;=\; \big(1 - e^{-\lambda_j t}\big)\, W^\star \bm{q}_j.
\label{eq:align-closed-form}
\end{equation}
\textbf{(i) Convergence on the effective subspace.} For any index set $J \subseteq [d]$ with $\lambda_J := \min_{j \in J}\lambda_j > 0$, writing $Q_J = [\bm{q}_j]_{j \in J}$,
\begin{equation}
\big\|(W_t - W^\star)\, Q_J\big\|_F^2
\;=\; \sum_{j \in J} \|W^\star \bm{q}_j\|^2\, e^{-2\lambda_j t}
\;\leq\; \|W^\star\|_F^2\, e^{-2\lambda_J t}.
\label{eq:align-conv}
\end{equation}
\textbf{(ii) Exact SVD alignment under commutativity.} If $W^{\star\top} W^\star$ commutes with $C_i$, then for every $t \ge 0$,
\begin{equation}
W_t^\top W_t \;=\; \sum_{j=1}^{d} \big(1 - e^{-\lambda_j t}\big)^{2}\, \sigma_j^{\star 2}\, \bm{q}_j \bm{q}_j^\top,
\qquad \sigma_j^{\star 2} := \bm{q}_j^\top W^{\star\top} W^\star \bm{q}_j,
\label{eq:align-svd}
\end{equation}
so the right singular vectors of $W_t$ coincide with the eigenvectors $\{\bm{q}_j\}$ of $C_i$, with squared singular values $(1 - e^{-\lambda_j t})^2 \sigma_j^{\star 2}$.
\end{proposition}

\begin{proof}
\emph{Closed form.} Let $X_t := W_t - W^\star$, so $\dot X_t = -X_t C_i$ with $X_0 = -W^\star$. The matrix-valued linear ODE has solution $X_t = X_0\, e^{-C_i t} = -W^\star\, e^{-C_i t}$, hence $W_t = W^\star (I - e^{-C_i t})$. Right-multiplying by $\bm{q}_j$ and using $C_i \bm{q}_j = \lambda_j \bm{q}_j$ gives~\eqref{eq:align-closed-form}.

\emph{(i) Convergence.} From~\eqref{eq:align-closed-form}, $(W_t - W^\star)\bm{q}_j = -e^{-\lambda_j t}\, W^\star \bm{q}_j$. Squaring and summing over $j \in J$,
\[
\|(W_t - W^\star) Q_J\|_F^2 \;=\; \sum_{j \in J} \|W^\star \bm{q}_j\|^2\, e^{-2\lambda_j t} \;\leq\; \Big(\sum_{j \in J} \|W^\star \bm{q}_j\|^2\Big)\, e^{-2\lambda_J t} \;\leq\; \|W^\star\|_F^2\, e^{-2\lambda_J t},
\]
using $e^{-2\lambda_j t} \leq e^{-2\lambda_J t}$ for $j \in J$.

\emph{(ii) Alignment.} From the closed form, $W_t^\top W_t = (I - e^{-C_i t})\, W^{\star\top} W^\star\, (I - e^{-C_i t})$. The factor $I - e^{-C_i t}$ is a function of $C_i$, hence diagonal in $\{\bm{q}_j\}$ with entries $(1 - e^{-\lambda_j t})$. Under the commutativity assumption $[W^{\star\top}W^\star, C_i] = 0$, both matrices are simultaneously diagonalizable in $\{\bm{q}_j\}$, so $W^{\star\top} W^\star = \sum_j \sigma_j^{\star 2} \bm{q}_j \bm{q}_j^\top$, giving~\eqref{eq:align-svd}.
\end{proof}

\begin{remark}\label{rem:align-discrete}
The closed form~\eqref{eq:align-closed-form} is the $\eta \to 0$ limit of SGD with step size $\eta$. Discrete-step convergence at rate $(1 - \eta\lambda_K)^{2t}$ for $\eta < 1/\lambda_1$ follows by standard arguments on linear gradient dynamics~\cite{gidel2019implicit}. Relaxing the commutativity assumption in~(ii) incurs a Davis-Kahan correction~\cite{davis1970rotation, yu2015useful} of order $\|W_t^\top W_t - W^{\star\top} W^\star\|_F / \delta_K^\star$, where $\delta_K^\star$ is the squared-singular-value gap of $W^\star$; this perturbation itself decays exponentially at rate $\lambda_{r^*}$ on the effective subspace, so the perfectly aligned conclusion of~(ii) is the limit of approximate alignment statements that hold whenever the commutator $[W^{\star\top}W^\star, C_i]$ is small.
\end{remark}

\begin{corollary}[Functional truncation error under alignment]\label{cor:trunc}
Suppose Propositions~\ref{prop:funcrank} (with $\epsilon$-effective rank $r^*$) and~\ref{prop:align}\,(ii) hold, and let $W_t^{(K)}$ denote the rank-$K$ SVD truncation of $W_t$ for some $K \ge r^*$ chosen so that the top-$K$ right singular subspace of $W_t$ contains $\mathrm{span}\{\bm{q}_1,\ldots,\bm{q}_{r^*}\}$. Then
\begin{equation}
\mathbb{E}_{\bm{z}\sim p_i}\big\|W_t \bm{z} - W_t^{(K)} \bm{z}\big\|^2
\;\leq\; \epsilon\, \|W_t\|_{\mathrm{op}}^2\, \mathrm{tr}(C_i).
\label{eq:cor-trunc}
\end{equation}
\end{corollary}

\begin{proof}
By Proposition~\ref{prop:align}\,(ii), the right singular vectors of $W_t$ are exactly $\{\bm{q}_j\}_{j=1}^{d}$. Let $\mathcal{J}_K \subseteq [d]$ index the top-$K$ right singular directions of $W_t$, so that $W_t^{(K)} \bm{z} = \sum_{j \in \mathcal{J}_K} (1 - e^{-\lambda_j t})\, W^\star \bm{q}_j\, (\bm{q}_j^\top \bm{z})$. By assumption $[r^*] \subseteq \mathcal{J}_K$, so the discarded directions all lie outside the top-$r^*$ eigensubspace of $C_i$. Applying the Step-1 decomposition of Proposition~\ref{prop:funcrank}'s proof to the residual $W_t - W_t^{(K)}$,
\[
\mathbb{E}\|W_t \bm{z} - W_t^{(K)} \bm{z}\|^2
\;=\; \sum_{j \notin \mathcal{J}_K} \lambda_j\, \|W_t \bm{q}_j\|^2
\;\leq\; \sum_{j > r^*} \lambda_j\, \|W_t \bm{q}_j\|^2
\;\leq\; \epsilon\, \|W_t\|_{\mathrm{op}}^2\, \mathrm{tr}(C_i),
\]
where the last inequality is the bound \eqref{eq:funcrank-bound} of Proposition~\ref{prop:funcrank}.
\end{proof}

\begin{proposition}[Sample complexity for the empirical truncation error]\label{prop:gen}
Fix $W \in \mathbb{R}^{p\times d}$ and its rank-$K$ SVD truncation $W^{(K)}$, and write $W^{(>K)} := W - W^{(K)}$. Suppose $\bm{z} \in \mathbb{R}^d$ is mean-zero sub-Gaussian with parameter $\nu$, i.e.\ $\mathbb{E}\big[\exp(\langle \bm{u}, \bm{z}\rangle)\big] \leq \exp(\nu^2 \|\bm{u}\|^2 / 2)$ for all $\bm{u} \in \mathbb{R}^d$. Let $\{\bm{z}_i\}_{i=1}^n$ be iid samples from $p_i$ and define the population and empirical truncation errors
\[
\mathcal{T} \;:=\; \mathbb{E}_{\bm{z}\sim p_i}\big\|W^{(>K)}\bm{z}\big\|^2,
\qquad
\widehat{\mathcal{T}}_n \;:=\; \tfrac{1}{n}\sum_{i=1}^n \big\|W^{(>K)}\bm{z}_i\big\|^2,
\]
together with the effective rank $d_{\mathrm{eff}}(C_i) := \mathrm{tr}(C_i) / \|C_i\|_{\mathrm{op}}$. There exists an absolute constant $c > 0$ such that for any $\delta \in (0,1)$ and any $n \geq d_{\mathrm{eff}}(C_i) + \log(2/\delta)$, with probability at least $1 - \delta$,
\begin{equation}
\big| \widehat{\mathcal{T}}_n - \mathcal{T} \big|
\;\leq\;
c\, \nu^2\, \big\|W^{(>K)}\big\|_F^2\, \sqrt{\frac{d_{\mathrm{eff}}(C_i) + \log(2/\delta)}{n}}.
\label{eq:gen-bound}
\end{equation}
\end{proposition}

\begin{proof}
Write $\mathcal{T} = \mathrm{tr}\big(W^{(>K)}\, C_i\, W^{(>K)\top}\big)$ and $\widehat{\mathcal{T}}_n = \mathrm{tr}\big(W^{(>K)}\, \widehat{C}_n\, W^{(>K)\top}\big)$, where $\widehat{C}_n = \tfrac{1}{n}\sum_i \bm{z}_i \bm{z}_i^\top$. Using $|\mathrm{tr}(MN)| \leq \|M\|_{\mathrm{op}}\, \|N\|_{*}$ with $N = W^{(>K)\top} W^{(>K)} \succeq 0$ and $\|N\|_* = \mathrm{tr}(N) = \|W^{(>K)}\|_F^2$,
\begin{equation}
\big| \widehat{\mathcal{T}}_n - \mathcal{T} \big|
\;=\; \big|\mathrm{tr}\!\big((\widehat{C}_n - C_i)\, W^{(>K)\top} W^{(>K)}\big)\big|
\;\leq\; \big\|\widehat{C}_n - C_i\big\|_{\mathrm{op}}\, \big\|W^{(>K)}\big\|_F^2.
\label{eq:gen-trace}
\end{equation}
Sub-Gaussianity gives matrix-Bernstein concentration of $\widehat{C}_n$ around $C_i$~\cite{tropp2015introduction, vershynin2018high, koltchinskii2017concentration}: there exists $c' > 0$ such that, with probability at least $1 - \delta$,
\[
\big\|\widehat{C}_n - C_i\big\|_{\mathrm{op}}
\;\leq\; c'\, \nu^2 \!\left( \sqrt{\tfrac{d_{\mathrm{eff}}(C_i) + \log(2/\delta)}{n}} + \tfrac{d_{\mathrm{eff}}(C_i) + \log(2/\delta)}{n} \right).
\]
For $n \geq d_{\mathrm{eff}}(C_i) + \log(2/\delta)$ the first term dominates, so $\|\widehat{C}_n - C_i\|_{\mathrm{op}} \leq 2 c'\, \nu^2 \sqrt{(d_{\mathrm{eff}}(C_i) + \log(2/\delta))/n}$. Substituting into \eqref{eq:gen-trace} and absorbing the constant yields \eqref{eq:gen-bound}.
\end{proof}

\begin{remark}\label{rem:effective-rank}
The bound \eqref{eq:gen-bound} scales with $d_{\mathrm{eff}}(C_i)$, not the ambient dimension $d$. When $C_i$ is sharply low-rank---empirically the case for FLAME's experts on every benchmark in Sec.~\ref{app:spectra}---we have $d_{\mathrm{eff}}(C_i) \approx r^*$, so $n = O(r^* / \varepsilon^2)$ samples suffice to estimate $\mathcal{T}$ to additive accuracy $\varepsilon \cdot \|W^{(>K)}\|_F^2$. The same low-rank structure that motivates the truncation in Proposition~\ref{prop:funcrank} also controls the sample complexity of validating it from finite data.
\end{remark}

\begin{remark}\label{rem:nonlinear}
Propositions~\ref{prop:align} and~\ref{prop:gen} are stated for the linear surrogate~\eqref{eq:linear-surrogate}; FLAME's experts add a length-preserving 1-D convolution and a position-wise MLP. NTK-regime analysis~\cite{jacot2018neural, atanasov2022silent} extends Proposition~\ref{prop:align} to wide nonlinear networks: in the lazy regime the network behaves as a linear predictor in NTK feature space, and the alignment argument applies to the NTK feature map. We invoke both propositions as motivation for the truncation in Eq.~\eqref{eq:compress}; the spectral plots in Fig.~\ref{fig:functional_energy} and Sec.~\ref{app:spectra} provide direct empirical confirmation that the conclusions hold for the nonlinear experts on every benchmark in our suite.
\end{remark}

\section{Dataset Details}\label{app:datasets}

We provide additional details on the four healthcare datasets used in our experiments, including cohort construction, label distributions, modality preprocessing, and train/validation/test splits. All datasets are publicly available under their respective data use agreements.

\textbf{MIMIC-IV.} MIMIC-IV~\cite{johnson2023mimic} is a de-identified critical care database collected at the Beth Israel Deaconess Medical Center, a single tertiary academic medical center. We follow the cohort construction and preprocessing pipeline of FuseMoE~\cite{han2024fusemoe}, which adapts the multi-task ICU benchmark of \cite{harutyunyan2019multitask} to MIMIC-IV. ICU stays are restricted to adult patients and the first stay per admission. Vital signs and laboratory measurements are assembled into an irregularly sampled multivariate time series and discretized for the multi-time attention encoder. Clinical notes within the prediction window are concatenated and tokenized for a clinical text encoder. Chest X-rays linked through MIMIC-CXR~\cite{johnson2024mimic} are matched at the stay level; stays missing imaging are retained with the modality marked absent. \emph{48-IHM} predicts in-hospital mortality from the first 48 hours of the stay, \emph{LOS} predicts whether the remaining stay exceeds seven days, and \emph{25-PHENO} predicts a 25-way multi-label phenotype indicator over acute care conditions. We adopt the standard 70/10/20 train/validation/test split with patient-level disjointness.

\textbf{eICU.} The eICU Collaborative Research Database (v2.0)~\cite{pollard2019eicu} comprises de-identified records from 208 U.S. hospitals and ${\sim}139{,}367$ unique patients, providing inter-institutional variation in coding practices and care patterns that complements MIMIC-IV's single-site cohort and exercises the routing under realistic distributional shift. We mirror the MIMIC-IV preprocessing pipeline for vital and laboratory time series and clinical notes. \emph{MOR} predicts in-ICU mortality and \emph{RAD} predicts 30-day all-cause readmission; both tasks use the standard cohort filters (adult patients, length of stay $\ge 24$ hours) and a 70/10/20 patient-level split.

\textbf{EMBED.} The Emory Breast Imaging Dataset~\cite{Jeong2023embed} contains 3.4M screening and diagnostic mammographic images across a racially diverse population. Each breast study yields four image streams---FFDM CC, FFDM MLO, synthesized 2D C-View CC, and synthesized 2D C-View MLO---which are resized and normalized for a shared image encoder. Tabular metadata (demographics, prior history) is one-hot/standardized. \emph{BIRADS} is a multi-class classification over BI-RADS final assessment categories, \emph{RISK} is a binary short-term cancer risk label, and \emph{DENSITY} is the multi-class ACR breast density category. Splits are constructed at the patient level to prevent leakage across breast laterality.

\textbf{ADNI.} The Alzheimer's Disease Neuroimaging Initiative cohort~\cite{Weiner2016adni} provides up to five modalities per subject: structural MRI, FDG-PET, genomic SNP profiles, biospecimen assays (CSF/blood biomarkers), and tabular clinical and cognitive assessments. Following Flex-MoE~\cite{yun2024flex}, we form modality-stratified subsets ranging from $757$ subjects with all five modalities to $5{,}073$ subjects with the most populous subset, and use the same five-modality combination at training time with the rest of the modalities marked absent. \emph{DIAG} is a three-way classification across Cognitively Normal (CN), Mild Cognitive Impairment (MCI), and Alzheimer's Dementia (AD), with a 70/10/20 subject-level split.

\section{Additional Experimental Results}\label{app:additional_result}

This section supplements the main-text experiments with full per-task tables, ablations, an additional CL baseline, spectral diagnostics, and routing visualizations. Subsections mirror the four research questions of Sec.~\ref{sec:exp}: Sec.~\ref{app:multitask_results} expands on multitask pretraining (RQ1, RQ2), Sec.~\ref{app:cl_results} on continual learning (RQ3, RQ4), and Secs.~\ref{app:spectra}, \ref{app:routing} provide the spectral and routing evidence that explains the multitask gains and the parameter savings (RQ2, RQ4).

\subsection{Hardware and Implementation}\label{app:hardware}
All experiments are run on $8\times$ NVIDIA A5500 GPUs with $24$~GB memory per GPU.

\subsection{Detailed Multitask Pretraining Results}\label{app:multitask_results}

\subsubsection{Full Per-Task Performance}\label{app:full_table}

Tab.~\ref{tab:full_single_task_appendix} extends the main-text per-task tables (Tabs.~\ref{tab:full_single_task_mimic_eicu},~\ref{tab:full_single_task_embed_adni}) along three axes: (i) it adds F1 and accuracy alongside AUROC and AUPRC, (ii) it adds FuseMoE~\cite{han2024fusemoe} as a fourth single-task baseline, and (iii) it reports both the per-dataset multitask FLAME run and the all-nine-task multitask FLAME run alongside the single-task numbers. Across all four metrics, FLAME's single-task numbers track HighMMT within $0.005$ AUROC on six of nine tasks and exceed it on AUPRC for the imbalanced binary tasks, confirming that the AUROC story in the main text is not a metric-specific artifact.

\textbf{Single-task baselines.} FLAME trails HighMMT by $0.011$ to $0.014$ AUROC on \textsc{LOS} and \textsc{RAD} but exceeds it on ADNI \textsc{DIAG} by $0.045$ ($0.788$ vs $0.743$). On AUPRC FLAME wins on the imbalanced binary tasks where positives are rare: \textsc{48-IHM} ($0.492$ vs $0.480$), \textsc{RISK} ($0.153$ vs $0.142$), and \textsc{DIAG} ($0.660$ vs $0.592$); the win on the imbalanced tail is the architectural signature of modality-conditioned dispatch sharpening the boundary on rare positives.

\textbf{Joint training is task-dependent.} Per-dataset MIMIC-IV multitask raises \textsc{48-IHM} AUROC from $0.808$ to $0.817$ and \textsc{LOS} to $0.821$ but drops \textsc{48-IHM} AUPRC by $0.042$, indicating that the AUROC gain comes from re-ranking near the operating threshold rather than from improved positive recall. Training all nine tasks at once preserves the binary-task AUROC gains and gives back $0.040$ AUROC on \textsc{DIAG}, the only task whose modality pool (MRI, FDG-PET, SNP, biospecimen, clinical) shares no encoder with any other. Tasks that share modalities pay no measurable cost for joint training; tasks with disjoint modality combinations regress modestly when squeezed into a single shared expert pool.

\begin{table}[h]
\centering
\caption{Per-task performance (mean $\pm$ std across 3 seeds, population std). Multi-class/multi-label tasks use macro-averaged AUROC/AUPRC/F1; PHENO accuracy is Hamming. We fix number of experts to 5 for all the MoE methods. Ablation on number of experts is provided in \cref{fig:expert_ablation_auprc}.}
\resizebox{\textwidth}{!}{%
\begin{tabular}{l|l|l|ccc|cc|ccc|c}
\hline
\multirow{2}{*}{Metric} & \multirow{2}{*}{Setting} & \multirow{2}{*}{Method} & \multicolumn{3}{c|}{MIMIC IV} & \multicolumn{2}{c|}{eICU} & \multicolumn{3}{c|}{EMBED} & \multicolumn{1}{c}{ADNI} \\
\cline{4-6}\cline{7-8}\cline{9-11}\cline{12-12}
 &  &  & 48-IHM & LOS & 25-PHENO & MOR & RAD & BIRADS & RISK & DENSITY & DIAG \\
\hline
\multirow{5}{*}{AUROC} & \multirow{3}{*}{Single Task} & HighMMT & $0.809 \pm 0.003$ & $0.830 \pm 0.004$ & $0.724 \pm 0.002$ & $0.850 \pm 0.001$ & $0.767 \pm 0.003$ & $0.813 \pm 0.002$ & $0.733 \pm 0.004$ & $0.927 \pm 0.001$ & $0.743 \pm 0.021$ \\
 &  & FuseMoE & $0.798 \pm 0.008$ & $0.819 \pm 0.004$ & $0.706 \pm 0.002$ & $0.846 \pm 0.003$ & $0.763 \pm 0.001$ & $0.802 \pm 0.001$ & $0.738 \pm 0.001$ & $0.920 \pm 0.003$ & $0.747 \pm 0.028$ \\
 &  & FlexMoE & $0.789 \pm 0.009$ & $0.817 \pm 0.004$ & $0.722 \pm 0.001$ & $0.849 \pm 0.003$ & $0.768 \pm 0.003$ & $0.809 \pm 0.003$ & $0.735 \pm 0.003$ & $0.925 \pm 0.000$ & $0.702 \pm 0.025$ \\
 &  & FLAME & $0.808 \pm 0.001$ & $0.816 \pm 0.005$ & $0.723 \pm 0.001$ & $0.846 \pm 0.002$ & $0.756 \pm 0.006$ & $0.813 \pm 0.002$ & $0.736 \pm 0.001$ & $0.926 \pm 0.001$ & $0.788 \pm 0.023$ \\
\cline{2-12}
 & Multitask (per dataset) & FLAME & $0.817 \pm 0.005$ & $0.821 \pm 0.007$ & $0.719 \pm 0.012$ & $0.843 \pm 0.005$ & $0.758 \pm 0.001$ & $0.811 \pm 0.002$ & $0.738 \pm 0.005$ & $0.927 \pm 0.000$ & -- \\
\cline{2-12}
 & Multitask (all datasets) & FLAME & $0.815 \pm 0.002$ & $0.821 \pm 0.006$ & $0.719 \pm 0.006$ & $0.835 \pm 0.005$ & $0.758 \pm 0.001$ & $0.791 \pm 0.007$ & $0.714 \pm 0.011$ & $0.915 \pm 0.007$ & $0.748 \pm 0.006$ \\
\hline
\multirow{5}{*}{AUPRC} & \multirow{3}{*}{Single Task} & HighMMT & $0.480 \pm 0.019$ & $0.748 \pm 0.008$ & $0.485 \pm 0.004$ & $0.300 \pm 0.007$ & $0.459 \pm 0.003$ & $0.563 \pm 0.001$ & $0.142 \pm 0.002$ & $0.765 \pm 0.003$ & $0.592 \pm 0.028$ \\
 &  & FuseMoE & $0.444 \pm 0.014$ & $0.737 \pm 0.003$ & $0.455 \pm 0.003$ & $0.293 \pm 0.007$ & $0.449 \pm 0.003$ & $0.550 \pm 0.003$ & $0.151 \pm 0.008$ & $0.755 \pm 0.008$ & $0.601 \pm 0.040$ \\
 &  & FlexMoE & $0.445 \pm 0.020$ & $0.726 \pm 0.010$ & $0.484 \pm 0.002$ & $0.300 \pm 0.006$ & $0.460 \pm 0.003$ & $0.567 \pm 0.002$ & $0.142 \pm 0.003$ & $0.756 \pm 0.004$ & $0.549 \pm 0.022$ \\
 &  & FLAME & $0.492 \pm 0.009$ & $0.731 \pm 0.007$ & $0.481 \pm 0.002$ & $0.293 \pm 0.006$ & $0.449 \pm 0.011$ & $0.564 \pm 0.003$ & $0.153 \pm 0.010$ & $0.758 \pm 0.002$ & $0.660 \pm 0.039$ \\
\cline{2-12}
 & Multitask (per dataset) & FLAME & $0.450 \pm 0.023$ & $0.731 \pm 0.012$ & $0.471 \pm 0.016$ & $0.279 \pm 0.004$ & $0.451 \pm 0.003$ & $0.560 \pm 0.008$ & $0.151 \pm 0.005$ & $0.764 \pm 0.002$ & -- \\
\cline{2-12}
 & Multitask (all datasets) & FLAME & $0.470 \pm 0.018$ & $0.743 \pm 0.008$ & $0.471 \pm 0.004$ & $0.271 \pm 0.008$ & $0.451 \pm 0.001$ & $0.542 \pm 0.010$ & $0.148 \pm 0.012$ & $0.743 \pm 0.011$ & $0.611 \pm 0.018$ \\
\hline
\multirow{5}{*}{F1} & \multirow{3}{*}{Single Task} & HighMMT & $0.386 \pm 0.078$ & $0.722 \pm 0.005$ & $0.354 \pm 0.013$ & $0.310 \pm 0.036$ & $0.432 \pm 0.043$ & $0.501 \pm 0.006$ & $0.150 \pm 0.039$ & $0.700 \pm 0.008$ & $0.549 \pm 0.014$ \\
 &  & FuseMoE & $0.323 \pm 0.053$ & $0.733 \pm 0.010$ & $0.326 \pm 0.002$ & $0.304 \pm 0.015$ & $0.468 \pm 0.006$ & $0.518 \pm 0.013$ & $0.187 \pm 0.035$ & $0.694 \pm 0.008$ & $0.520 \pm 0.031$ \\
 &  & FlexMoE & $0.324 \pm 0.103$ & $0.702 \pm 0.012$ & $0.370 \pm 0.013$ & $0.188 \pm 0.041$ & $0.443 \pm 0.027$ & $0.510 \pm 0.017$ & $0.120 \pm 0.015$ & $0.702 \pm 0.004$ & $0.370 \pm 0.123$ \\
 &  & FLAME & $0.405 \pm 0.028$ & $0.710 \pm 0.023$ & $0.316 \pm 0.017$ & $0.229 \pm 0.029$ & $0.388 \pm 0.035$ & $0.506 \pm 0.018$ & $0.171 \pm 0.009$ & $0.704 \pm 0.002$ & $0.576 \pm 0.026$ \\
\cline{2-12}
 & Multitask (per dataset) & FLAME & $0.321 \pm 0.106$ & $0.717 \pm 0.055$ & $0.306 \pm 0.008$ & $0.147 \pm 0.053$ & $0.398 \pm 0.028$ & $0.533 \pm 0.011$ & $0.145 \pm 0.009$ & $0.701 \pm 0.003$ & -- \\
\cline{2-12}
 & Multitask (all datasets) & FLAME & $0.209 \pm 0.050$ & $0.704 \pm 0.030$ & $0.292 \pm 0.001$ & $0.164 \pm 0.050$ & $0.431 \pm 0.012$ & $0.489 \pm 0.019$ & $0.096 \pm 0.014$ & $0.677 \pm 0.012$ & $0.478 \pm 0.065$ \\
\hline
\multirow{5}{*}{Accuracy} & \multirow{3}{*}{Single Task} & HighMMT & $0.853 \pm 0.008$ & $0.753 \pm 0.004$ & $0.768 \pm 0.008$ & $0.954 \pm 0.004$ & $0.865 \pm 0.001$ & $0.729 \pm 0.018$ & $0.743 \pm 0.145$ & $0.760 \pm 0.001$ & $0.556 \pm 0.015$ \\
 &  & FuseMoE & $0.845 \pm 0.003$ & $0.738 \pm 0.006$ & $0.765 \pm 0.004$ & $0.954 \pm 0.003$ & $0.860 \pm 0.002$ & $0.702 \pm 0.021$ & $0.842 \pm 0.075$ & $0.734 \pm 0.010$ & $0.538 \pm 0.037$ \\
 &  & FlexMoE & $0.846 \pm 0.002$ & $0.739 \pm 0.007$ & $0.770 \pm 0.001$ & $0.959 \pm 0.000$ & $0.866 \pm 0.002$ & $0.737 \pm 0.010$ & $0.662 \pm 0.071$ & $0.756 \pm 0.001$ & $0.474 \pm 0.048$ \\
 &  & FLAME & $0.857 \pm 0.001$ & $0.723 \pm 0.016$ & $0.774 \pm 0.000$ & $0.958 \pm 0.002$ & $0.863 \pm 0.001$ & $0.736 \pm 0.026$ & $0.839 \pm 0.013$ & $0.757 \pm 0.000$ & $0.601 \pm 0.022$ \\
\cline{2-12}
 & Multitask (per dataset) & FLAME & $0.855 \pm 0.011$ & $0.749 \pm 0.016$ & $0.770 \pm 0.004$ & $0.958 \pm 0.000$ & $0.865 \pm 0.002$ & $0.734 \pm 0.037$ & $0.772 \pm 0.027$ & $0.758 \pm 0.000$ & -- \\
\cline{2-12}
 & Multitask (all datasets) & FLAME & $0.852 \pm 0.005$ & $0.743 \pm 0.011$ & $0.770 \pm 0.006$ & $0.959 \pm 0.000$ & $0.867 \pm 0.000$ & $0.666 \pm 0.083$ & $0.516 \pm 0.153$ & $0.729 \pm 0.022$ & $0.516 \pm 0.049$ \\
\hline
\end{tabular}}
\label{tab:full_single_task_appendix}
\end{table}

\subsubsection{Pairwise Task Interactions}\label{app:pairwise}

Fig.~\ref{fig:confusion} reports the pairwise multitask-gain matrix that grounds the main-text RQ2 discussion. Each row corresponds to a focal task; each column gives a co-trained partner; cell $(t, t')$ is the AUROC (left) or AUPRC (right) of task $t$ when FLAME is jointly trained on the pair $\{t, t'\}$. Two structural patterns are visible. \emph{Within-dataset stability}: the EMBED block (\textsc{BIRADS}, \textsc{RISK}, \textsc{DENSITY}) is uniformly bright, with \textsc{DENSITY} above $0.92$ AUROC for every column it appears in, indicating that mammography representations transfer readily within the dataset and resist interference from unrelated co-trained tasks. \emph{Cross-dataset asymmetry}: the \textsc{IHM} row shifts by $0.02$--$0.03$ depending on the partner, with the largest gain coming from co-training with \textsc{LOS}, while ADNI's \textsc{DIAG} (whose modality pool is disjoint from every other task) regresses modestly under most pairings. Tasks pair well precisely when their per-modality routers (Sec.~\ref{app:routing-multi}) overlap; the heatmap therefore serves as a data-driven recommender for which task pairs should share a model.

\begin{figure}[h]
    \centering
    \begin{subfigure}[b]{0.49\linewidth}
        \centering
        \includegraphics[width=\linewidth]{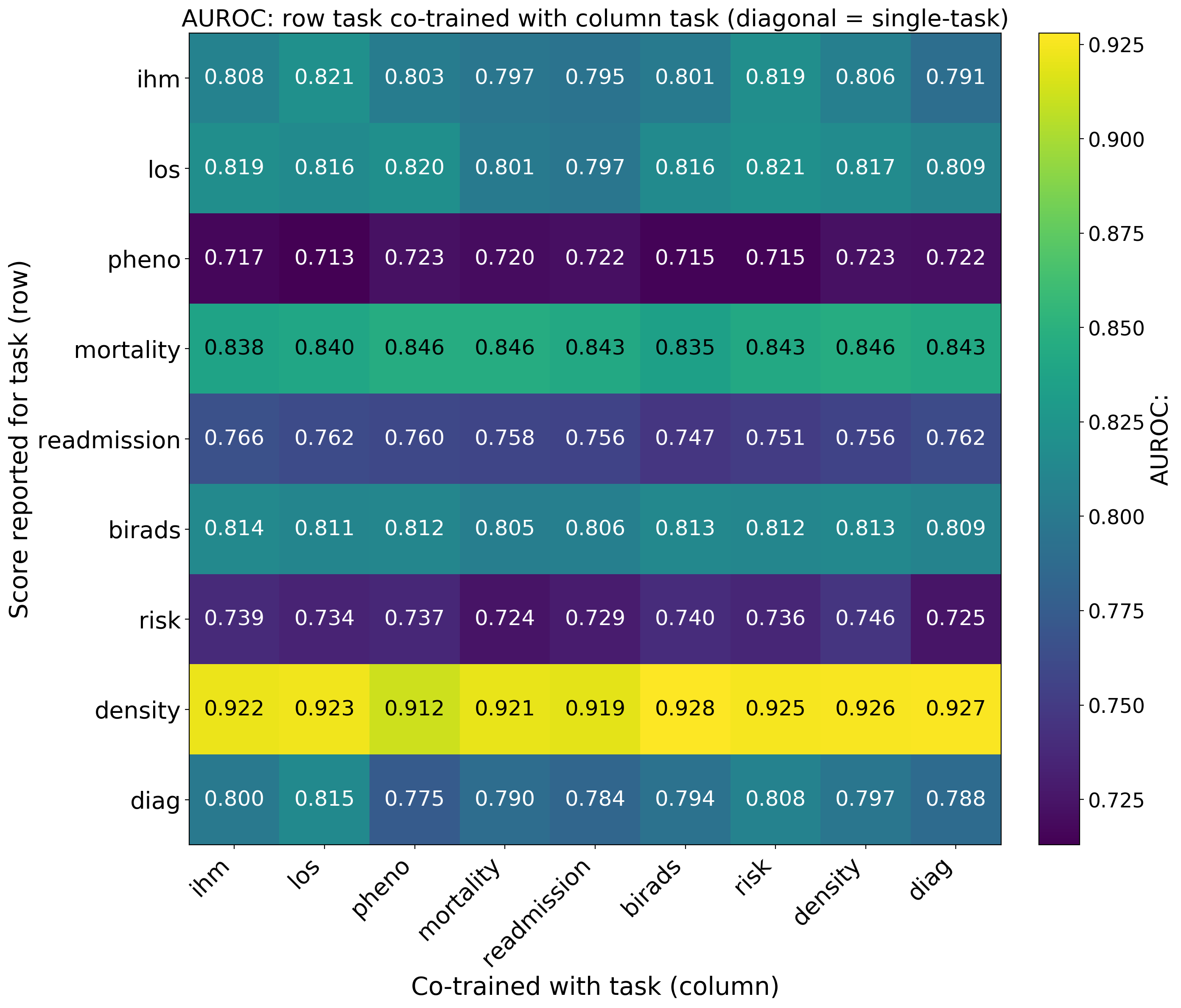}
        \caption{AUROC}
        \label{fig:confusion1}
    \end{subfigure}
    \hfill
    \begin{subfigure}[b]{0.49\linewidth}
        \centering
        \includegraphics[width=\linewidth]{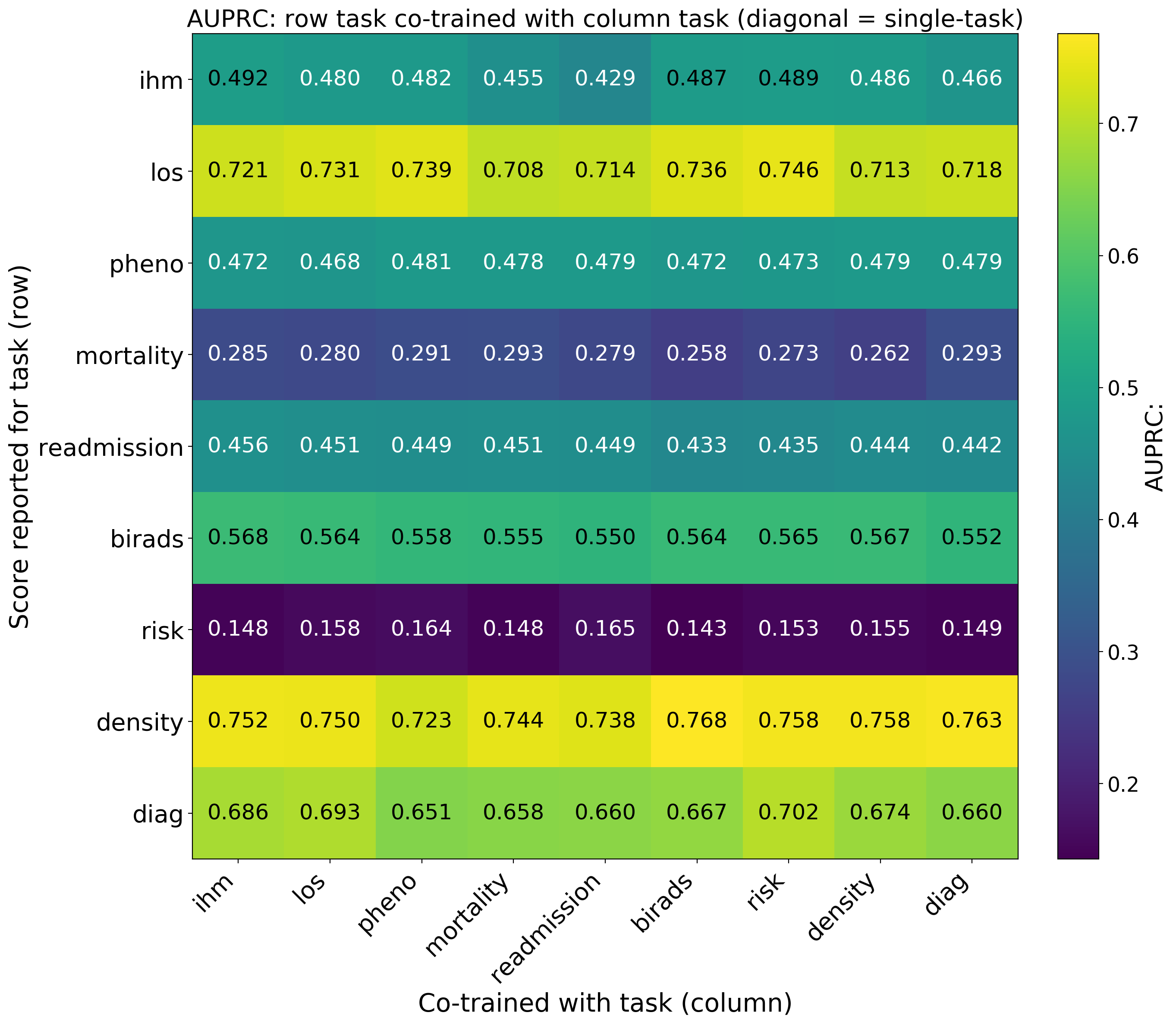}
        \caption{AUPRC}
        \label{fig:confusion2}
    \end{subfigure}
    \caption{Pairwise task performance. Left: AUROC, Right: AUPRC. Each row represents the focal task's performance when trained jointly with the partner task on the column.}
    \label{fig:confusion}
\end{figure}

\subsubsection{Number-of-Experts Ablation}\label{app:expert_ablation}

We sweep the shared expert pool size $N \in \{2, 4, 8, 16, 32\}$ with FLAME jointly trained on all nine tasks at seed~$42$. Fig.~\ref{fig:expert_ablation_auprc} shows the AUROC and AUPRC trajectories; Tabs.~\ref{tab:expert_ablation_auroc}, \ref{tab:expert_ablation_auprc} give the numbers per task. Performance is largely insensitive to $N$ for $N \ge 4$: per-task AUROC moves by less than $0.02$ across the entire sweep on every task except \textsc{48-IHM} (where $N{=}16$ regresses but $N{=}32$ recovers). The sweet spot is $N \in \{4, 8\}$, balancing capacity against per-modality router collapse risk; all main-text experiments use $N{=}5$. The flatness of the curve indicates that FLAME's gains do not rest on a particular pool size: the per-modality dispatch already concentrates each modality on a small expert subset, so adding more experts saturates rather than shifts the routing.

\begin{figure}[h]
\centering
\includegraphics[width=\linewidth]{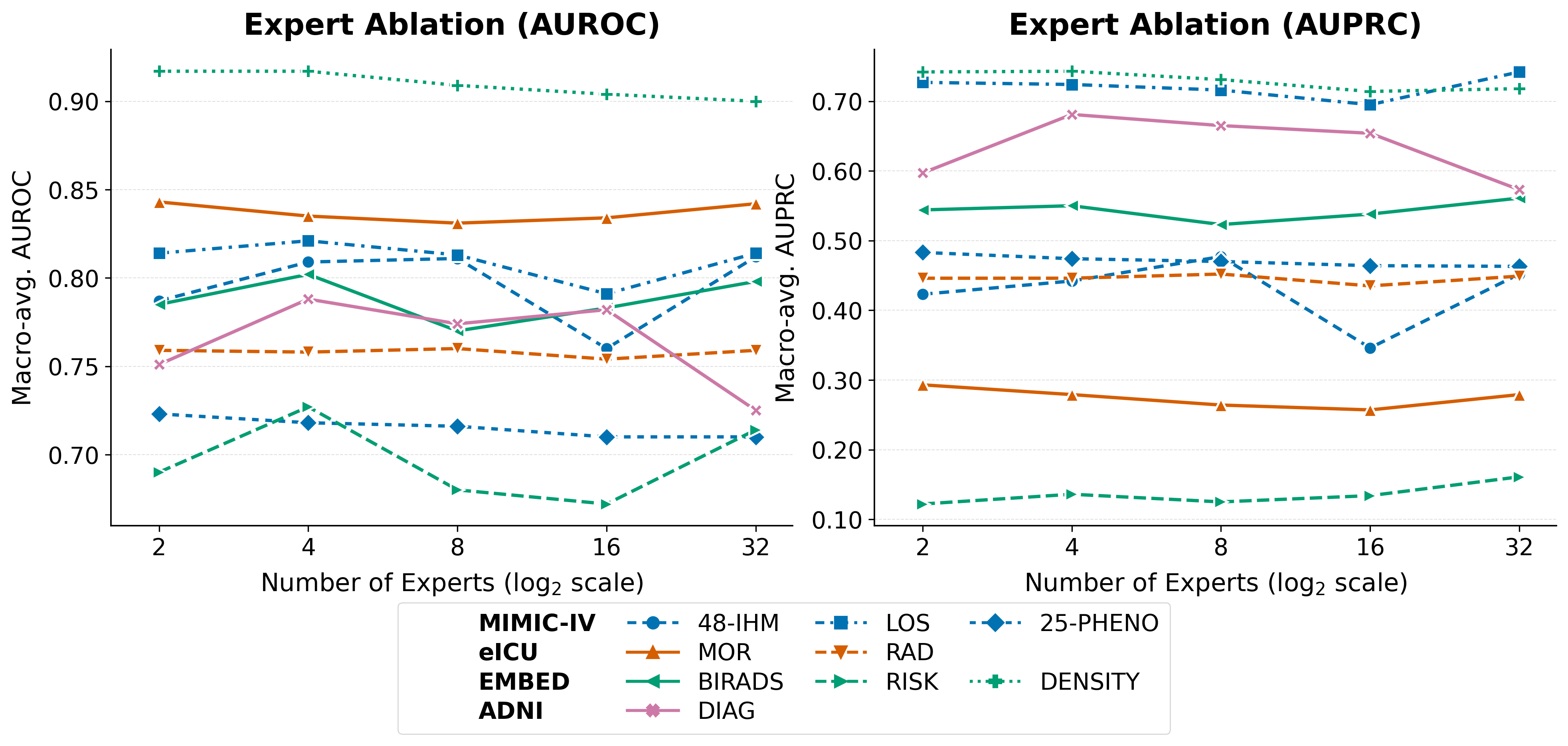}
\caption{Ablation over number of experts; Left: AUROC, Right: AUPRC. FLAME jointly trained on all 9 tasks. Multi-class/multi-label tasks use macro-averaged AUROC and AUPRC.}
\label{fig:expert_ablation_auprc}
\end{figure}

\begin{table}[h]
\centering
\caption{Ablation over number of experts (AUROC; seed=42; FLAME jointly trained on all 9 tasks). Multi-class/multi-label tasks use macro-averaged AUROC.}
\resizebox{\textwidth}{!}{%
\begin{tabular}{l|ccc|cc|ccc|c}
\hline
\multirow{2}{*}{\# Experts} & \multicolumn{3}{c|}{MIMIC IV} & \multicolumn{2}{c|}{eICU} & \multicolumn{3}{c|}{EMBED} & \multicolumn{1}{c}{ADNI} \\
\cline{2-4}\cline{5-6}\cline{7-9}\cline{10-10}
 & 48-IHM & LOS & 25-PHENO & MOR & RAD & BIRADS & RISK & DENSITY & DIAG \\
\hline
2 & 0.787 & 0.814 & 0.723 & 0.843 & 0.759 & 0.785 & 0.690 & 0.917 & 0.751 \\
4 & 0.809 & 0.821 & 0.718 & 0.835 & 0.758 & 0.802 & 0.727 & 0.917 & 0.788 \\
8 & 0.811 & 0.813 & 0.716 & 0.831 & 0.760 & 0.770 & 0.680 & 0.909 & 0.774 \\
16 & 0.760 & 0.791 & 0.710 & 0.834 & 0.754 & 0.783 & 0.672 & 0.904 & 0.782 \\
32 & 0.812 & 0.814 & 0.710 & 0.842 & 0.759 & 0.798 & 0.714 & 0.900 & 0.725 \\
\hline
\end{tabular}}
\label{tab:expert_ablation_auroc}
\end{table}

\begin{table}[h]
\centering
\caption{Ablation over number of experts (AUPRC; seed=42; FLAME jointly trained on all 9 tasks). Multi-class/multi-label tasks use macro-averaged AUPRC.}
\resizebox{\textwidth}{!}{%
\begin{tabular}{l|ccc|cc|ccc|c}
\hline
\multirow{2}{*}{\# Experts} & \multicolumn{3}{c|}{MIMIC IV} & \multicolumn{2}{c|}{eICU} & \multicolumn{3}{c|}{EMBED} & \multicolumn{1}{c}{ADNI} \\
\cline{2-4}\cline{5-6}\cline{7-9}\cline{10-10}
 & 48-IHM & LOS & 25-PHENO & MOR & RAD & BIRADS & RISK & DENSITY & DIAG \\
\hline
2 & 0.423 & 0.727 & 0.483 & 0.293 & 0.446 & 0.544 & 0.122 & 0.742 & 0.597 \\
4 & 0.442 & 0.724 & 0.474 & 0.279 & 0.446 & 0.550 & 0.136 & 0.743 & 0.681 \\
8 & 0.477 & 0.716 & 0.470 & 0.264 & 0.452 & 0.523 & 0.125 & 0.731 & 0.665 \\
16 & 0.346 & 0.695 & 0.464 & 0.257 & 0.435 & 0.538 & 0.134 & 0.714 & 0.654 \\
32 & 0.452 & 0.742 & 0.463 & 0.279 & 0.449 & 0.561 & 0.161 & 0.718 & 0.573 \\
\hline
\end{tabular}}
\label{tab:expert_ablation_auprc}
\end{table}

\subsubsection{Gating-Function Ablation}\label{app:gating_ablation}

Tab.~\ref{tab:gating_ablation} compares Laplace~\cite{han2024fusemoe}, Gaussian, and Softmax~\cite{shazeer2017outrageously} gating under three regimes (single-task, per-dataset multitask, all-nine-task multitask) at $N{=}8$ and seed~$42$. The three gates are within $0.01$ AUROC and $0.02$ AUPRC of one another on every task and every regime, and no gate wins uniformly. The methodological implication is sharp: FLAME's parameter-efficiency and continual-retention gains do not come from the gating non-linearity but from the structural choices around it---per-modality routing, sample-level dispatch via TAP, and rank-reserved spectral compression. Replacing the gate is essentially a no-op; replacing any of the structural pieces is not.

\begin{table}[h]
\centering
\caption{Gating-style ablation (FLAME, 8 experts, seed=42). Rows are tasks grouped by training setup: Single Task trains IHM alone; Per-dataset (MIMIC IV) trains IHM, LOS, PHENO jointly; All Datasets trains all 9 tasks jointly. Multi-class/multi-label tasks (PHENO, BIRADS, DENSITY, DIAG) use macro-averaged AUROC/AUPRC.}
\resizebox{\textwidth}{!}{%
\begin{tabular}{l|l|ccc|ccc}
\hline
\multirow{2}{*}{Setup} & \multirow{2}{*}{Task} & \multicolumn{3}{c|}{AUROC} & \multicolumn{3}{c}{AUPRC} \\
\cline{3-5}\cline{6-8}
 &  & Laplace & Gaussian & Softmax & Laplace & Gaussian & Softmax \\
\hline
Single Task & 48-IHM & 0.794 & 0.797 & 0.802 & 0.442 & 0.461 & 0.433 \\
\hline
\multirow{3}{*}{Per-dataset (MIMIC IV)} & 48-IHM & 0.805 & 0.804 & 0.810 & 0.419 & 0.426 & 0.443 \\
 & LOS & 0.819 & 0.806 & 0.809 & 0.719 & 0.709 & 0.714 \\
 & 25-PHENO & 0.710 & 0.711 & 0.714 & 0.467 & 0.468 & 0.472 \\
\hline
\multirow{9}{*}{All Datasets} & 48-IHM & 0.811 & 0.816 & 0.817 & 0.477 & 0.487 & 0.487 \\
 & LOS & 0.813 & 0.817 & 0.815 & 0.716 & 0.721 & 0.705 \\
 & 25-PHENO & 0.716 & 0.717 & 0.719 & 0.470 & 0.472 & 0.475 \\
 & RAD & 0.760 & 0.759 & 0.758 & 0.452 & 0.453 & 0.455 \\
 & MOR & 0.831 & 0.841 & 0.843 & 0.264 & 0.274 & 0.271 \\
 & BIRADS & 0.770 & 0.779 & 0.792 & 0.523 & 0.519 & 0.551 \\
 & RISK & 0.680 & 0.732 & 0.729 & 0.125 & 0.126 & 0.143 \\
 & DENSITY & 0.909 & 0.914 & 0.900 & 0.731 & 0.739 & 0.711 \\
 & DIAG & 0.774 & 0.803 & 0.755 & 0.665 & 0.691 & 0.603 \\
\hline
\end{tabular}}
\label{tab:gating_ablation}
\end{table}

\subsection{Detailed Continual Learning Results}\label{app:cl_results}

\subsubsection{Walkthrough Against Standard CL Baselines}\label{app:cl_walkthrough}

This subsubsection unpacks the continual-learning headline of Sec.~\ref{sec:results} into per-method and per-stage detail.

\textbf{Setup.} We compare \flamecl{} against simple fine-tuning (\simpleft{}), Elastic Weight Consolidation (EWC), and Low-Rank Adaptation (\lora{}) on four task sequences that span single-task and multitask stages over MIMIC-IV, eICU, and EMBED (Setups~1--4 in Fig.~\ref{fig:cl-grid}). \simpleft{} updates the full encoder in place at every stage; EWC adds a Fisher-weighted regularizer; \lora{} freezes the base and stacks a fresh low-rank adapter per stage; \flamecl{} trains a fresh additive component per stage and freezes a rank-$r$ truncation of it on completion (Sec.~\ref{subsec:cl}).

\textbf{Retention.} On AUROC averaged over the final-stage evaluation of every introduced task, \flamecl{} is within $0.01$ of \lora{} on every sequence and beats \simpleft{} and EWC by $0.02$--$0.05$ on the earliest tasks of every sequence (the tasks most exposed to forgetting). \lora{} is competitive on retention only because its frozen base shields prior-task weights, a property \flamecl{} recovers structurally through frozen prior components and cursor-based inference at a fraction of the parameter cost.

\textbf{Per-stage parameter footprint.} The middle and right rows of Fig.~\ref{fig:cl-grid} report the method-attributable encoder and MoE parameter count at the end of every stage. \simpleft{} and EWC sit flat at $109$M encoder parameters because they update the full encoder in place; \lora{} grows from $109.54$M at stage~1 to $125.29$M by stage~3 of Setup~4, since each new stage stacks a fresh adapter on top of the frozen base; \flamecl{} grows from $7.99$M to $23.96$M, a five- to fifteen-fold reduction depending on the stage. The MoE budget shows the same shape ($0.17$M to $0.50$M for \flamecl{} versus $0.41$M to $0.74$M for \lora{}).

\textbf{Spectral hypothesis link.} Proposition~\ref{prop:funcrank} predicts this efficiency: a trained expert's functional energy concentrates in a low-rank subspace because the routed-input distribution is itself low-rank. Sec.~\ref{app:spectra} confirms it: across every task and every joint configuration the data-aware spectrum saturates well before rank~$40$ of an ambient~$128$ while the weight-only spectrum stays close to full rank. The compression in Eq.~\eqref{eq:compress} therefore removes only directions that contribute negligibly to the expert's output on the task it was trained for.

\subsubsection{Comparison Against Lifelong-PT}\label{app:cl_vs_lifelong}

Fig.~\ref{fig:cl-vs-lifelong} compares \flamecl{} against Lifelong-PT~\cite{chen2023lifelong}, an MoE-based CL method that grows the expert pool with each new task. We use the same four sequences and the same axes as Fig.~\ref{fig:cl-grid} of the main text. \flamecl{} (red) matches Lifelong-PT (blue) on AUROC at every stage of every sequence, while keeping the encoder budget below the static base in Setups~1 and~4 and the MoE budget within $0.4$M of the smallest method in every sequence. Lifelong-PT pays a constant per-task expert-addition cost that does not amortize across tasks; \flamecl{}'s rank-reserved stack reuses idle capacity inside a fixed-size pool and amortizes the budget over the entire stream. This experiment isolates the comparison to the ``expand the pool'' approach to MoE continual learning specifically, and shows that \flamecl{} delivers the same retention without paying its parameter price.

\begin{figure}[h]
\centering
\includegraphics[width=\linewidth]{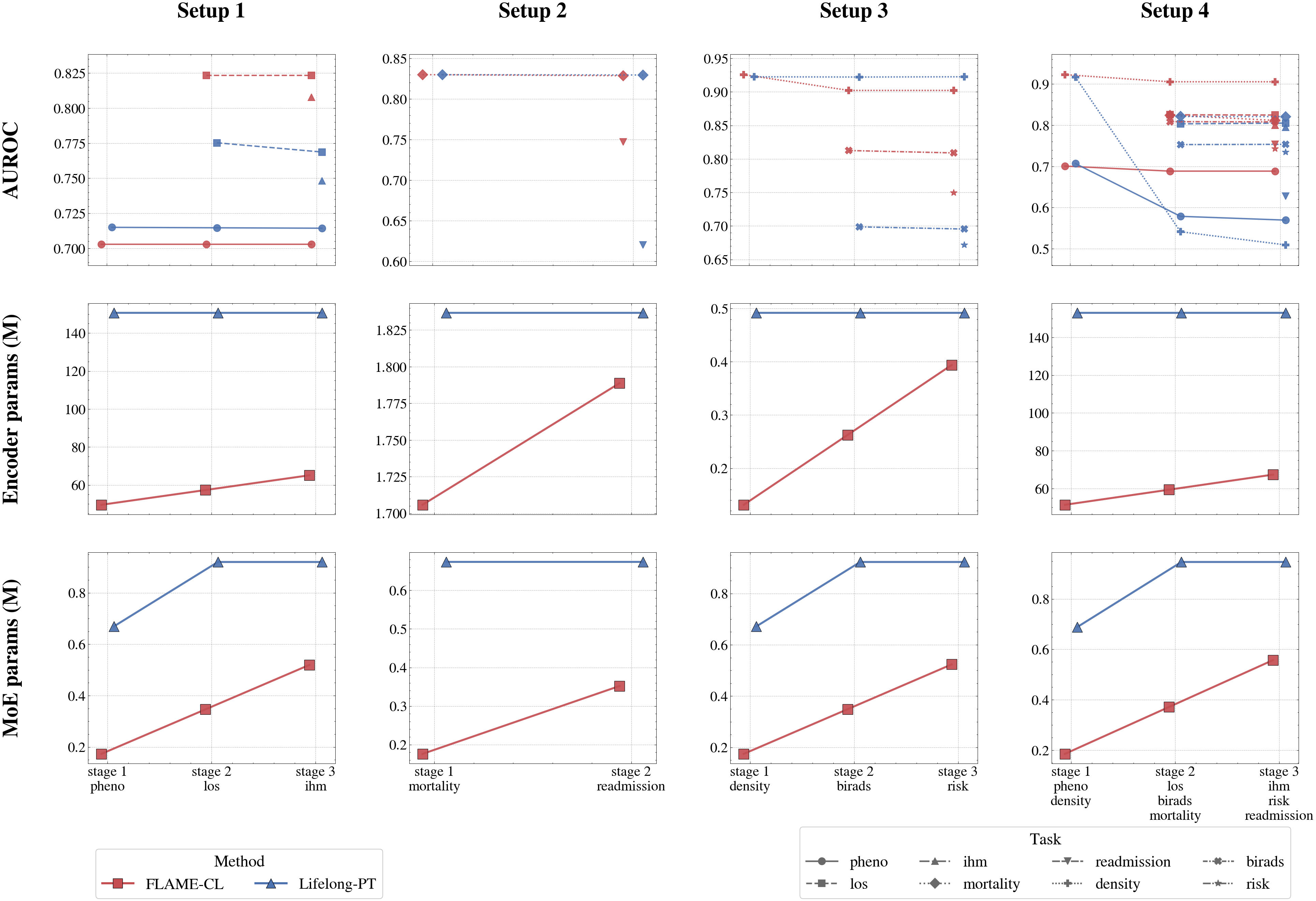}
\caption{Comparison against Lifelong-PT~\cite{chen2023lifelong}, an MoE-based continual-learning method that grows the expert pool with each new task. Per-stage trajectories for the four sequences (columns: Setup~1 to Setup~4) across three metrics (rows: AUROC, encoder parameter count in millions, MoE parameter count in millions). x-axis: training-stage checkpoint; within each stage, AUROC panels overlay one line per task seen so far. FLAME-CL (red) matches Lifelong-PT (blue) on AUROC at every stage of every sequence while keeping the encoder budget below the static base in Setups~1 and~4 and the MoE budget within $0.4$M of the smallest method in every sequence; Lifelong-PT pays a constant per-task expert-addition cost that does not amortize. Same axes, sequences, and grouping as Fig.~\ref{fig:cl-grid}; replacing the four CL baselines (Simple FT, EWC, LoRA, FLAME-CL) with the two-method comparison FLAME-CL versus Lifelong-PT.}
\label{fig:cl-vs-lifelong}
\end{figure}

\subsubsection{Comparison Against LoRA}\label{app:cl_vs_lora}

Fig.~\ref{fig:cl-vs-lora} isolates the comparison from Sec.~\ref{app:cl_walkthrough} to \flamecl{} versus \lora{} alone, using the same four sequences and the same three metric rows (AUROC, encoder params, MoE params) as Fig.~\ref{fig:cl-grid}. Among the standard CL baselines, \lora{} is the only one that retains earlier-stage AUROC by freezing the base, so the relevant question is whether \flamecl{} can match \lora{}'s retention at a smaller per-stage parameter cost.

The panels answer this. \flamecl{}'s AUROC trajectories track \lora{}'s within $0.01$ at every stage of every sequence, while the encoder-parameter and MoE-parameter rows are uniformly smaller per stage. Encoder parameters grow roughly linearly with stage in both methods, but \flamecl{}'s slope is below \lora{}'s, with the gap widening as the task stream lengthens (Setup~4); the MoE-parameter row has the same shape, with both methods one to two orders of magnitude below the encoder count. The retention parity follows from the cursor-based inference of Sec.~\ref{subsec:cl}: \flamecl{}'s frozen prior components are inaccessible to the optimizer at later stages, mirroring the protection \lora{} buys by freezing its base, but with each new component a rank-$r_t$ slice rather than a full-shape low-rank adapter.

\begin{figure}[h]
\centering
\includegraphics[width=\linewidth]{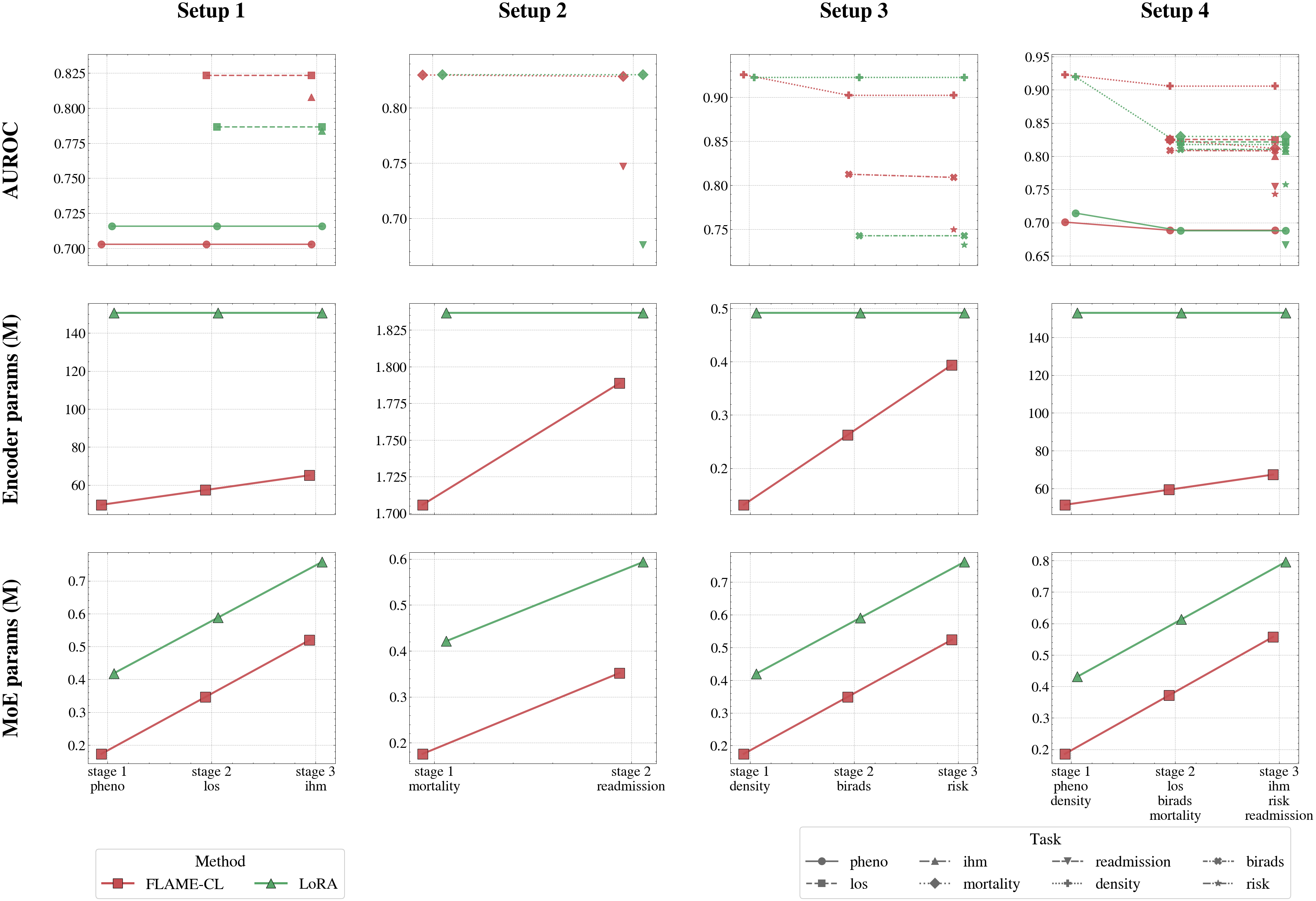}
\caption{Comparison against \lora{}~\cite{hu2022lora}, the strongest retention baseline among the standard CL methods. Per-stage trajectories for the four sequences (columns: Setup~1 to Setup~4) across three metrics (rows: AUROC, encoder parameter count in millions, MoE parameter count in millions). x-axis: training-stage checkpoint; within each stage, AUROC panels overlay one line per task seen so far. \flamecl{} (red) matches \lora{} (green) on AUROC at every stage of every sequence while using fewer encoder and MoE parameters at every stage. Same axes, sequences, and grouping as Fig.~\ref{fig:cl-grid}; replacing the four CL baselines (Simple FT, EWC, LoRA, FLAME-CL) with the two-method comparison FLAME-CL versus LoRA.}
\label{fig:cl-vs-lora}
\end{figure}

\subsubsection{Per-Stage AUPRC Trajectories}\label{app:cl-auprc}

For completeness we report per-stage AUPRC trajectories complementing the AUROC panels of Fig.~\ref{fig:cl-grid}; axes, sequences (S1--S4), method ordering, and grouping are identical (Fig.~\ref{fig:cl-auprc}). The trend mirrors the AUROC story: \flamecl{} stays close to or above the best baseline at every stage of every sequence, with the largest margins on the earliest tasks of each sequence (the tasks most exposed to forgetting). AUPRC is the more informative metric for the imbalanced binary tasks (\textsc{48-IHM}, \textsc{RISK}, \textsc{MOR}), and the absence of degradation there indicates that the structural no-forgetting guarantee from cursor-based inference protects positive recall, not only the operating threshold.

\begin{figure}[h]
\centering
\begin{subfigure}[t]{0.48\textwidth}\includegraphics[width=\linewidth,height=1.7in]{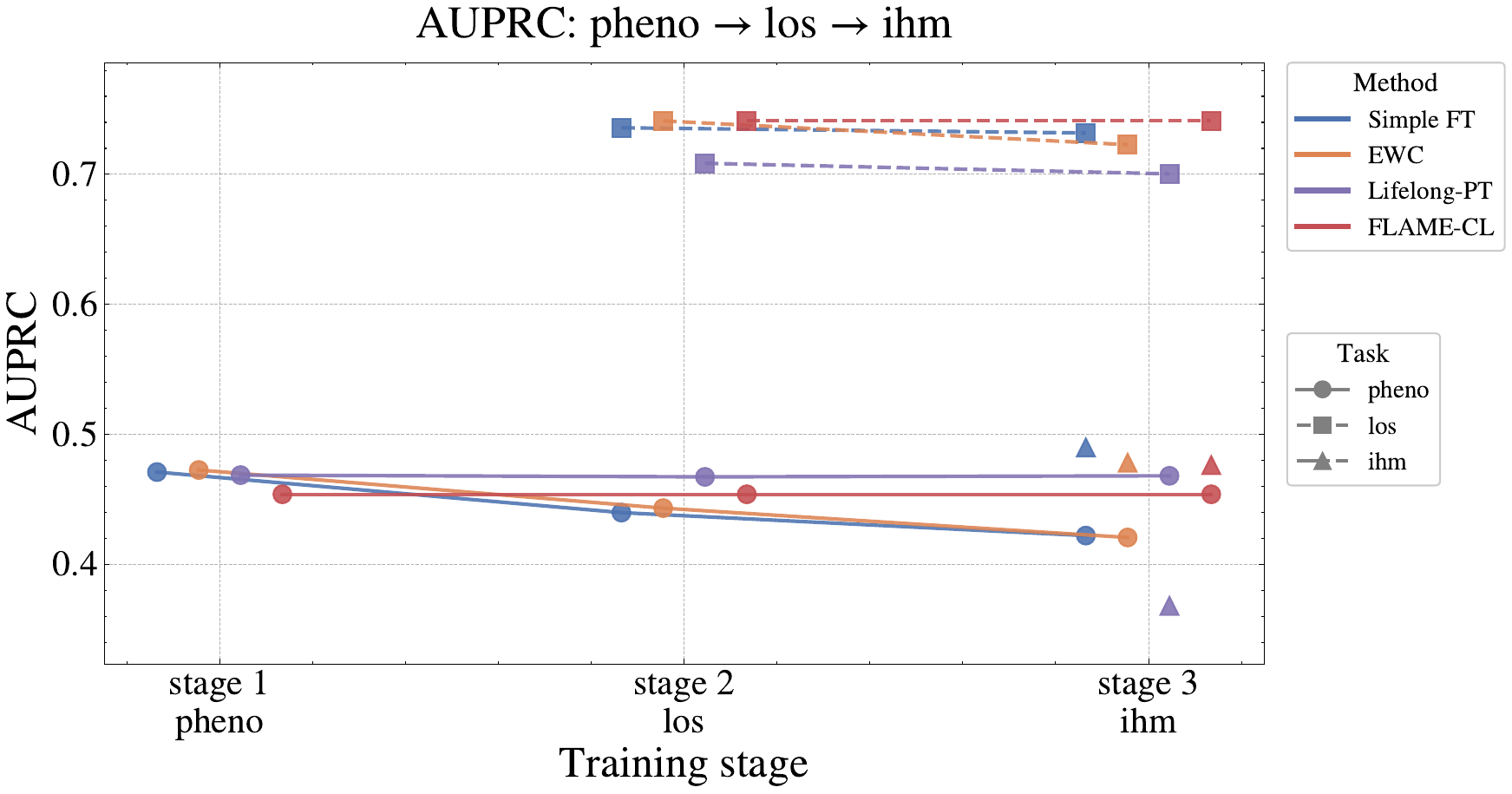}\caption{S1.}\label{fig:cl-auprc-s1}\end{subfigure}\hfill
\begin{subfigure}[t]{0.48\textwidth}\includegraphics[width=\linewidth,height=1.7in]{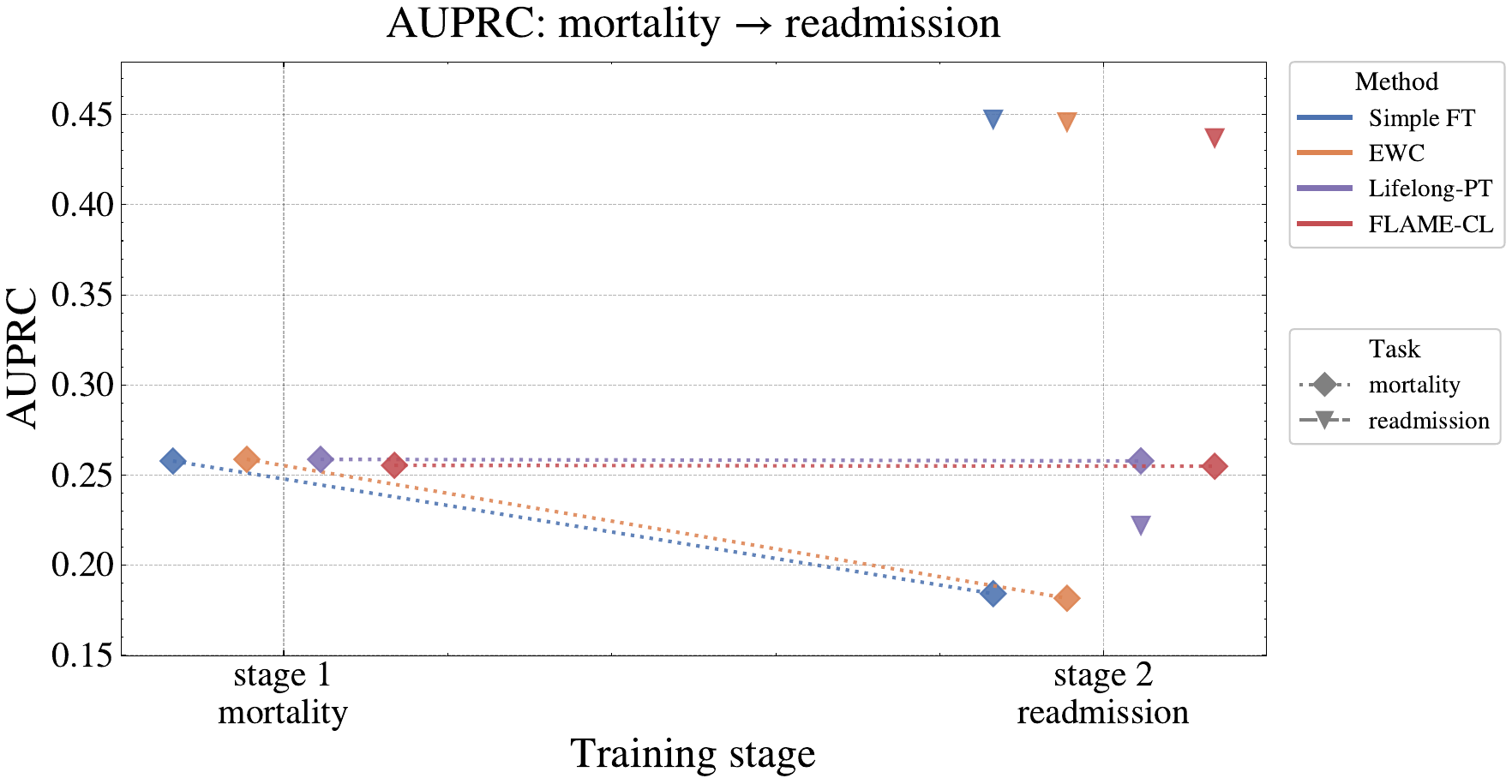}\caption{S2.}\label{fig:cl-auprc-s2}\end{subfigure}\\[4pt]
\begin{subfigure}[t]{0.48\textwidth}\includegraphics[width=\linewidth,height=1.7in]{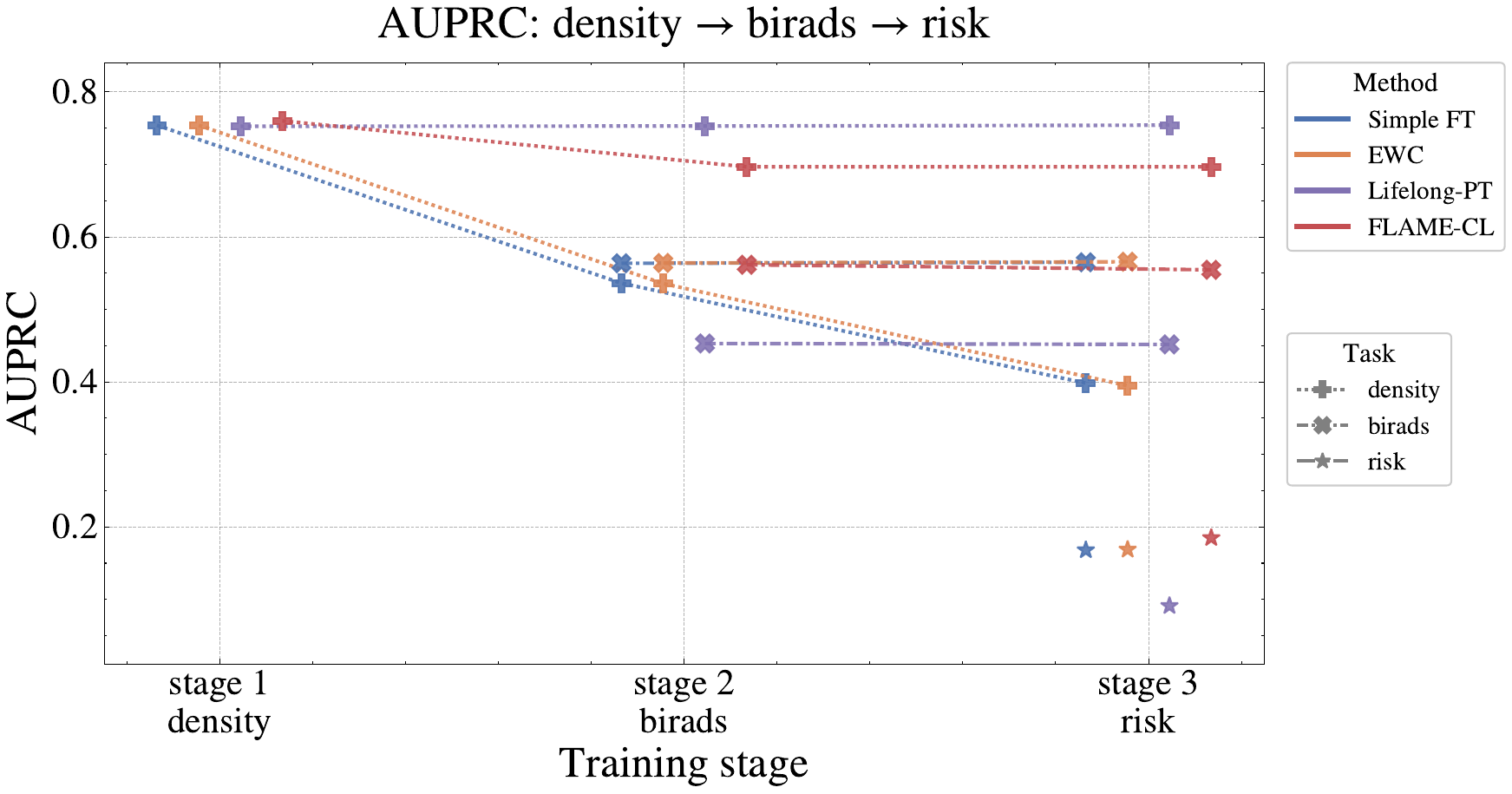}\caption{S3.}\label{fig:cl-auprc-s3}\end{subfigure}\hfill
\begin{subfigure}[t]{0.48\textwidth}\includegraphics[width=\linewidth,height=1.7in]{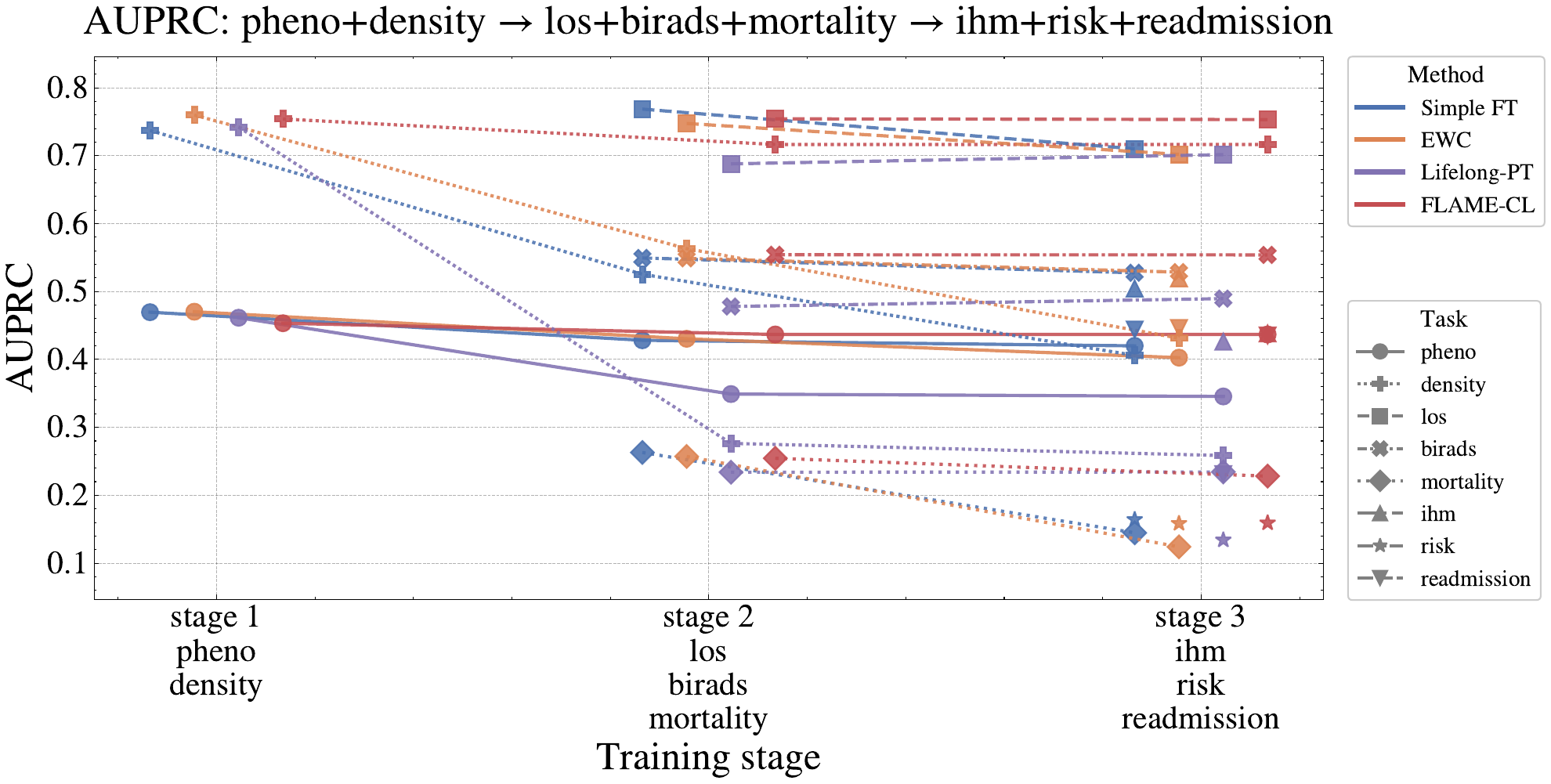}\caption{S4.}\label{fig:cl-auprc-s4}\end{subfigure}
\caption{Per-stage \textbf{AUPRC} for the four continual-learning sequences (S1--S4); axes and grouping as in Fig.~\ref{fig:cl-grid}.}
\label{fig:cl-auprc}
\end{figure}

\subsection{Spectral Validation of Proposition~\ref{prop:funcrank}}\label{app:spectra}

This section reports the empirical spectra that motivate the spectral-compression step of FLAME (Eq.~\eqref{eq:compress}) and validate Proposition~\ref{prop:funcrank} across every task and joint training configuration considered in our experiments. For each configuration we plot, in the first cross-modal MoE block, the cumulative top-$K$ energy of all 10 expert sublayers (5 experts $\times$ \{\texttt{fc1}, \texttt{fc2}\}) under three lenses:
\begin{itemize}\itemsep0pt
\item \textbf{Input spectrum} ($C_i$ eigenvalues): the eigenvalue spectrum of the routed-input covariance $C_i = \mathbb{E}_{\bm{z}}[\bm{z}\bm{z}^\top]$, estimated from the test-set tokens that each expert actually receives.
\item \textbf{Weight-only} (Frobenius): the singular-value energy $\{\sigma_{i,k}^2\}$ of the expert's weight matrix $W_i$, ignoring the input distribution.
\item \textbf{Data-aware} (test activations): the per-rank functional energy $\mathcal{E}_{i,k} = \sigma_{i,k}^2 \bm{v}_{i,k}^\top C_i \bm{v}_{i,k}$ from Eq.~\eqref{eq:funcrank-decomp}, which couples weight and input geometry.
\end{itemize}
The dashed red lines mark the $90\%$ and $99\%$ cumulative-energy thresholds. Across every benchmark, the input spectrum and the data-aware curve saturate by rank ${\approx}20$--$40$ while the weight-only spectrum remains near full rank, exposing the idle capacity that FLAME's continual-learning step reallocates.

\subsubsection{Single-Task Pretraining Spectra}\label{app:spectra-single}

Figs.~\ref{fig:spectra-mimic}--\ref{fig:spectra-adni} show the three-lens spectra for each of the nine tasks individually trained. The pattern is consistent across modality types: time-series ICU tasks (Figs.~\ref{fig:spectra-mimic}, \ref{fig:spectra-eicu}), high-resolution mammography views (Fig.~\ref{fig:spectra-embed}), and multi-modal Alzheimer biomarkers (Fig.~\ref{fig:spectra-adni}) all show data-aware energy curves that saturate well before the ambient rank, despite the wide variation in input statistics. This consistency is what makes per-task rank reservation viable: the spectral-rank budget does not need to be re-tuned per modality, and the truncation step in Eq.~\eqref{eq:compress} discards directions with negligible functional energy regardless of which clinical task contributes them.

\begin{figure}[h]
\centering
\begin{subfigure}{\textwidth}\includegraphics[width=\linewidth]{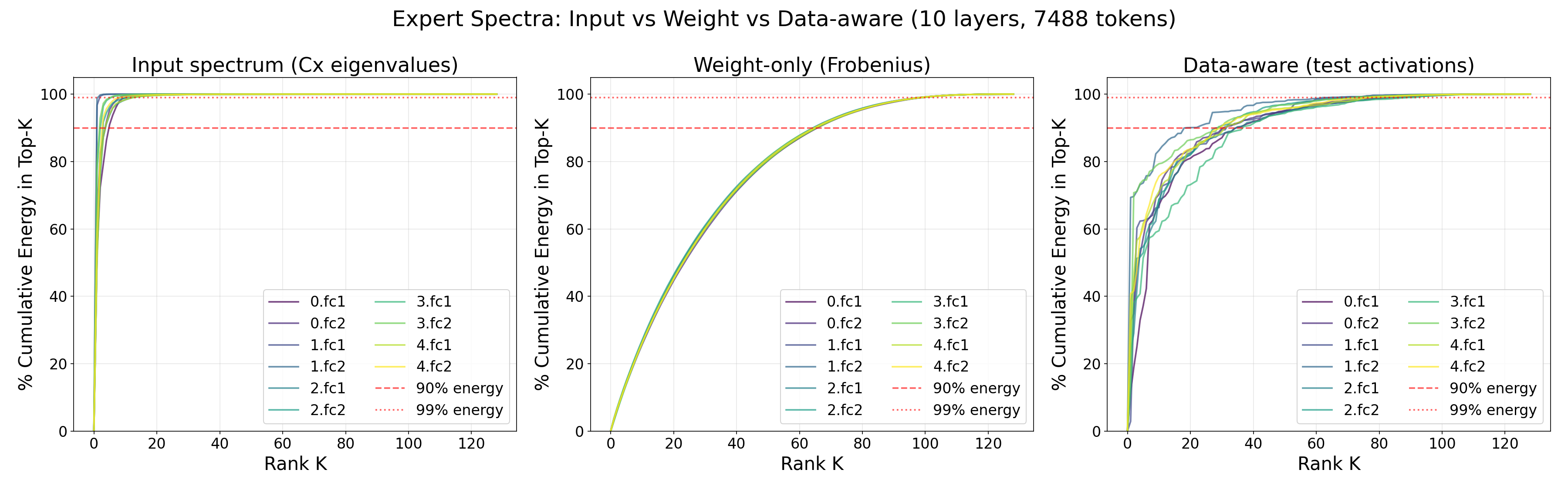}
\caption{48-IHM (in-hospital mortality), MIMIC-IV.}\label{fig:spec-ihm}\end{subfigure}\\[2pt]
\begin{subfigure}{\textwidth}\includegraphics[width=\linewidth]{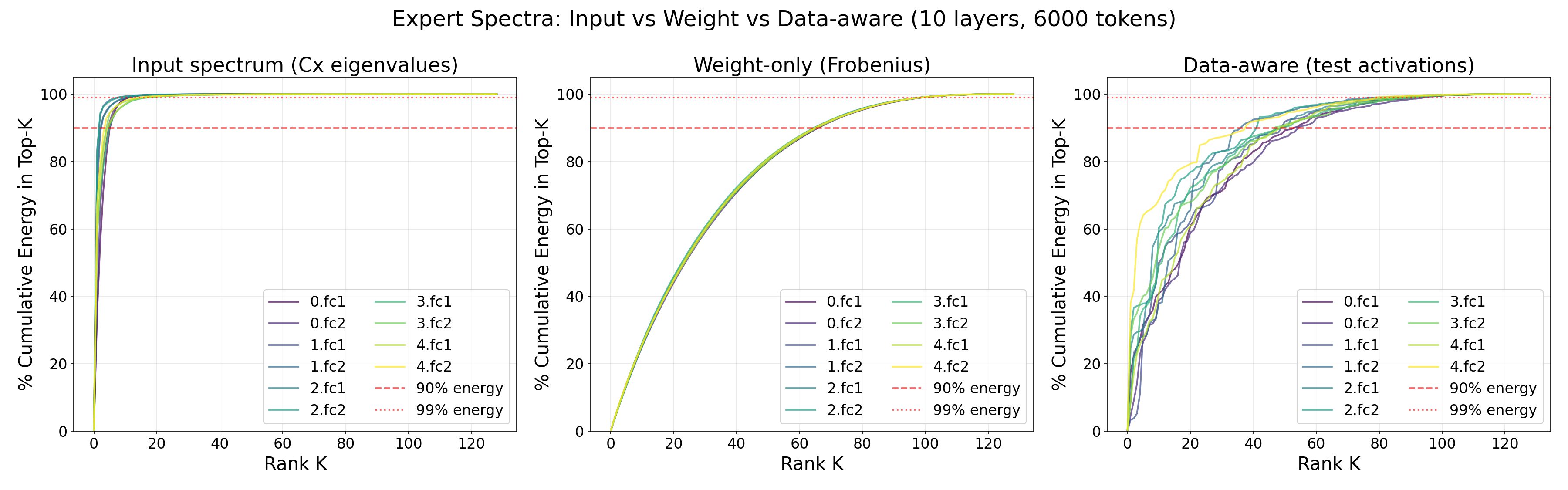}
\caption{LOS (length-of-stay), MIMIC-IV.}\label{fig:spec-los}\end{subfigure}\\[2pt]
\begin{subfigure}{\textwidth}\includegraphics[width=\linewidth]{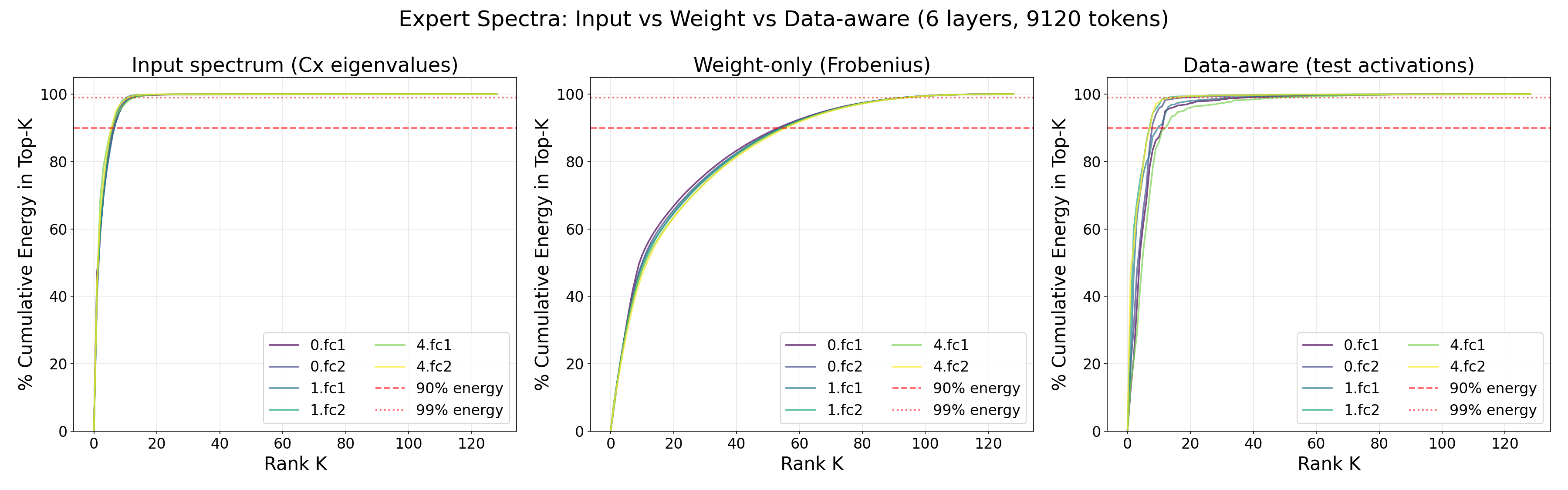}
\caption{25-PHENO (phenotype classification), MIMIC-IV.}\label{fig:spec-pheno}\end{subfigure}
\caption{Expert input spectra after single-task FLAME pretraining on the three MIMIC-IV tasks (modalities: time series, clinical text, chest X-ray).}
\label{fig:spectra-mimic}
\end{figure}

\begin{figure}[h]
\centering
\begin{subfigure}{\textwidth}\includegraphics[width=\linewidth]{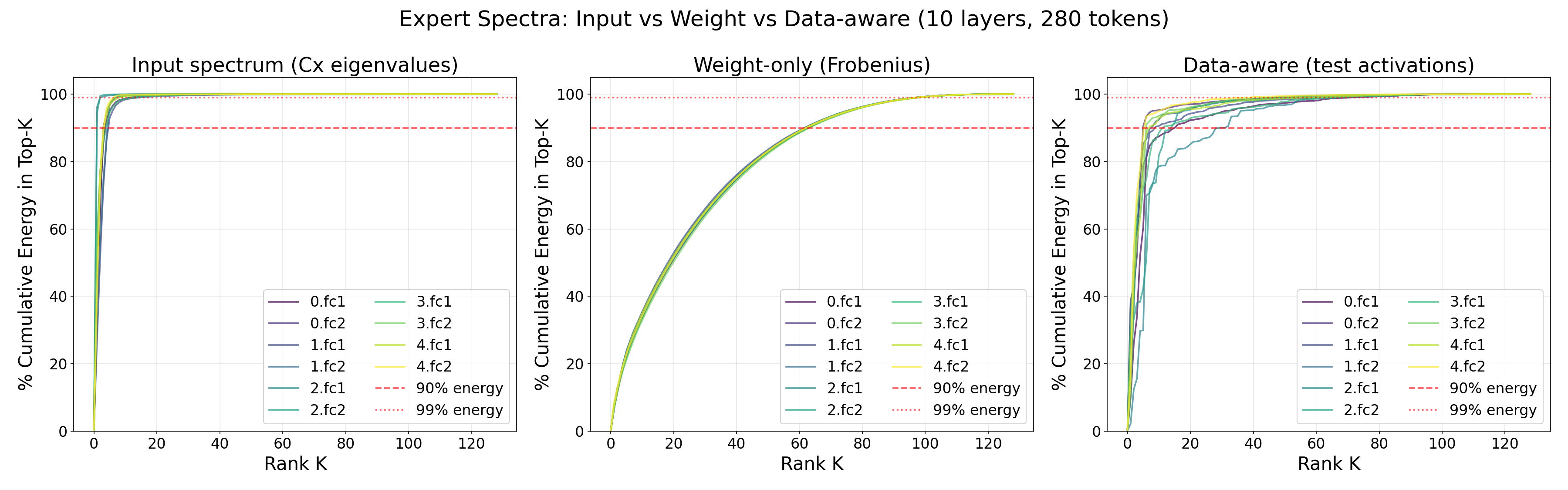}
\caption{MOR (mortality), eICU.}\label{fig:spec-mor}\end{subfigure}\\[2pt]
\begin{subfigure}{\textwidth}\includegraphics[width=\linewidth]{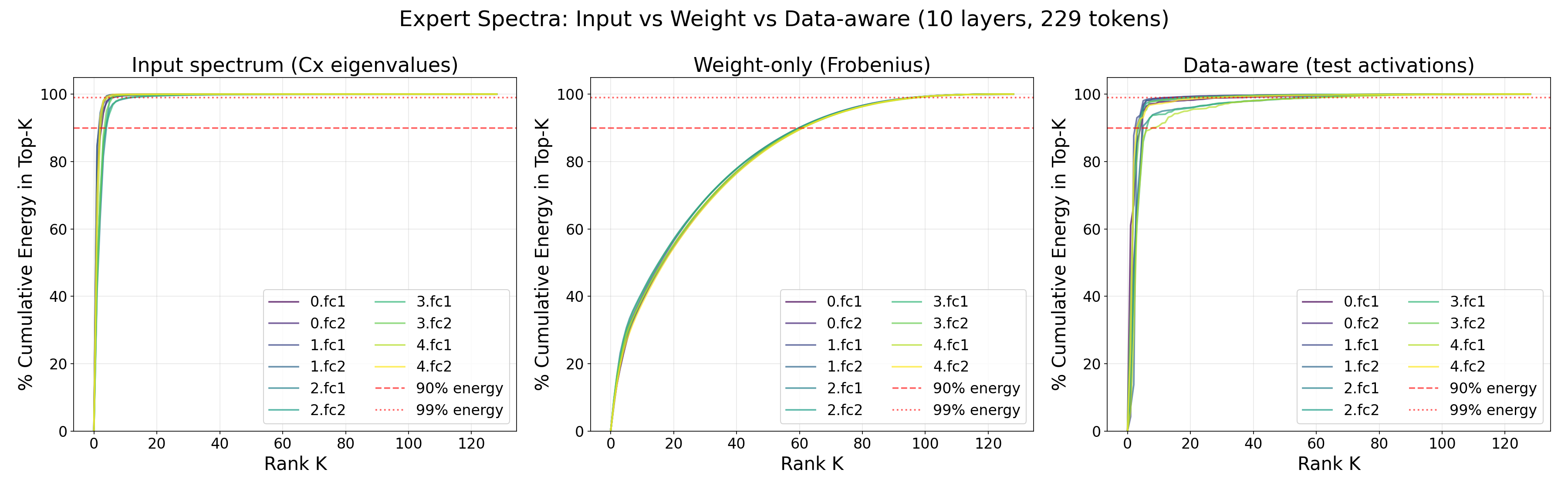}
\caption{RAD (readmission), eICU.}\label{fig:spec-rad}\end{subfigure}
\caption{Expert input spectra after single-task FLAME pretraining on the two eICU tasks.}
\label{fig:spectra-eicu}
\end{figure}

\begin{figure}[h]
\centering
\begin{subfigure}{\textwidth}\includegraphics[width=\linewidth]{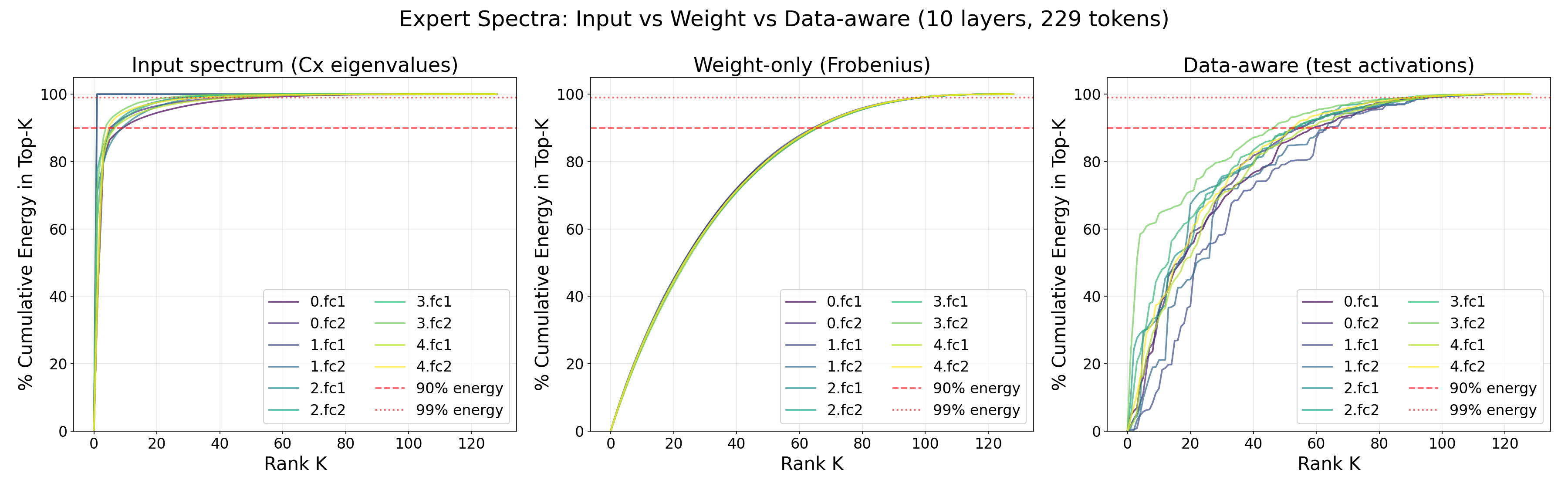}
\caption{BIRADS (breast density category), EMBED.}\label{fig:spec-birads}\end{subfigure}\\[2pt]
\begin{subfigure}{\textwidth}\includegraphics[width=\linewidth]{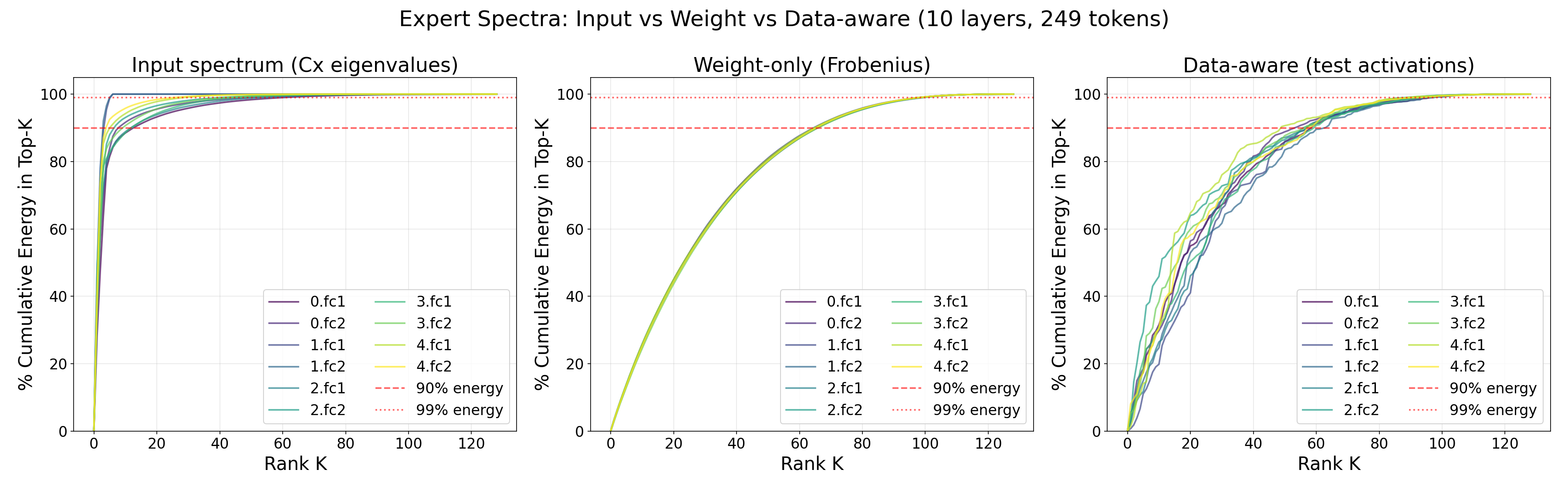}
\caption{DENSITY (breast density score), EMBED.}\label{fig:spec-density}\end{subfigure}\\[2pt]
\begin{subfigure}{\textwidth}\includegraphics[width=\linewidth]{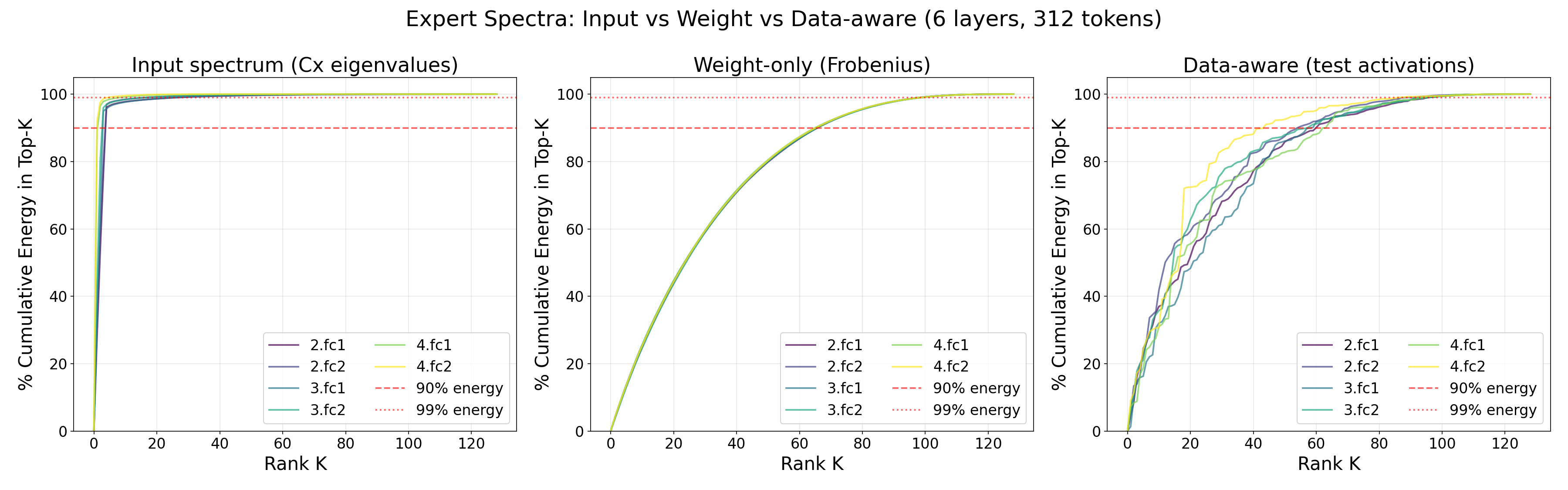}
\caption{RISK (cancer risk), EMBED.}\label{fig:spec-risk}\end{subfigure}
\caption{Expert input spectra after single-task FLAME pretraining on the three EMBED breast-imaging tasks (modalities: 2D-CC, 2D-MLO, CC, MLO mammography views).}
\label{fig:spectra-embed}
\end{figure}

\begin{figure}[h]
\centering
\begin{subfigure}{\textwidth}\includegraphics[width=\linewidth]{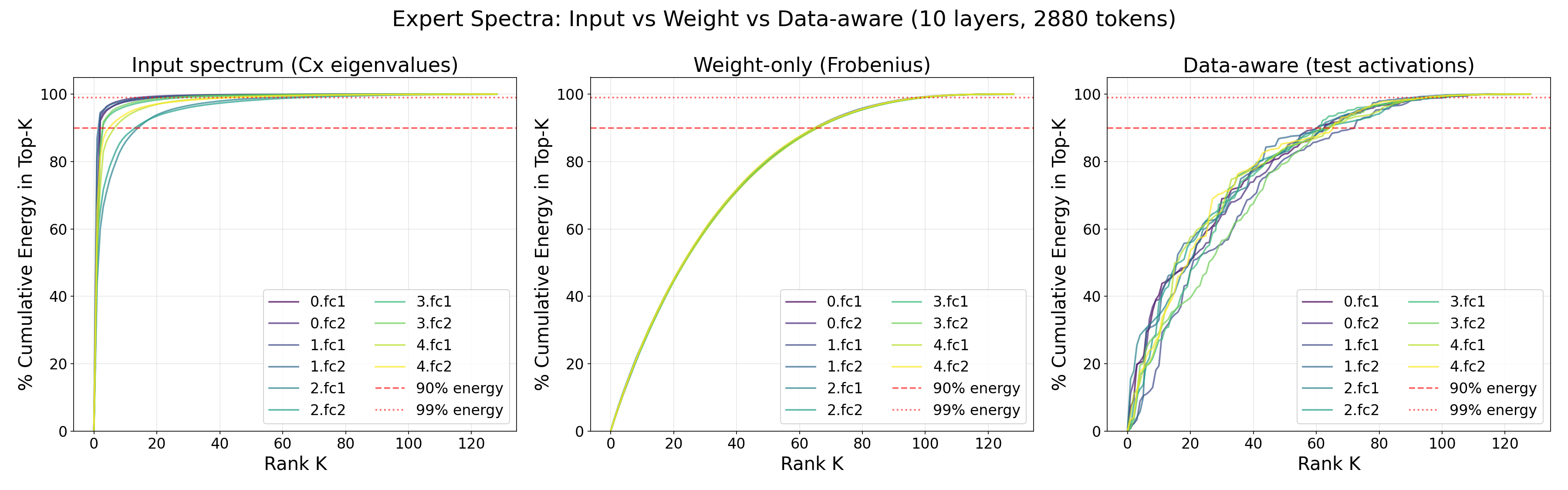}
\caption{DIAG (Alzheimer's diagnosis), ADNI.}\label{fig:spec-diag}\end{subfigure}
\caption{Expert input spectra after single-task FLAME pretraining on ADNI (multi-modal tabular and imaging biomarkers).}
\label{fig:spectra-adni}
\end{figure}

\subsubsection{Multi-Task Pretraining Spectra}\label{app:spectra-multi}

Fig.~\ref{fig:spectra-multi} reports the same three-lens analysis under joint pretraining, both within MIMIC-IV (\{\textsc{IHM, LOS, PHENO}\}) and across all nine tasks (MIMIC-IV $\cup$ eICU $\cup$ EMBED $\cup$ ADNI). Even with a larger and more heterogeneous routed-input distribution, the data-aware functional energy still saturates well before the ambient rank---only marginally later than the single-task spectra above. Multitask routing therefore does not break the low-rank structure that Proposition~\ref{prop:funcrank} relies on; it preserves it. This is the empirical foundation for using a single rank budget $r_t$ during continual stages regardless of how many tasks were jointly pretrained at stage~$0$.

\begin{figure}[h]
\centering
\begin{subfigure}{\textwidth}\includegraphics[width=\linewidth]{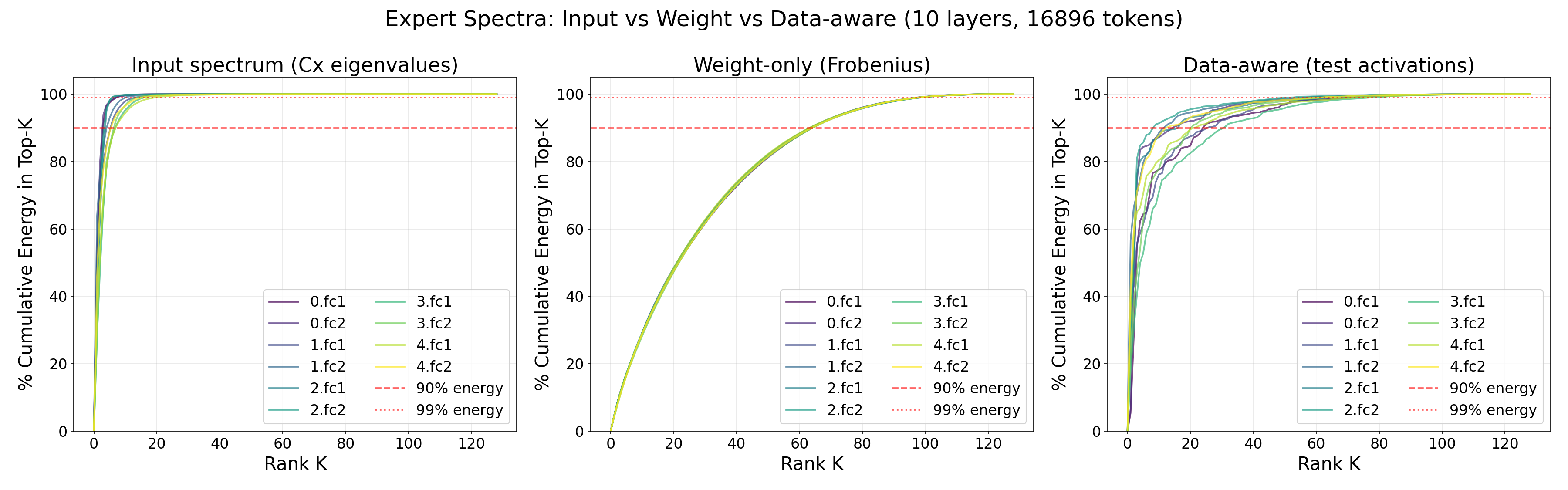}
\caption{Joint MIMIC-IV pretraining over \{IHM, LOS, PHENO\}.}\label{fig:spec-mimicjoint}\end{subfigure}\\[2pt]
\begin{subfigure}{\textwidth}\includegraphics[width=\linewidth]{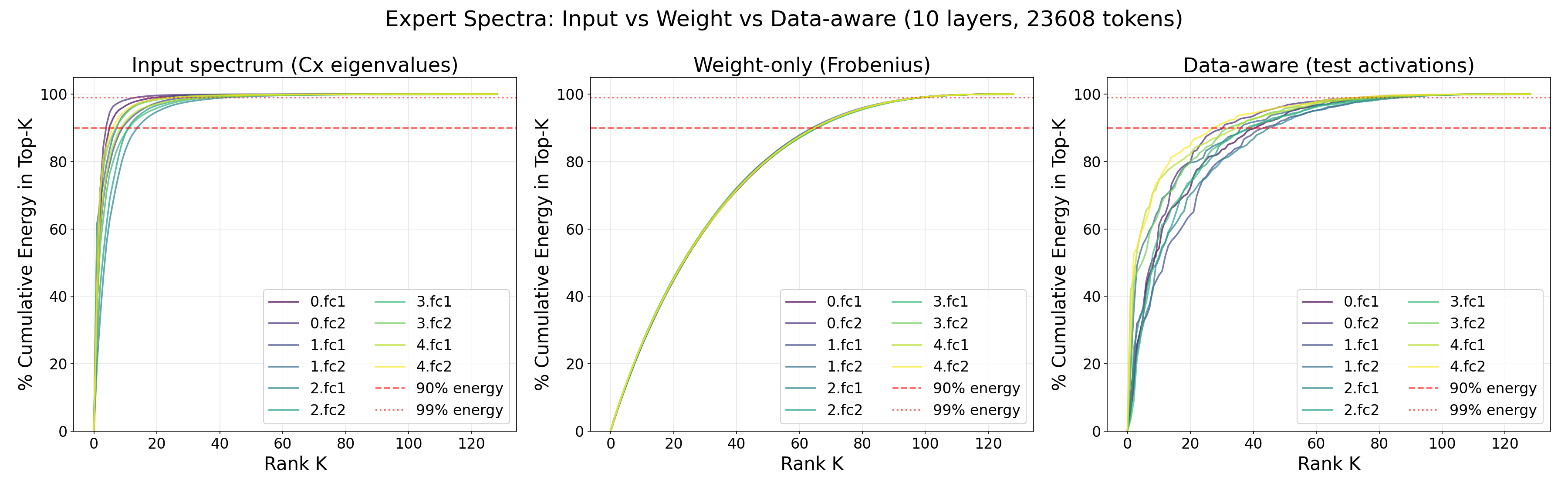}
\caption{Joint pretraining over all 9 tasks (MIMIC-IV $\cup$ eICU $\cup$ EMBED $\cup$ ADNI).}\label{fig:spec-alljoint}\end{subfigure}
\caption{Expert input spectra under joint FLAME pretraining. Despite the larger and more heterogeneous routed-input distribution, the data-aware functional energy still saturates well before the ambient rank, supporting per-task low-rank reservation in continual learning.}
\label{fig:spectra-multi}
\end{figure}

\subsection{Per-Modality Routing Distributions}\label{app:routing}

This section reports the routing patterns learned by FLAME's per-modality routers in the first cross-modal MoE block. For every configuration we plot, for each modality (and each task in the multi-task panels):
\begin{itemize}\itemsep0pt
\item \textbf{Activation ratio} (left): the fraction of modality-$m$ tokens for which expert $e$ is among the top-$K$ activated experts.
\item \textbf{Mean gate weight} (right): the average sparse gate weight assigned to expert $e$, conditioned on $e$ being activated.
\end{itemize}
Modality-specific routers concentrate each modality on a small subset of experts, while different modalities reuse overlapping experts only when their representations are complementary---the empirical basis for the per-modality routing design of Sec.~\ref{subsec:multitask}. In the multi-task panels, comparing rows reveals which task pairs share modality-conditioned experts and which carve out disjoint capacity.

\subsubsection{Single-Task Routing}\label{app:routing-single}

Figs.~\ref{fig:routing-mimic}--\ref{fig:routing-adni} show per-modality routing for each task trained individually. The MIMIC-IV triple (Fig.~\ref{fig:routing-mimic}) gives a clear illustration of the design's intent: chest X-ray, clinical text, and time series each concentrate on a distinct pair of experts, with mild reuse on a shared expert when the modalities are complementary for the task. The eICU pair (Fig.~\ref{fig:routing-eicu}) follows the same structure on its two-modality input. EMBED's four mammography views (Fig.~\ref{fig:routing-embed}) split cleanly between view-specific experts, confirming that the router exploits the same-modality-different-view signal supplied by breast laterality. ADNI's five biomarker modalities (Fig.~\ref{fig:routing-adni}) occupy expert positions that are essentially disjoint from the other datasets', reflecting its unique modality pool.

\begin{figure}[h]
\centering
\begin{subfigure}{0.49\textwidth}\includegraphics[width=\linewidth]{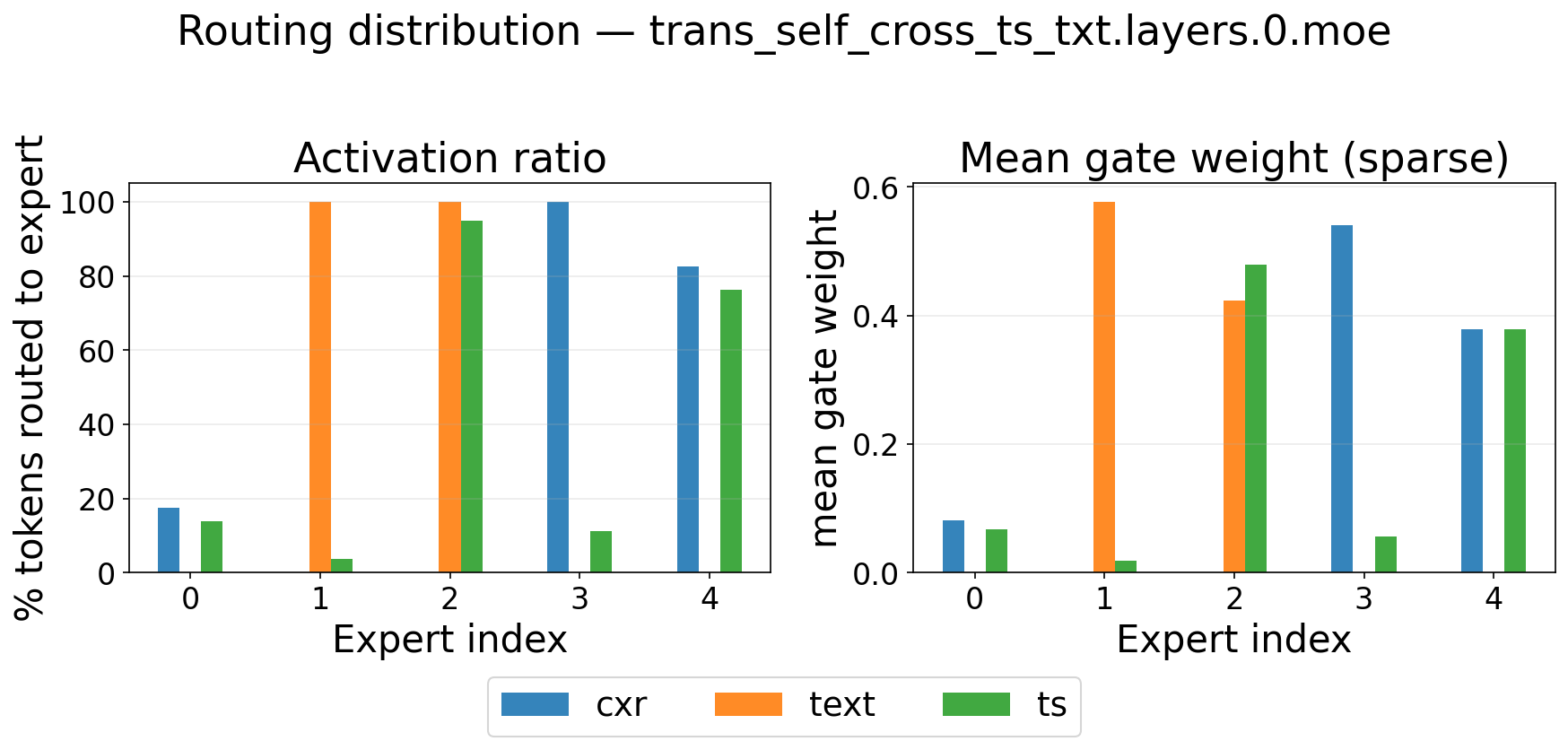}
\caption{48-IHM, MIMIC-IV.}\label{fig:rout-ihm}\end{subfigure}
\hfill
\begin{subfigure}{0.49\textwidth}\includegraphics[width=\linewidth]{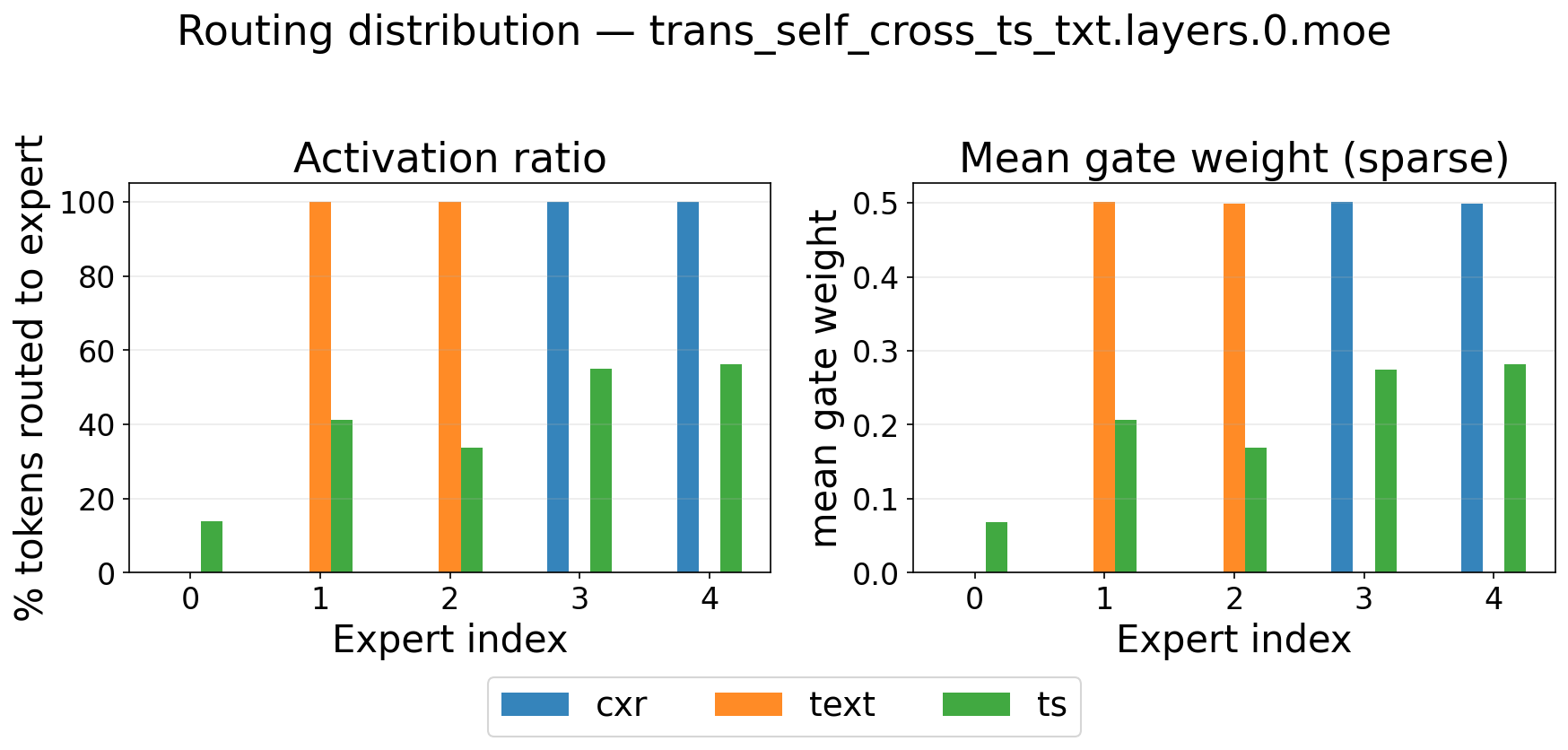}
\caption{LOS, MIMIC-IV.}\label{fig:rout-los}\end{subfigure}\\[4pt]
\begin{subfigure}{0.49\textwidth}\includegraphics[width=\linewidth]{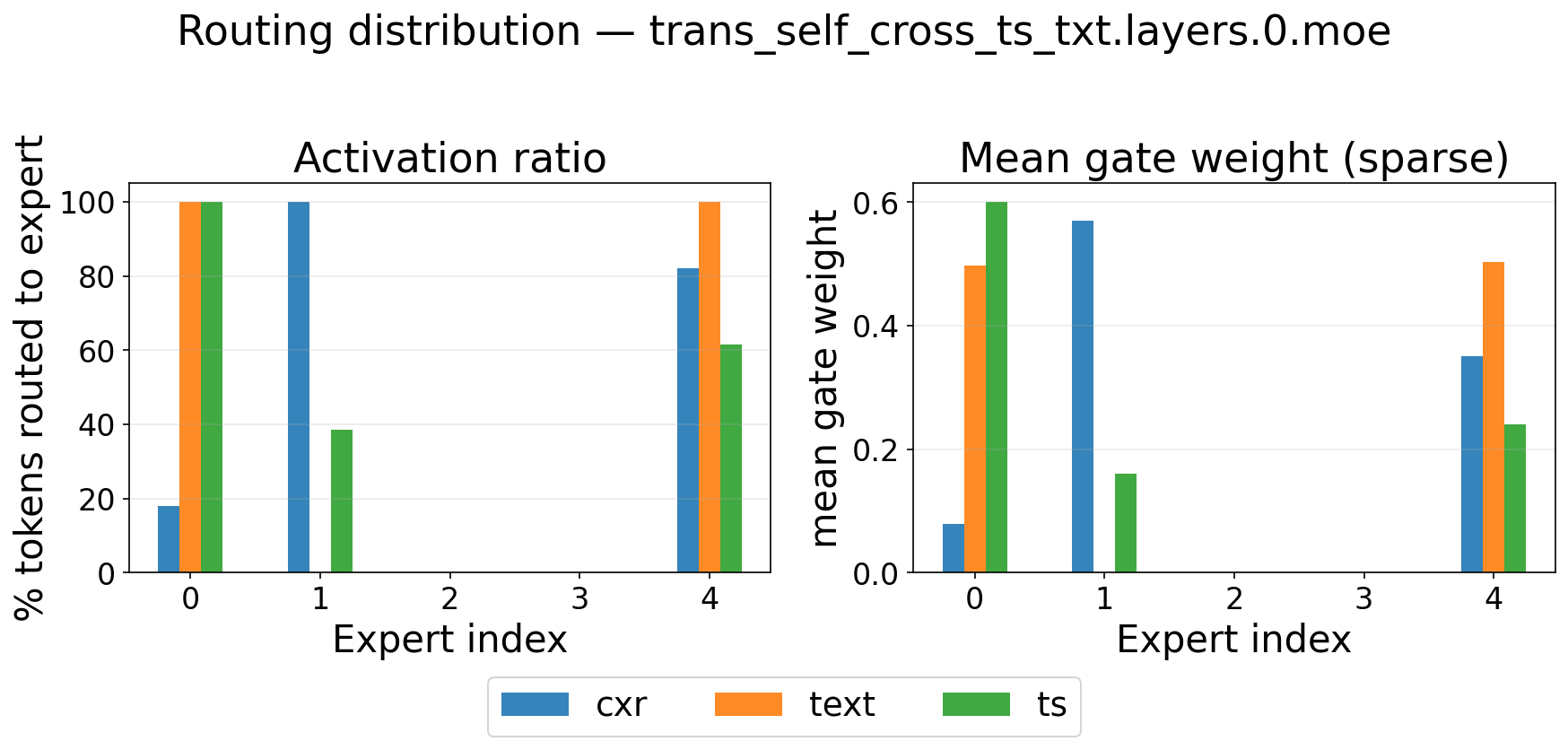}
\caption{25-PHENO, MIMIC-IV.}\label{fig:rout-pheno}\end{subfigure}
\caption{Per-modality routing for the three MIMIC-IV tasks. Modalities: chest X-ray (\texttt{cxr}), clinical text (\texttt{text}), time series (\texttt{ts}). Each modality concentrates on a distinct pair of experts, with mild cross-modality reuse on the shared expert.}
\label{fig:routing-mimic}
\end{figure}

\begin{figure}[h]
\centering
\begin{subfigure}{0.49\textwidth}\includegraphics[width=\linewidth]{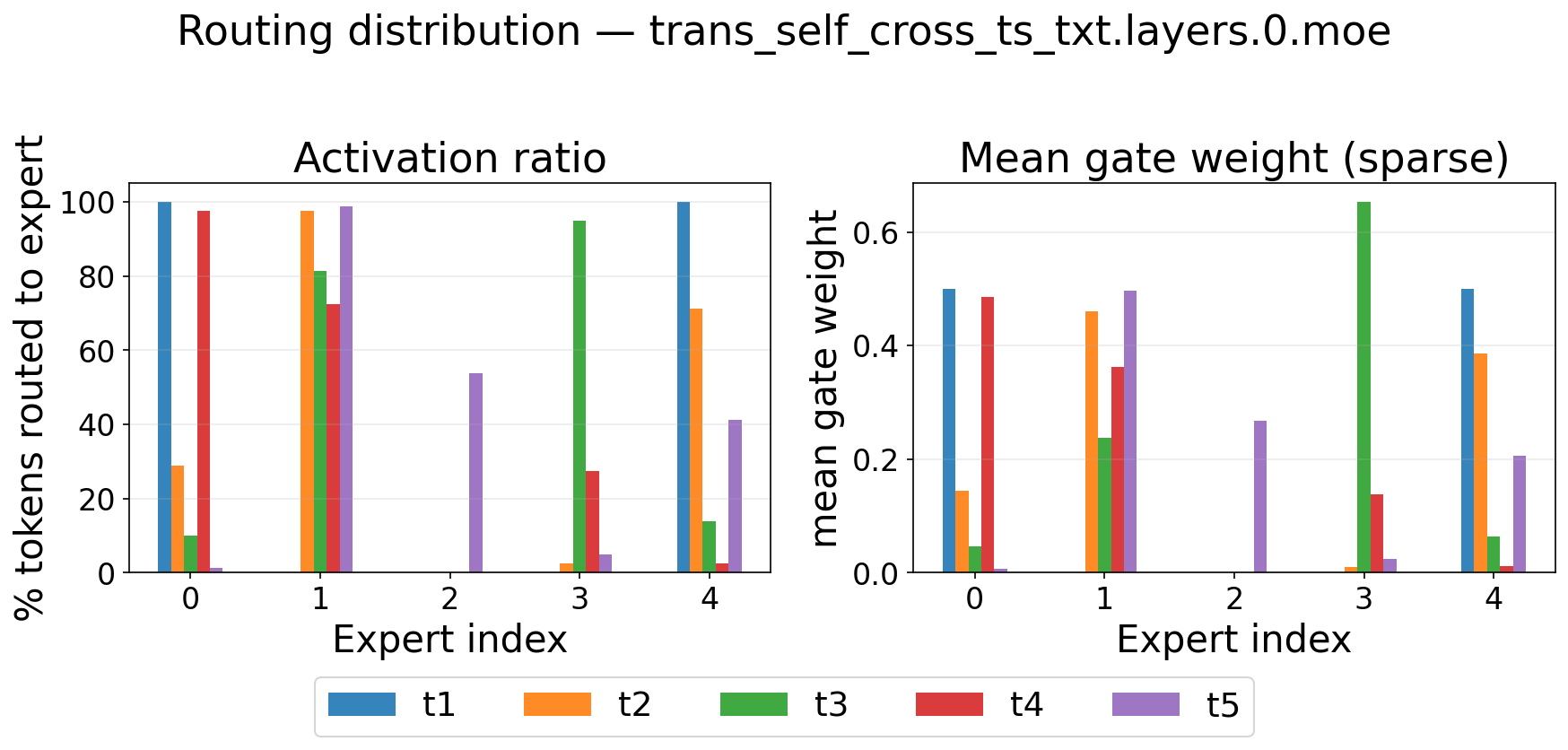}
\caption{MOR (mortality), eICU.}\label{fig:rout-mor}\end{subfigure}
\hfill
\begin{subfigure}{0.49\textwidth}\includegraphics[width=\linewidth]{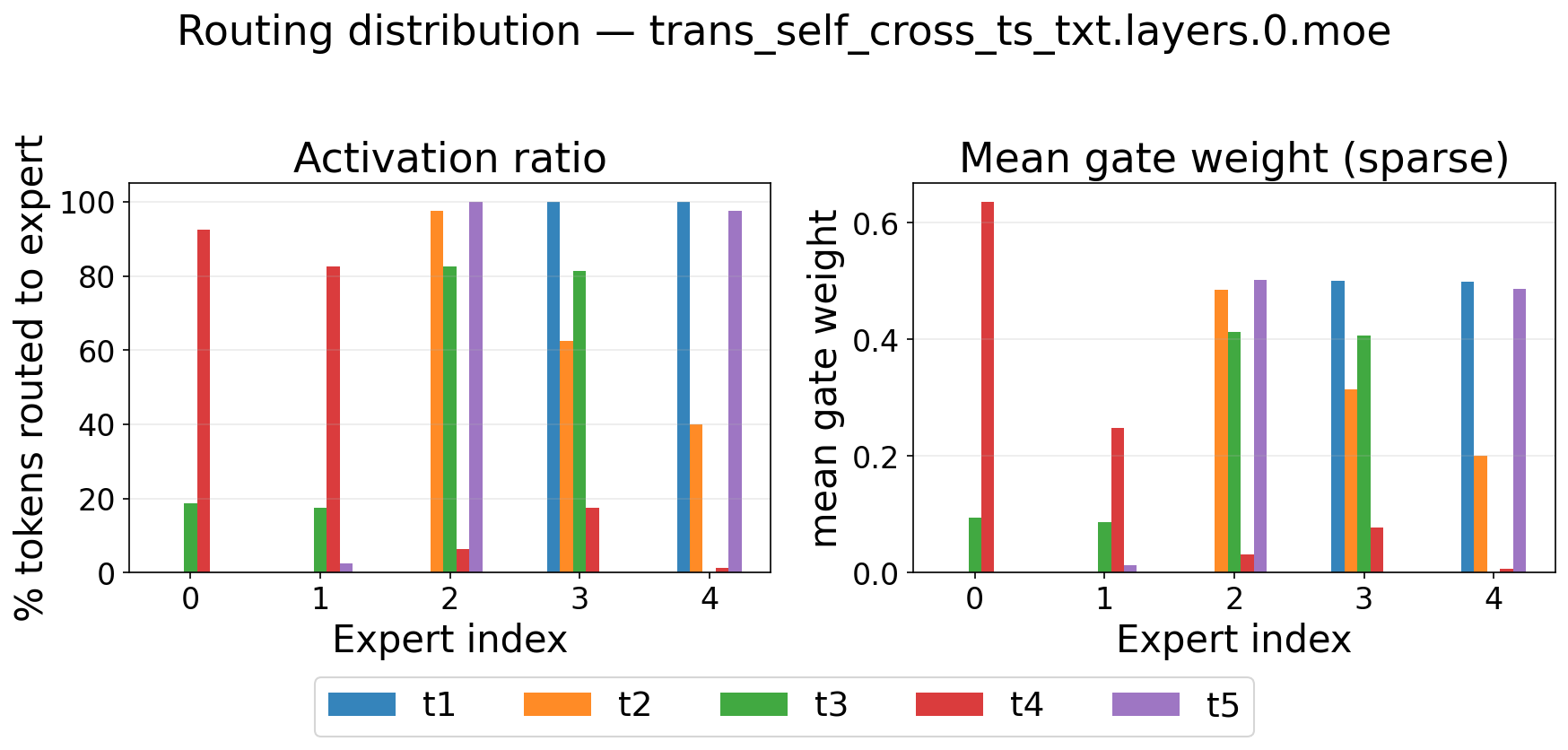}
\caption{RAD (readmission), eICU.}\label{fig:rout-rad}\end{subfigure}
\caption{Per-modality routing for the two eICU tasks.}
\label{fig:routing-eicu}
\end{figure}

\begin{figure}[h]
\centering
\begin{subfigure}{0.49\textwidth}\includegraphics[width=\linewidth]{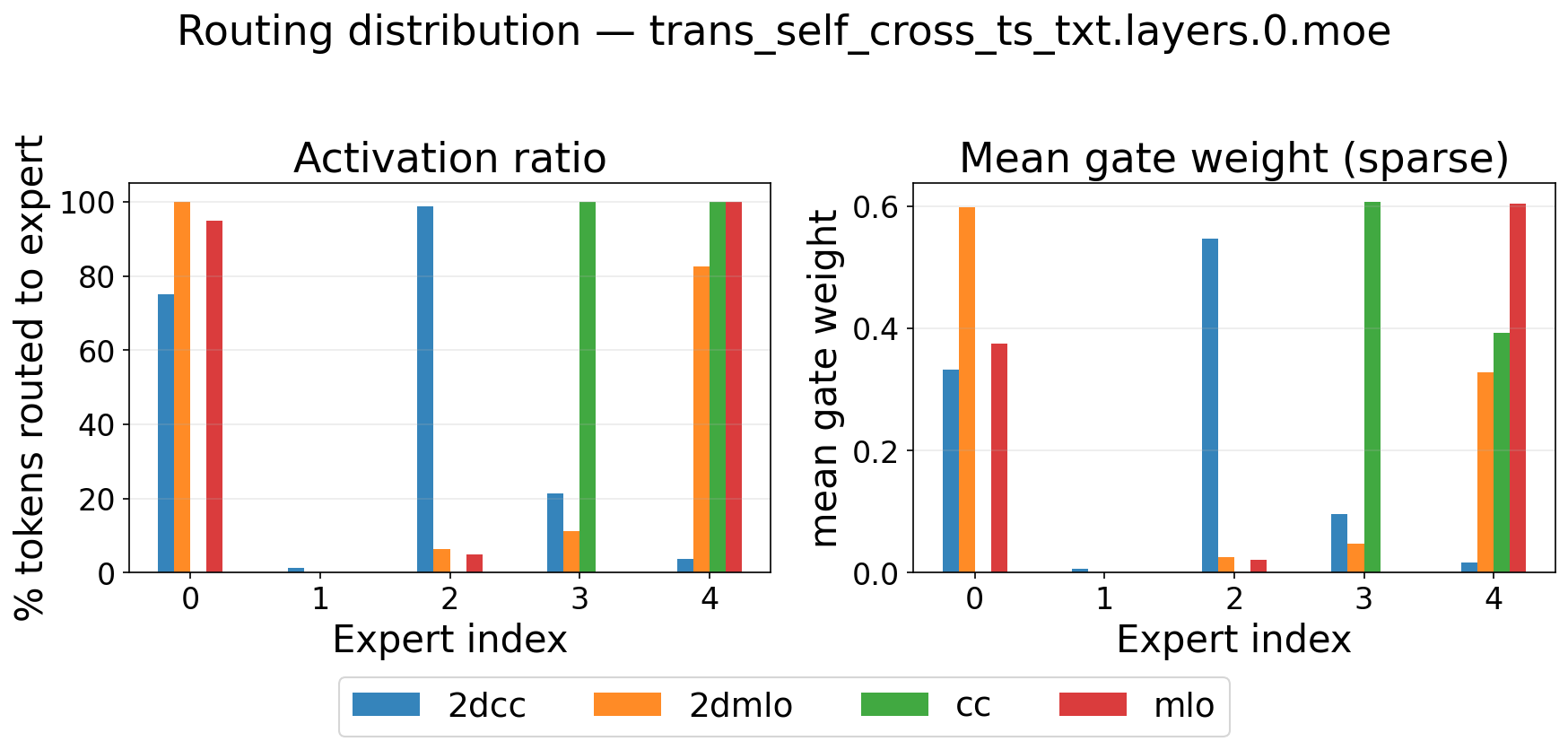}
\caption{BIRADS, EMBED.}\label{fig:rout-birads}\end{subfigure}
\hfill
\begin{subfigure}{0.49\textwidth}\includegraphics[width=\linewidth]{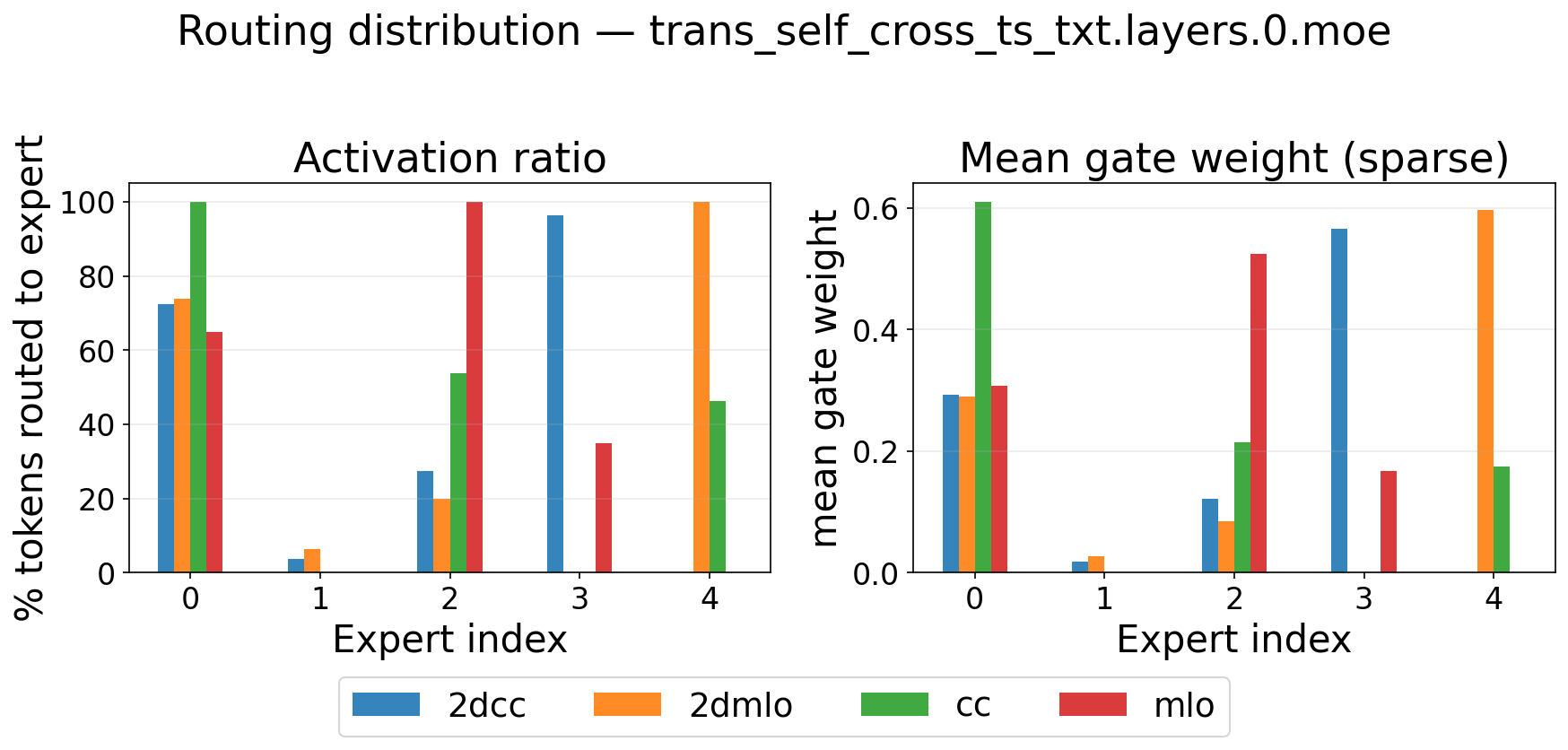}
\caption{DENSITY, EMBED.}\label{fig:rout-density}\end{subfigure}\\[4pt]
\begin{subfigure}{0.49\textwidth}\includegraphics[width=\linewidth]{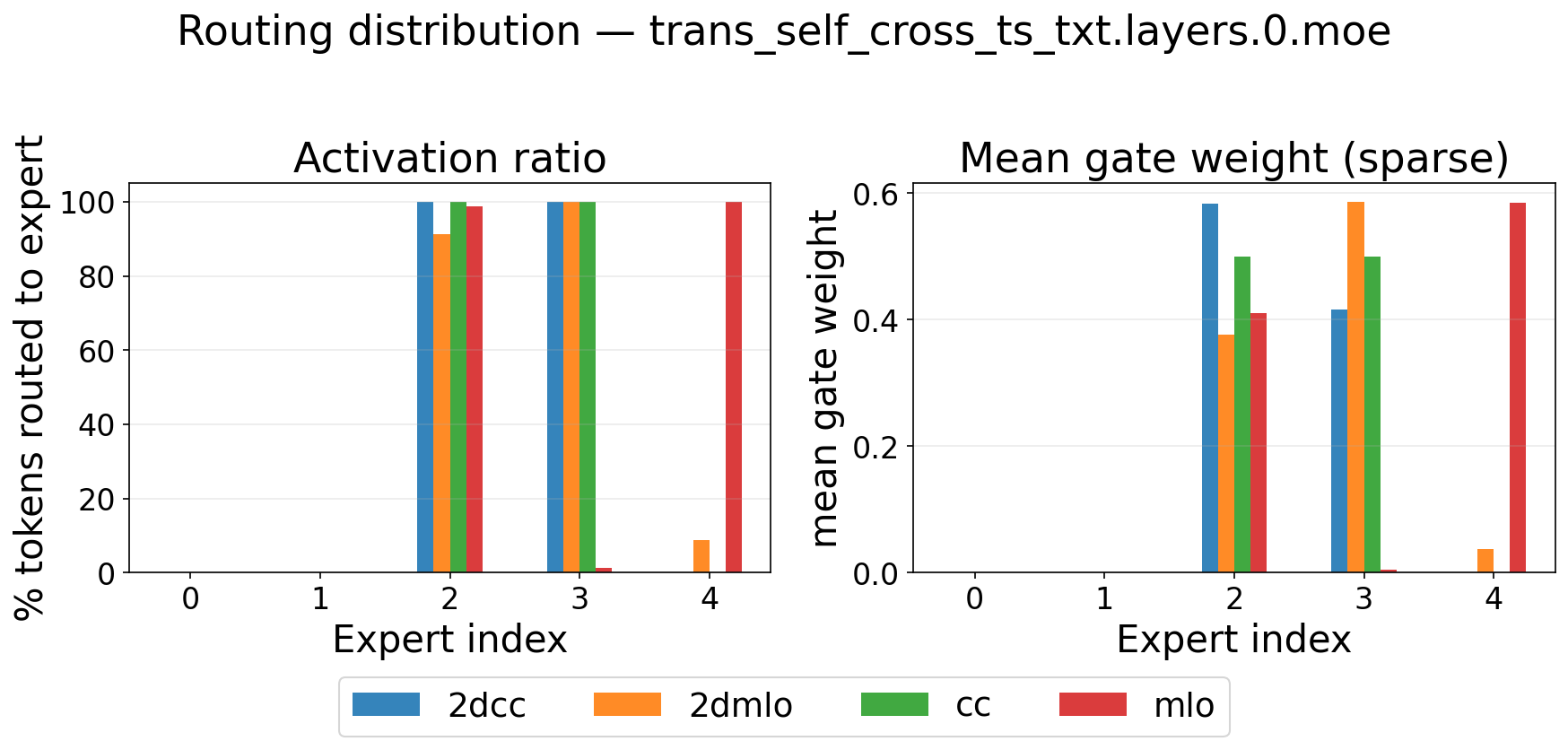}
\caption{RISK, EMBED.}\label{fig:rout-risk}\end{subfigure}
\caption{Per-modality routing for the three EMBED breast-imaging tasks. Mammography views (\texttt{2dcc}, \texttt{2dmlo}, \texttt{cc}, \texttt{mlo}) split cleanly between view-specific experts.}
\label{fig:routing-embed}
\end{figure}

\begin{figure}[h]
\centering
\begin{subfigure}{0.7\textwidth}\includegraphics[width=\linewidth]{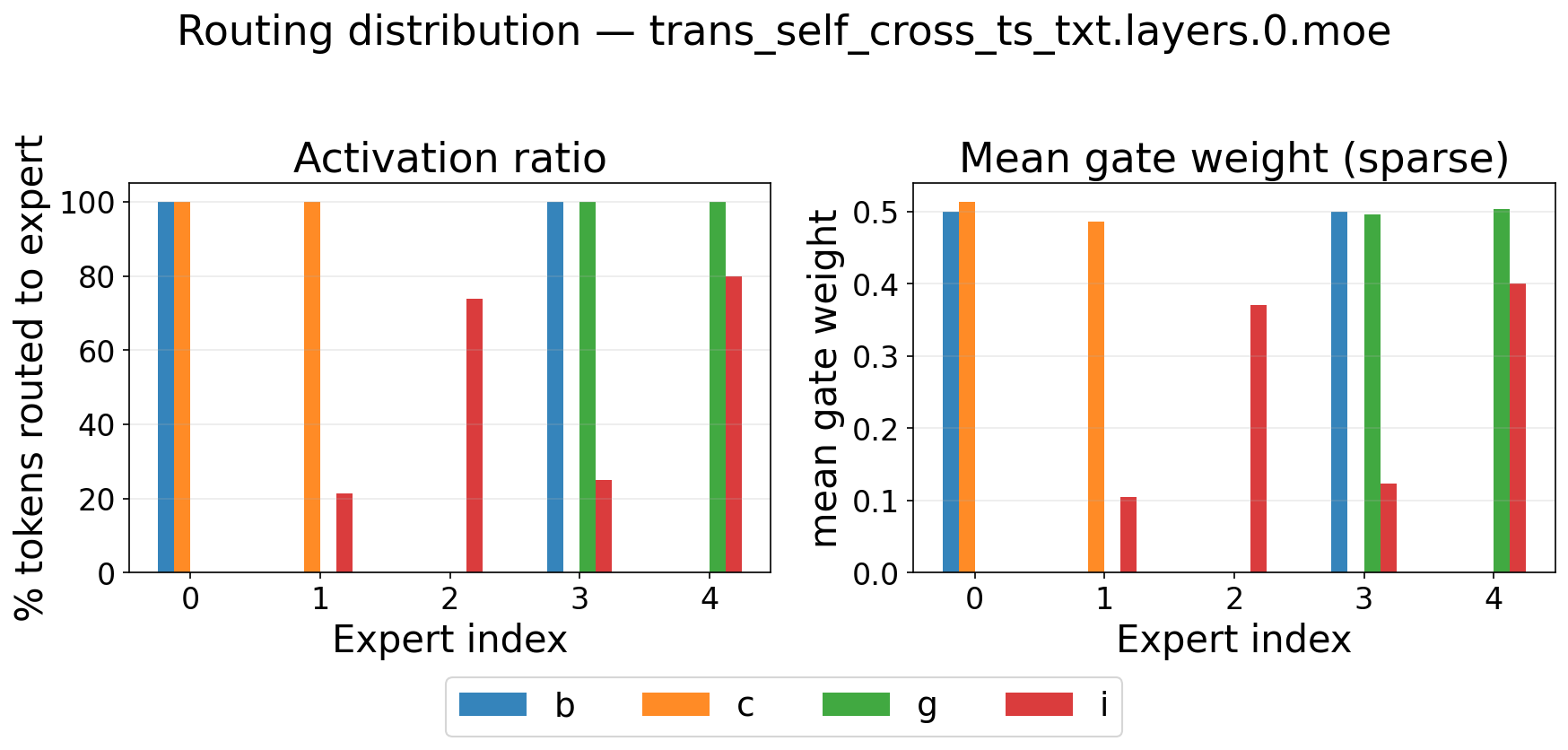}
\caption{DIAG, ADNI.}\label{fig:rout-diag}\end{subfigure}
\caption{Per-modality routing for ADNI Alzheimer's diagnosis (T1--T5 imaging biomarkers).}
\label{fig:routing-adni}
\end{figure}

\subsubsection{Multi-Task Routing}\label{app:routing-multi}

Figs.~\ref{fig:routing-mimicjoint}, \ref{fig:routing-alljoint} report the per-task, per-modality routing under joint training. Within MIMIC-IV (Fig.~\ref{fig:routing-mimicjoint}), \textsc{IHM} and \textsc{LOS} share an almost identical routing fingerprint---the structural reason their multitask AUROC reinforces---whereas \textsc{25-PHENO} carves out a distinct subset via the \texttt{cxr\_pheno}, \texttt{text\_pheno}, and \texttt{ts\_pheno} per-task router heads. In the nine-task panel (Fig.~\ref{fig:routing-alljoint}), the EMBED triple \textsc{BIRADS}/\textsc{DENSITY}/\textsc{RISK} collapses onto two shared experts; the eICU pair \textsc{MOR}/\textsc{RAD} settles on another; ADNI's biomarker tasks occupy expert positions that no other dataset reaches. This expert-level task clustering is the data-driven evidence behind FLAME's joint-training pairing recommendations and reads consistently with the pairwise heatmap in Sec.~\ref{app:pairwise}.

\begin{figure}[h]
\centering
\includegraphics[width=0.75\linewidth]{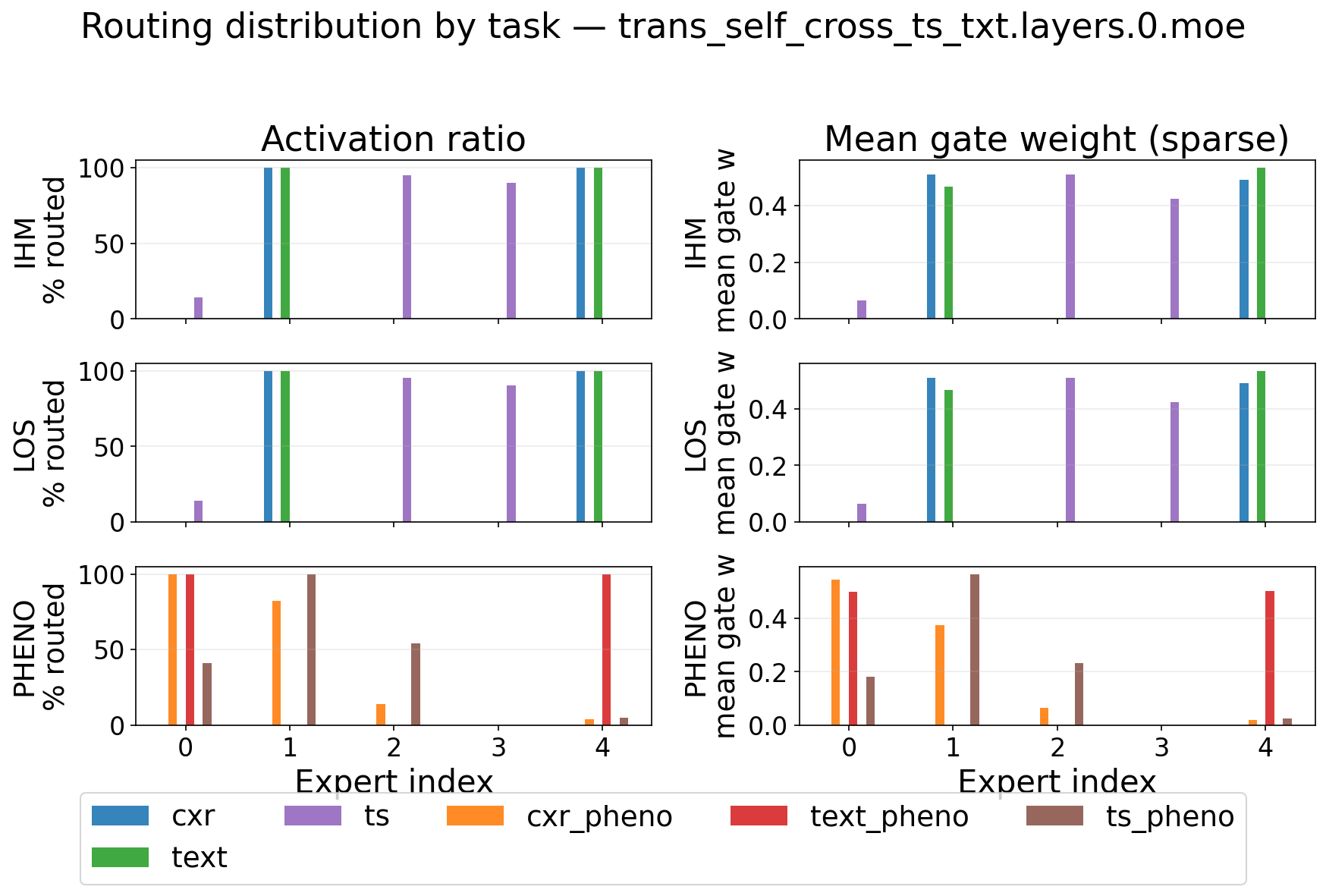}
\caption{Per-task, per-modality routing under joint MIMIC-IV pretraining over \{IHM, LOS, PHENO\}. Each row corresponds to one task. IHM and LOS share an almost identical routing fingerprint, whereas PHENO carves out a distinct subset of experts via the \texttt{cxr\_pheno}, \texttt{text\_pheno} and \texttt{ts\_pheno} per-task router heads.}
\label{fig:routing-mimicjoint}
\end{figure}

\begin{figure}[h]
\centering
\includegraphics[width=0.7\linewidth]{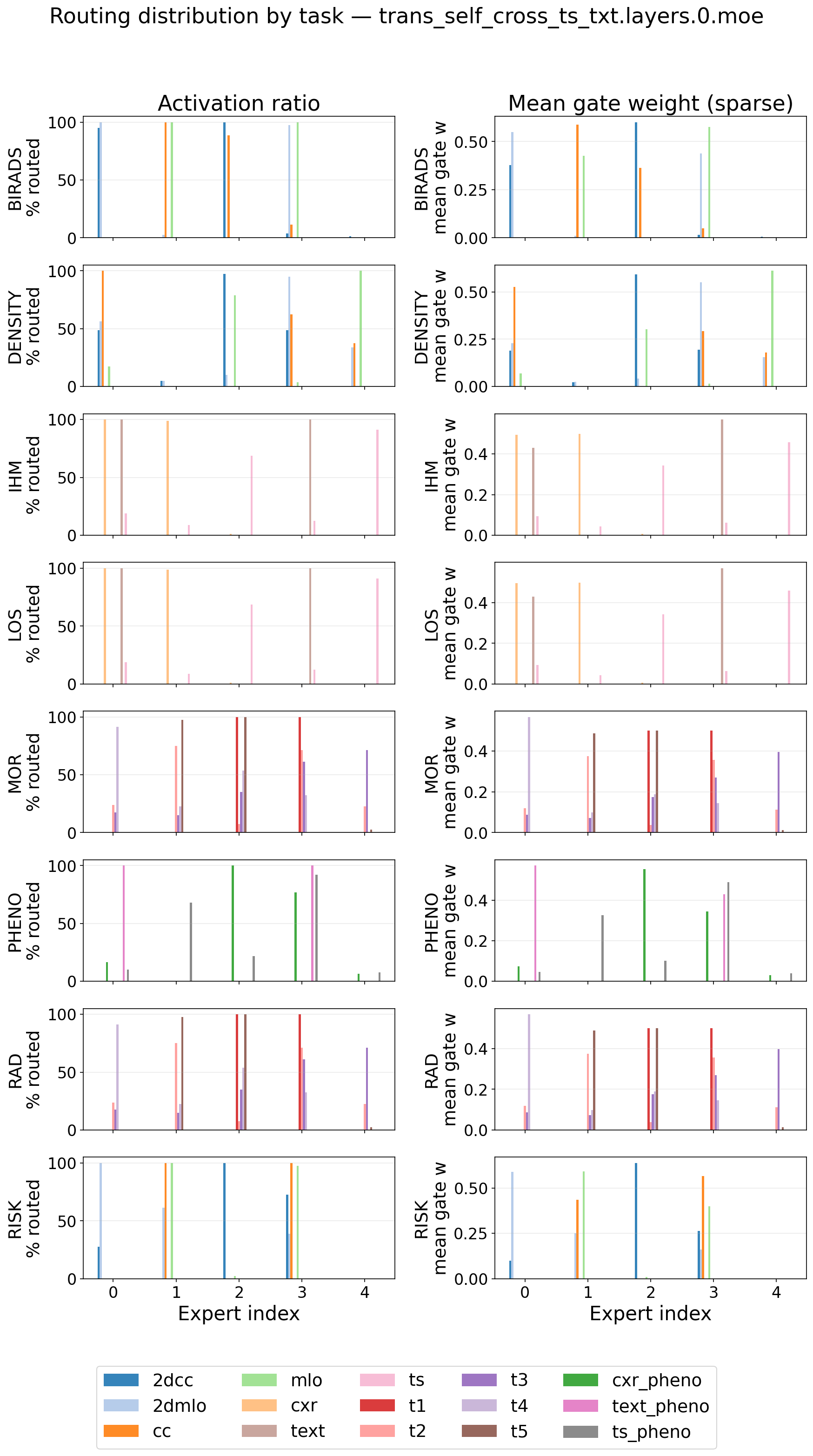}
\caption{Per-task, per-modality routing under joint pretraining over all 9 tasks across MIMIC-IV, eICU, EMBED, and ADNI. Rows correspond to tasks; columns are the same five experts. Tasks that share modality semantics (e.g., the EMBED triple BIRADS / DENSITY / RISK on mammography views, or the eICU pair MOR / RAD on time series and text) collapse onto overlapping expert subsets, while tasks with disjoint modality pools occupy disjoint experts. This expert-level task clustering provides the data-driven evidence behind FLAME's joint-training pairing recommendations.}
\label{fig:routing-alljoint}
\end{figure}

\subsubsection{Continual-Stage Routing}\label{app:routing-cl}

This subsection extends the routing analysis from joint multi-task pretraining (Sec.~\ref{app:routing-multi}) to the continual-stage setting: Figs.~\ref{fig:routing-cl-mimic}--\ref{fig:routing-cl-mixed} visualize how the per-stage router heads $\{G_m^{(k)}\}_m$ (Algorithm~\ref{alg:ours}) deploy each modality across the shared expert pool at every stage of the four continual sequences from Sec.~\ref{app:cl_results}. Each panel row corresponds to one stage with the tasks introduced at that stage labeled; columns within a row are the modalities used at that stage; the heatmap inside each block is the per-expert activation ratio. Two structural properties of \flamecl{} read directly off the panels. 

\textbf{Frozen prior fingerprints.} Once stage~$k$ converges, the router heads $\{G_m^{(k)}\}_m$ and the rank-$r_k$ truncations $\{W_i^{(k)}\}$ are frozen (Sec.~\ref{subsec:cl}); the routing fingerprints assigned at stage~$k$ therefore appear unchanged in every later stage's panels --- rows do not shift, swap experts, or reweight when subsequent tasks are added. This is the routing-level reading of the cursor identity in Eq.~\eqref{eq:cursor}: under cursor $\tau = k(t)$, task $t$ at inference sees only its own stage's router heads, and new stages cannot perturb them. 

\textbf{Within-dataset reuse vs.\ cross-dataset isolation.} For sequences whose stages share modality semantics (S1: \textsc{PHENO}$\to$\textsc{LOS}$\to$\textsc{IHM} on MIMIC-IV in Fig.~\ref{fig:routing-cl-mimic}; S2: \textsc{MOR}$\to$\textsc{RAD} on eICU in Fig.~\ref{fig:routing-cl-eicu}; S3: \textsc{DENSITY}$\to$\textsc{BIRADS}$\to$\textsc{RISK} on EMBED in Fig.~\ref{fig:routing-cl-embed}), each new stage's router lands on the same modality-conditioned experts that prior stages occupied, with at most mild reweighting --- knowledge transfer at the architectural level rather than at the parameter level. In the mixed cross-dataset sequence (S4, Fig.~\ref{fig:routing-cl-mixed}), tasks from different datasets occupy disjoint expert subsets at every stage, mirroring the all-nine-task pretraining panel (Fig.~\ref{fig:routing-alljoint}); cross-dataset interference is bounded structurally, with no replay buffer or alignment loss. Together with the AUROC and parameter trajectories of Sec.~\ref{app:cl_results}, these visualizations explain \emph{how} \flamecl{} delivers within-$0.01$ AUROC retention at $5{-}15{\times}$ fewer encoder parameters: shared-modality stages share experts, disjoint-modality stages occupy disjoint experts, and prior fingerprints are never overwritten.

\begin{figure}[h]
\centering
\includegraphics[width=\linewidth]{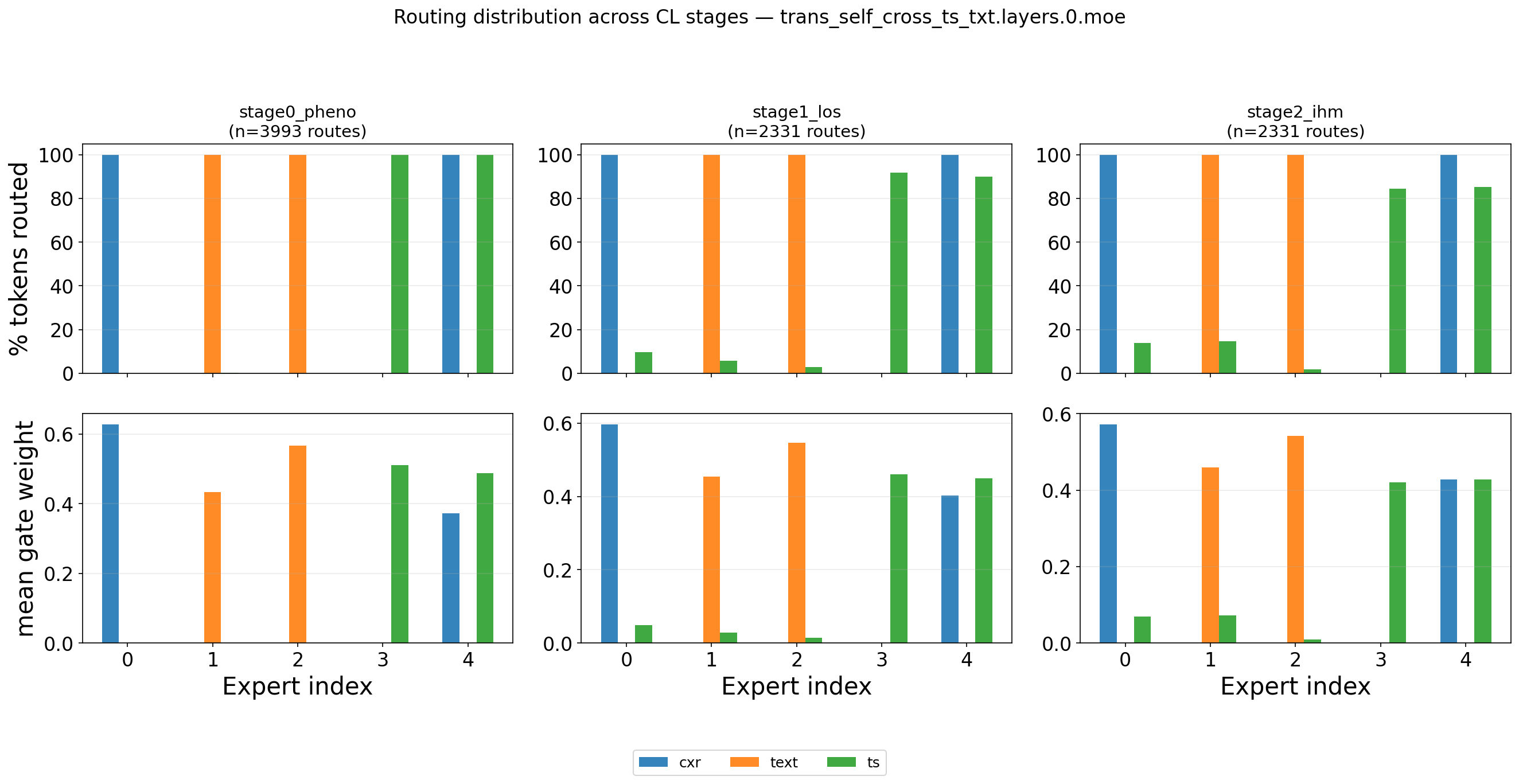}
\caption{S1: per-modality, per-stage routing under MIMIC-IV continual learning \textsc{PHENO}$\to$\textsc{LOS}$\to$\textsc{IHM}. Modalities at every stage: time series (\texttt{TS}), clinical text (\texttt{Text}), chest X-ray (\texttt{CXR}). Each subsequent stage's fingerprint overlaps heavily with stage~0's, which is the structural reason \flamecl{} retains prior-task AUROC within $0.01$ on this sequence (Fig.~\ref{fig:cl-grid}, S1).}
\label{fig:routing-cl-mimic}
\end{figure}

\begin{figure}[h]
\centering
\includegraphics[width=\linewidth]{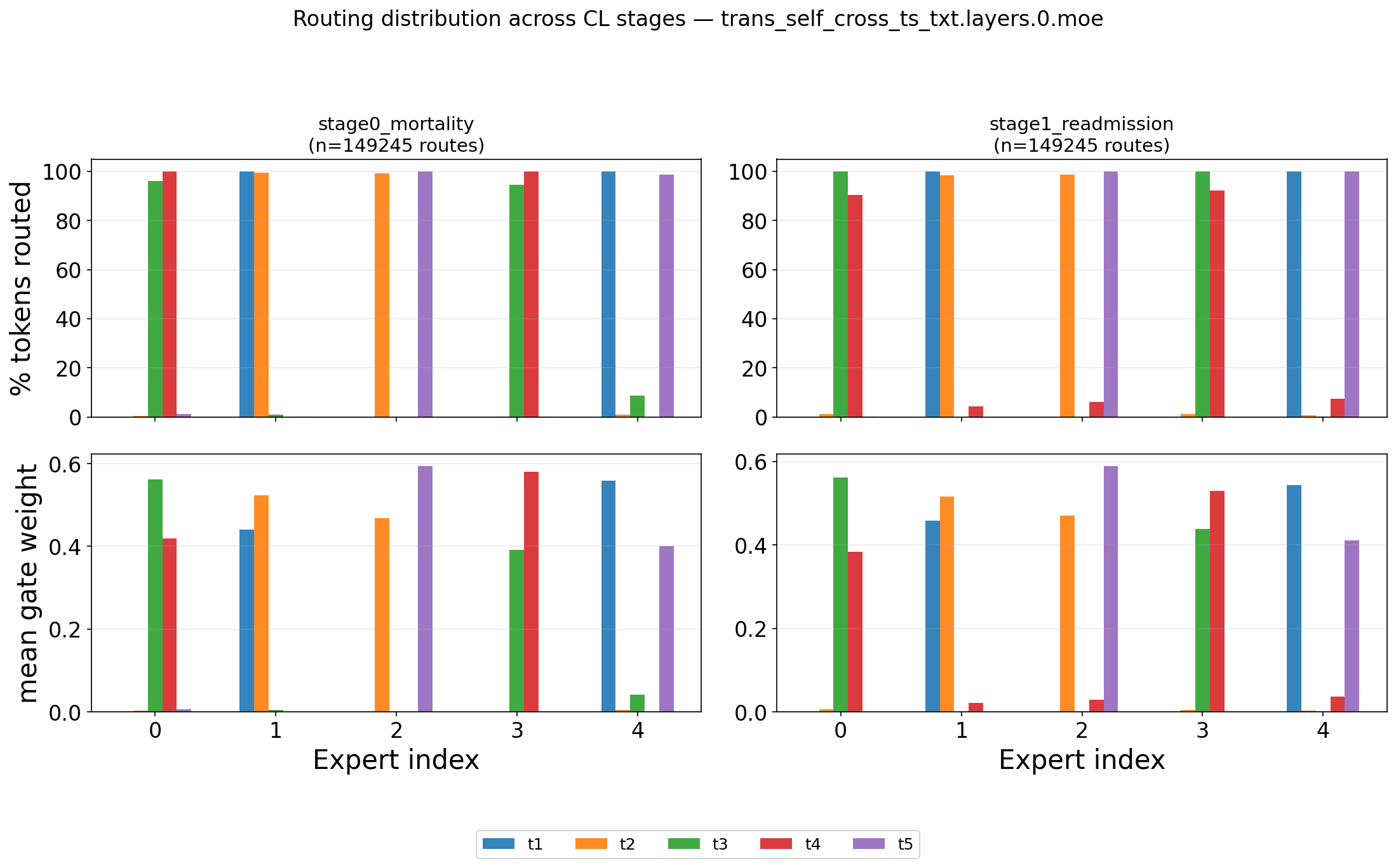}
\caption{S2: per-modality, per-stage routing under eICU continual learning \textsc{MOR}$\to$\textsc{RAD}. The two tasks share the same multivariate time-series modality bank (\texttt{T1}--\texttt{T5}), and the corresponding routing fingerprints overlap almost identically across stages.}
\label{fig:routing-cl-eicu}
\end{figure}

\begin{figure}[h]
\centering
\includegraphics[width=\linewidth]{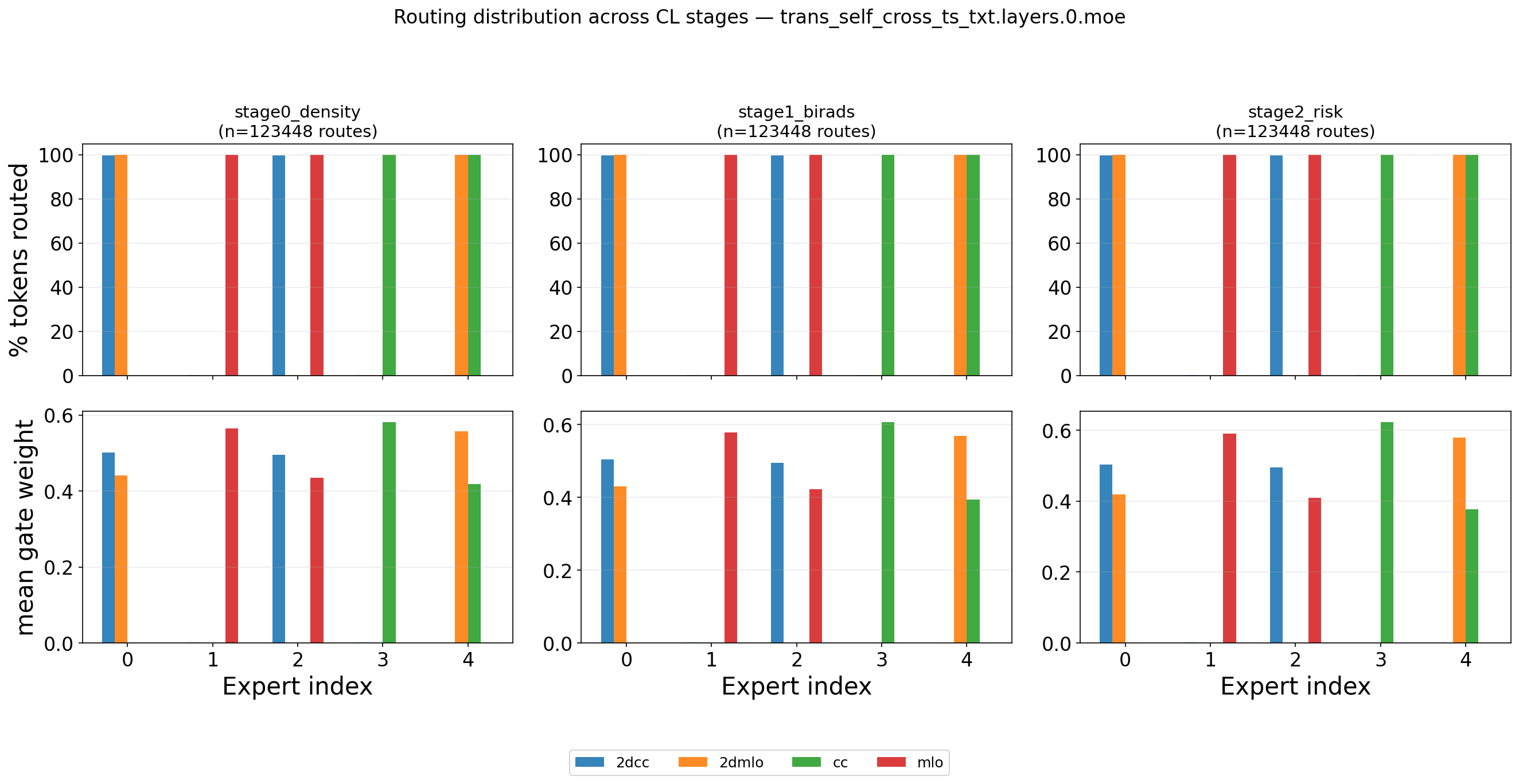}
\caption{S3: per-modality, per-stage routing under EMBED continual learning \textsc{DENSITY}$\to$\textsc{BIRADS}$\to$\textsc{RISK}. The four mammography views (\texttt{cc}, \texttt{mlo}, \texttt{2dcc}, \texttt{2dmlo}) keep their view-specific expert assignments across all three stages with negligible drift, providing the routing-level analogue to the multi-task panel of Fig.~\ref{fig:routing-embed}.}
\label{fig:routing-cl-embed}
\end{figure}

\begin{figure}[h]
\centering
\includegraphics[width=\linewidth]{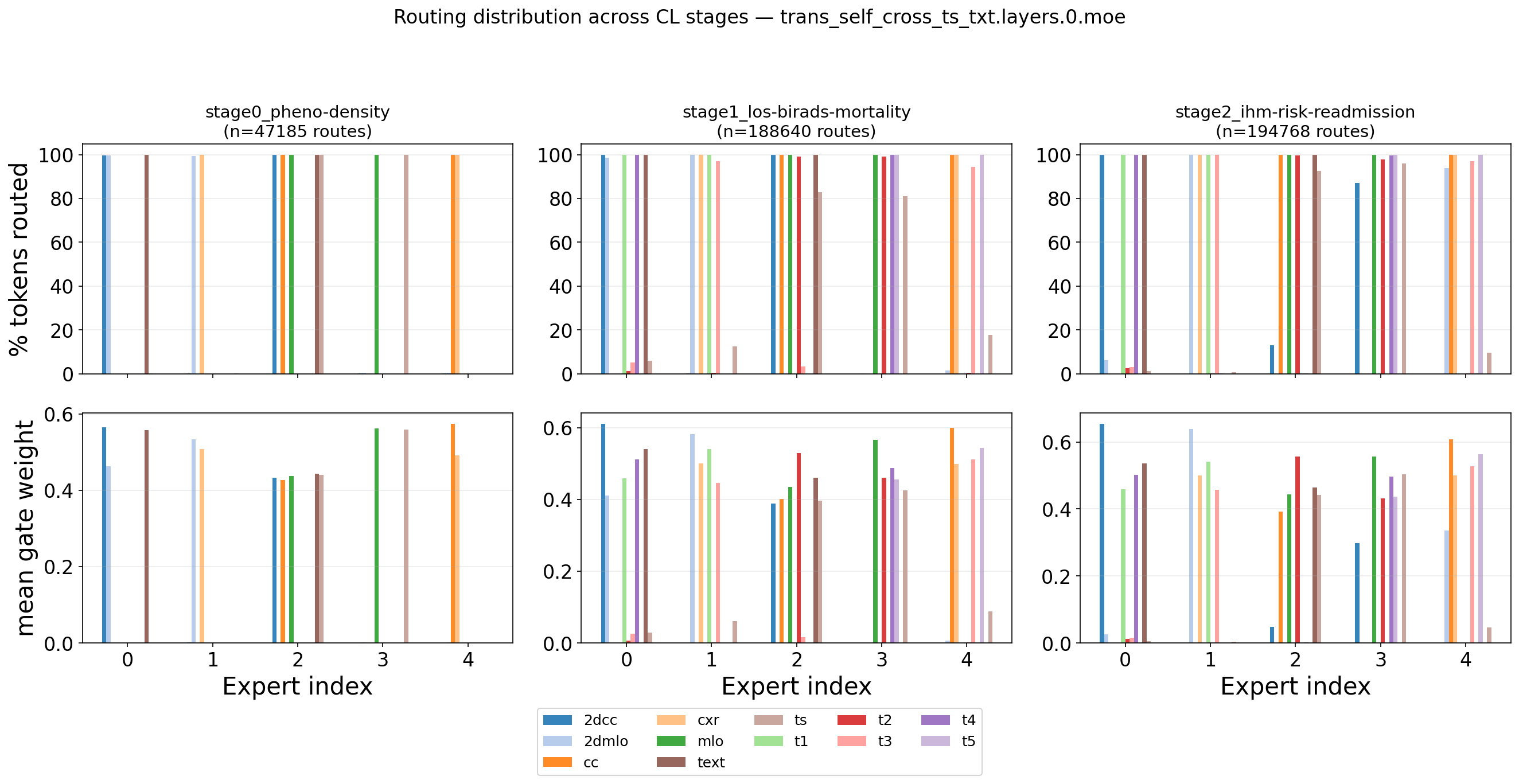}
\caption{S4: per-modality, per-stage routing under the mixed cross-dataset continual sequence \{\textsc{PHENO},\,\textsc{DENSITY}\}$\to$\{\textsc{LOS},\,\textsc{BIRADS},\,\textsc{MOR}\}$\to$\{\textsc{IHM},\,\textsc{RISK},\,\textsc{RAD}\}. MIMIC-IV (\texttt{TS}/\texttt{Text}/\texttt{CXR}), EMBED (\texttt{cc}/\texttt{mlo}/\texttt{2dcc}/\texttt{2dmlo}), and eICU (\texttt{T1}--\texttt{T5}) modalities occupy disjoint expert subsets that persist across stages; cross-dataset interference is therefore bounded structurally rather than by an auxiliary loss.}
\label{fig:routing-cl-mixed}
\end{figure}


\subsection{Extended Discussion of Research Questions}\label{app:extended_discussion}

The brief RQ answers in Sec.~\ref{subsec:discussion} are summaries; below we restate the full long-form discussion, including interpretive framing that ties results back to the design choices of Sec.~\ref{sec:method}. Numerical detail referenced here lives in Secs.~\ref{app:full_table}, \ref{app:cl_walkthrough}, \ref{app:spectra}, \ref{app:routing}.

\textbf{(RQ1) Headline performance.} \flame{}'s per-modality router matches HighMMT and FlexMoE on AUROC across the standard single-task setting and overtakes them on AUPRC for the imbalanced binary tasks (\textsc{IHM}, \textsc{RISK}, \textsc{DIAG}), where modality-conditioned dispatch sharpens the boundary on rare positives. Joint training under the same routing helps when modality semantics align across tasks, with \textsc{48-IHM} and \textsc{LOS} both gaining $0.01$ AUROC under per-dataset multitask, and the cost is small when they do not, with a $0.04$ AUROC drop on \textsc{DIAG}, the only task whose modality pool shares no encoder with any other. The all-nine multitask FLAME run still trails the per-task baselines by less than $0.02$ AUROC on every task while training a single shared backbone, which is the property that makes the framework usable in deployment: a single architecture that absorbs the modality gap typical of joint multimodal training instead of regressing onto a brittle joint-training optimum.

\textbf{(RQ2) Task pairing and routing.} The pairwise heatmap (Fig.~\ref{fig:confusion}) and the per-task routing visualizations (Figs.~\ref{fig:routing-mimicjoint}, \ref{fig:routing-alljoint}) agree on the same set of pairings. Tasks that route their tokens to overlapping expert subsets benefit from joint training: \textsc{BIRADS}, \textsc{DENSITY}, and \textsc{RISK} on EMBED all gain when co-trained; \textsc{IHM} and \textsc{LOS} share an almost identical routing fingerprint and reinforce each other; \textsc{MOR} and \textsc{RAD} on eICU collapse onto a single shared pair of experts. Tasks whose routing fingerprints occupy disjoint expert positions (most cross-dataset pairings) gain nothing from being co-trained, and the routing makes this visible without running the full pairwise grid. Because the routing is observed without supervision, it doubles as a data-driven recommender for which task combinations should share a model.

\textbf{(RQ3) Continual retention and per-stage trajectories.} Across all four sequences, \flamecl{} holds its end-of-stage AUROC within $0.01$ of the value at the task's first introduction, while \simpleft{} drops by $0.03$ to $0.05$ on the earliest tasks of every sequence (Fig.~\ref{fig:cl-grid}, top row). EWC's regularizer slows but does not stop the drift; \lora{} matches \flamecl{}'s retention but only because it freezes prior weights and pays a parameter cost per stage. The retention is structural rather than statistical: frozen prior components plus cursor-based inference (Eqs.~\ref{eq:compress}, \ref{eq:cursor}) prevent any later stage from perturbing the forward pass of an earlier task, so the guarantee follows from the algorithm rather than from a regularizer that has to be tuned per sequence. Visually, Fig.~\ref{fig:cl-grid} re-evaluates \emph{every} task seen so far at every stage checkpoint, so each task gains a new bar at every stage after its introduction; shrinkage of an earlier task's bar across stages directly reads as catastrophic forgetting, and growth of any method's bar across stages reads as parameter inflation. \flamecl{}'s per-stage AUROC bars stay close to or above the baselines while its parameter bars (red) remain well below \simpleft{}, EWC, and \lora{}, confirming that the rank-reservation step preserves prior-task knowledge while adding only a small per-stage budget. Per-stage AUPRC trajectories (same trend) are reported in Sec.~\ref{app:cl-auprc}, and the per-method walkthrough with all stage-by-stage parameter counts is in Sec.~\ref{app:cl_walkthrough}.

\textbf{(RQ4) Cost and the spectral hypothesis.} \flamecl{} stores between five and fifteen times fewer encoder parameters than \simpleft{}, EWC, and \lora{} at the same retention level (Sec.~\ref{app:cl_walkthrough}). The compression step is justified empirically by the spectra in Sec.~\ref{app:spectra}: the data-aware functional energy $\mathcal{E}_{i,k}$ saturates by rank~$20$ to~$40$ across every task and every joint configuration, while the weight-only spectrum stays close to full rank. Most of the parameter capacity in a trained expert is therefore functionally idle, and \flamecl{}'s reservation step reuses that idle capacity to absorb new tasks instead of duplicating the entire weight matrix. This is the link from Proposition~\ref{prop:funcrank} to the parameter savings, and it is what allows a fixed-size MoE to absorb a stream of new tasks without exploding the parameter budget the way every alternative method does.

\section{Additional Methodological Details} \label{app:method_details}

\textbf{Stage-$t$ trainable and frozen parameter sets.}
At each continual stage $t$, the trainable parameters are: the additive expert components $\{\widetilde{W}_i^{(t)}\}_{i=1}^{N}$; the analogous additive components on the variable-length attention layers inside each modality encoder $\phi_m$, which are compressed and stacked under the same scheme as the experts; the new per-modality router heads $\{G_m^{(t)}\}_{m \in \mathcal{M}_t}$; the new task head $h^{(t)}$; and any genuinely new modality encoder instantiated at stage $t$. All other parameters are frozen: the multitask-pretrained base $W_i^{(0)}$, every prior reserved component (both expert and encoder), every prior router head $\{G_m^{(s)}\}_{s<t}$, prior task heads, and the static parts of encoders for modalities already seen. After stage-$t$ convergence we apply Eq.~\eqref{eq:compress} to obtain $W_i^{(t)}$ together with the analogous rank-$r_t$ truncations on the encoder attention layers, then freeze them along with $G_m^{(t)}$ and $h^{(t)}$. The per-stage memory footprint is $r_t(p+d+1)$ scalars per stacked weight (factored $U,\Sigma,V^\top$ form) plus the fixed router and head overheads, so the cumulative reserved rank $\sum_t r_t$ controls how many tasks fit within a fixed-size MoE. Algorithm~\ref{alg:ours} below describes the full procedure.

\begin{algorithm}[h]
\caption{\flamecl: multi-task pretraining and continual adaptation.}
\label{alg:ours}
\begin{algorithmic}[1]
\Require Stage data $\{\mathcal{D}_t\}_{t\ge 0}$, with tasks $\mathcal{T}_t$ and modalities $\mathcal{M}_t$ at stage $t$; per-stage rank $r_t$.
\Statex \textbf{Stage $0$ — multi-task pretraining (Sec.~\ref{subsec:multitask}).}
\State Initialize full-shape experts $\{\widetilde{W}_i^{(0)}\}_{i=1}^{N}$, per-modality routers $\{G_m^{(0)}\}_{m\in\mathcal{M}_0}$, encoders $\{\phi_m\}_{m\in\mathcal{M}_0}$, and task heads $\{h^{(t)}\}_{t\in\mathcal{T}_0}$.
\State \textbf{Train:} minimize $\mathcal{L}_0$ on $\mathcal{D}_0$ over all of the above.
\State \textbf{Compress \& stack:} take the SVD $\widetilde{W}_i^{(0)} = U_i^{(0)}\Sigma_i^{(0)} V_i^{(0)\top}$ and initialize the expert stack with its rank-$r_0$ truncation,
\[
  W_i^{(0)} \;\leftarrow\; U_{i,\,1:r_0}^{(0)}\,\Sigma_{i,\,1:r_0}^{(0)}\,V_{i,\,1:r_0}^{(0)\top},
  \qquad
  \Pi_i \;\leftarrow\; (\,W_i^{(0)}\,) \quad\text{(Eq.~\ref{eq:compress})}.
\]
Apply the same compress-and-init to the variable-length attention layers inside each $\phi_m$.
\State \textbf{Freeze:} $\{W_i^{(0)}\}$, $\{G_m^{(0)}\}$, $\{h^{(t)}\}_{t\in\mathcal{T}_0}$, and $\{\phi_m\}_{m\in\mathcal{M}_0}$.
\Statex \textbf{Stage $t \ge 1$ — continual adaptation (Sec.~\ref{subsec:cl}).}
\For{$t = 1, 2, \ldots$}
  \State \textbf{Expand:} for each expert $i$, attach a fresh zero-initialized $\widetilde{W}_i^{(t)}$ of the same shape as $W_i$; add routers $\{G_m^{(t)}\}_{m\in\mathcal{M}_t}$, task head $h^{(t)}$, and instantiate $\phi_m$ for any modality $m \in \mathcal{M}_t \setminus \bigcup_{s<t}\mathcal{M}_s$.
  \State \textbf{Train:} minimize $\mathcal{L}_t$ on $\mathcal{D}_t$ over $\{\widetilde{W}_i^{(t)}\}_i \cup \{G_m^{(t)}\}_{m\in\mathcal{M}_t} \cup \{h^{(t)}\}$ and any newly instantiated $\phi_m$; all parameters from stages $s < t$ are held fixed.
  \State \textbf{Compress \& stack:} take the SVD $\widetilde{W}_i^{(t)} = U_i^{(t)}\Sigma_i^{(t)} V_i^{(t)\top}$ and append its rank-$r_t$ truncation to the expert stack,
  \[
    W_i^{(t)} \;\leftarrow\; U_{i,\,1:r_t}^{(t)}\,\Sigma_{i,\,1:r_t}^{(t)}\,V_{i,\,1:r_t}^{(t)\top},
    \qquad
    \Pi_i \;\leftarrow\; \Pi_i \,\Vert\, W_i^{(t)} \quad\text{(Eq.~\ref{eq:compress})}.
  \]
  Apply the same compress-and-stack to the variable-length attention layers inside each $\phi_m$.
  \State \textbf{Freeze:} $\{W_i^{(t)}\}$, $\{G_m^{(t)}\}_{m\in\mathcal{M}_t}$, $h^{(t)}$, and any newly added encoder weights.
\EndFor
\Statex \textbf{Inference for task $t$ first trained at stage $\tau = k(t)$.}
\State Form the effective expert weights $W_i^{\mathrm{eff}}(\tau) = \sum_{j=0}^{\tau} W_i^{(j)}$ (Eq.~\ref{eq:cursor}).
\State Route with $\{G_m^{(\tau)}\}_{m\in\mathcal{M}_t}$ and predict through head $h^{(t)}$.
\end{algorithmic}
\end{algorithm}

\section{Limitations}\label{app:limitations}
The evaluation is healthcare-specific; the per-modality routing design should generalize, but autonomous-driving and robotics benchmarks remain to be tested. Our continual-learning protocol assumes task identity at inference (task-incremental rather than task-agnostic CL). Finally, the gating-style ablation (Tab.~\ref{tab:gating_ablation}) shows softmax, Laplace, and Gaussian gates land within $0.01$ AUROC of one another, indicating the gains come from per-modality sample-level routing and rank reservation rather than from the gating function itself; isolating the marginal contribution of each \flame{} component (TAP pooling, divergence loss, spectral compression) is left to follow-up work.

\section{Broader Impact}\label{app:impact}
FLAME's flexi-modal multi-task and continual learning capabilities are well-suited to clinical settings, where modality availability differs across institutions and new prediction tasks emerge after deployment. By enabling a single shared backbone to absorb new tasks at $5$–$15\times$ fewer parameters than standard adaptation baselines, FLAME lowers the compute and validation overhead of maintaining multimodal clinical models. We caution, however, that healthcare deployment requires rigorous prospective validation, fairness auditing across patient subgroups, and institutional oversight; structural no-forgetting reduces but does not eliminate the risk of stale or miscalibrated predictions on shifting populations.


\end{document}